\documentclass[%
  twoside,
  11pt,
  final,
]{article}

\usepackage[utf8]{inputenc}
\usepackage[T1]{fontenc}

\usepackage{caption}  

\usepackage[
  hyperref,
  preprint,
]{jmlr2e}

\usepackage{lastpage}

\newcommand*{\shorttitle}[1]{\newcommand*{\jmlrshorttitle}{#1}}
\newcommand*{\authorlistfull}[1]{\newcommand*{\jmlrauthorlistfull}{#1}}
\newcommand*{\authorlistlast}[1]{\newcommand*{\jmlrauthorlistlast}{#1}}

\newcommand*{\jmlrmeta}[6]{%
    \newcommand*{\jmlrvolume}{#1}%
    \newcommand*{\jmlryear}{#2}%
    \newcommand*{\jmlrfirstpage}{#3}%
    \newcommand*{\jmlrdatesubmitted}{#4}%
    \newcommand*{\jmlrdatepublished}{#5}%
    \newcommand*{\jmlrpaperid}{#6}%
}

\newcommand*{\jmlrsetup}{%
    \firstpageno{\jmlrfirstpage}%
    \jmlrheading{\jmlrvolume}%
                {\jmlryear}%
                {\jmlrfirstpage{-}\pageref*{LastPage}}%
                {\jmlrdatesubmitted}%
                {\jmlrdatepublished}%
                {\jmlrpaperid}%
                {\jmlrauthorlistfull}%
    \ShortHeadings{\jmlrshorttitle}{\makeatletter\jmlrauthorlistlast\makeatother}%
}

\usepackage{ifdraft}

\usepackage[
  obeyFinal,
  textsize=footnotesize,
  figwidth=0.99\linewidth,
]{todonotes}
\usepackage{environ}

\ifoptionfinal{}{\setlength{\marginparwidth}{2.5cm}}

\NewEnviron{todoenv}[1][]{%
  \todo[inline, caption={2do}, #1]{%
    \begin{minipage}{\textwidth - 4pt}
      \BODY
    \end{minipage}
  }
}

\usepackage{etoolbox}  
\usepackage{xparse}    

\usepackage[inline]{enumitem}

\usepackage[numbered]{bookmark}  

\usepackage{natbib}
\usepackage{doi}

\usepackage[nottoc, section]{tocbibind}

\usepackage[
  labelformat=simple,
]{subcaption}

\usepackage{tikz}

\usetikzlibrary{positioning}

\usepackage{graphicx}

\graphicspath{%
  {figures/}%
}
\usepackage{booktabs}
\usepackage{multirow}

\newcommand{\cmark}{}%
\DeclareRobustCommand{\cmark}{%
  \tikz\fill[scale=0.4, color=black!30!green]
  (0,.35) -- (.25,0) -- (1,.7) -- (.25,.15) -- cycle;%
}
\newcommand{\xmark}{}%
\DeclareRobustCommand{\xmark}{%
  \tikz [x=1.4ex,y=1.4ex,line width=.2ex, color=red] \draw (0,0) -- (1,1) (0,1) -- (1,0);
}
\usepackage{algorithm}
\usepackage[noend]{algpseudocode}

\makeatletter
\expandafter\patchcmd\csname\string\algorithmic\endcsname{\itemsep\z@}{\itemsep=0.25ex}{}{} 
\newcommand\fs@booktabsruled{%
    \def\@fs@cfont{\bfseries\strut}\let\@fs@capt\floatc@ruled
    \def\@fs@pre{\hrule height\heavyrulewidth depth0pt \kern\belowrulesep}%
    \def\@fs@mid{\kern\aboverulesep\hrule height\lightrulewidth\kern\belowrulesep}%
    \def\@fs@post{\kern\aboverulesep\hrule height\heavyrulewidth\relax}%
    \let\@fs@iftopcapt\iftrue
}
\makeatother
\floatstyle{booktabsruled} 
\restylefloat{algorithm}
\algrenewcommand{\algorithmiccomment}[1]{\hfill {\small \textcolor{darkgray}{$\vartriangleright$ #1}}}  
\algrenewcommand\algorithmicindent{1.5em}   
\algrenewcommand\alglinenumber[1]{\small {\textcolor{darkgray}{#1}}} 
\usepackage{amsmath, amssymb}
\usepackage{bm}
\usepackage{mathtools}
\usepackage{nicefrac}

\usepackage{amsthm}
\usepackage{thmtools}
\usepackage{thm-restate}

\allowdisplaybreaks    

\numberwithin{equation}{section}  

\newcommand*{\defeq}{\coloneqq}  
\newcommand*{\rdefeq}{\eqqcolon}  

\providecommand{\where}{}
\providecommand{\suchthat}{}

\DeclarePairedDelimiterX{\set}[1]\{\}{%
\renewcommand*{\where}{\colon}
\renewcommand*{\suchthat}{%
  \nonscript\:\delimsize\vert%
  \allowbreak%
  \nonscript\:%
  \mathopen{}%
}%
#1}

\newcommand*{\setsym}[1]{\ensuremath{\mathbb{#1}}}

\newcommand*{\N}{\setsym{N}}

\newcommand*{\R}{\setsym{R}}

\DeclarePairedDelimiter{\card}{\lvert}{\rvert}

\DeclarePairedDelimiter{\ointerval}{]}{[}

\newcommand*{\maps}[2]{#2^{#1}}  

\newcommand*{\preim}[1]{#1^{-1}}  

\newcommand*{\inv}{^{-1}}

\newcommand*{\id}[1][]{\operatorname{id}\ifstrempty{#1}{}{_{#1}}}

\newcommand*{\spacesym}[1]{{\mathbb{#1}}}  

\newcommand*{\boundary}[1]{\partial #1}
\newcommand*{\closure}[1]{\overline{#1}}

\DeclareMathOperator{\seqcl}{scl}

\newcommand*{\metricsp}[1]{\spacesym{#1}}

\renewcommand*{\vec}[1]{{\bm{#1}}}
\newcommand*{\vecelem}[1]{{#1}}

\newcommand*{\pvec}[1]{\begin{pmatrix} #1 \end{pmatrix}}

\newcommand*{\linspan}[1]{\operatorname{span} \left( #1 \right)}

\newcommand*{\mat}[1]{{\bm{#1}}}
\newcommand*{\matelem}[1]{{#1}}

\NewDocumentCommand{\range}{s m}{%
  \operatorname{ran}
  \IfBooleanTF{#1}{\left(}{(}
  #2
  \IfBooleanTF{#1}{\right)}{)}
}

\NewDocumentCommand{\kernel}{s m}{%
  \operatorname{ker}
  \IfBooleanTF{#1}{\left(}{(}
  #2
  \IfBooleanTF{#1}{\right)}{)}
}

\newcommand*{\isomorphism}[2]{\linop{I}_{#1}^{#2}}

\newcommand*{\pinv}{^\dagger}

\newcommand*{\proj}[1]{\linop{P}_{#1}}

\newcommand*{\dualsp}[1]{#1'}
\newcommand*{\bidualsp}[1]{#1''}

\newcommand*{\banachsp}[1][B]{\spacesym{#1}}  

\DeclarePairedDelimiter{\abs}{\lvert}{\rvert}

\makeatletter
\DeclarePairedDelimiterXPP{\@norm}[2]{}{\lVert}{\rVert}{\ifstrempty{#2}{}{_{#2}}}{#1}
\NewDocumentCommand{\norm}{s O{} m O{}}{%
  \IfBooleanTF{#1}{%
    \@norm*{#3}{#4}%
  }{%
    \@norm[#2]{#3}{#4}%
  }%
}
\makeatother

\newcommand*{\dualop}{'}

\newcommand*{\hilbertsp}[1][H]{\spacesym{#1}}  

\makeatletter
\DeclarePairedDelimiterXPP{\@inprod}[3]{}{\langle}{\rangle}{\ifstrempty{#3}{}{_{#3}}}{#1, #2}
\NewDocumentCommand{\inprod}{s O{} m m O{}}{%
  \IfBooleanTF{#1}{%
    \@inprod*{#3}{#4}{#5}%
  }{%
    \@inprod[#2]{#3}{#4}{#5}%
  }%
}
\makeatother

\newcommand*{\T}{^\top}
\newcommand*{\adj}{^*}

\NewDocumentCommand{\cfns}{m o}{%
  C (#1\IfValueT{#2}{, #2})
}

\NewDocumentCommand{\cdfns}{O{0} m o}{%
  C^{#1} (#2\IfValueT{#3}{, #3})
}

\NewDocumentCommand{\bucdfns}{m m}{%
  C^{#1}(\closure{#2})
}

\newcommand*{\testfns}[1]{C_c^\infty \left( #1 \right)}

\RenewDocumentCommand{\L}{m o}{%
  \ensuremath{
    L_{#1}
    \IfNoValueF{#2}{\left( #2 \right)}
  }
}

\NewDocumentCommand{\sobolev}{o m o}{%
  \IfNoValueTF{#1}{
    H^{#2}
  }{
    W^{#1,#2}
  }
  \IfNoValueF{#3}{\left( #3 \right)}
}
\NewDocumentCommand{\sobolevtest}{o m o}{%
  \sobolev[#1]{#2}_0
  \IfNoValueF{#3}{\left( #3 \right)}
}

\newcommand*{\rkhs}[1]{\hilbertsp_{#1}}

\newcommand*{\linop}[1]{\mathcal{#1}}

\newcommand*{\linfctls}[1]{{\bm{\linop{#1}}}}
\newcommand*{\linfctlselem}[1]{{\linop{#1}}}

\makeatletter
\DeclarePairedDelimiter{\@evallinop@oparg}{[}{]}
\DeclarePairedDelimiter{\@evallinop@fnarg}{(}{)}
\NewDocumentCommand{\evallinop}{m s O{} m s O{} d()}{%
  #1%
  \IfBooleanTF{#2}{%
    \@evallinop@oparg*{#4}%
  }{%
    \@evallinop@oparg[#3]{#4}%
  }%
  \IfValueT{#7}{%
    \IfBooleanTF{#5}{%
      \@evallinop@fnarg*{#7}%
    }{%
      \@evallinop@fnarg[#6]{#7}%
    }%
  }%
}
\makeatother

\NewDocumentCommand{\linopat}{m s O{} m d()}{%
  \IfBooleanTF{#2}{%
    \evallinop{\linop{#1}}*{#4}(#5)%
  }{%
    \evallinop{\linop{#1}}[#3]{#4}(#5)%
  }%
}

\NewDocumentCommand{\linfctlsat}{m s O{} m d()}{%
  \IfBooleanTF{#2}{%
    \evallinop{\linfctls{#1}}*{#4}(#5)%
  }{%
    \evallinop{\linfctls{#1}}[#3]{#4}(#5)%
  }%
}

\newcommand*{\borelsigalg}[1]{\mathcal{B} \left( #1 \right)}

\newcommand*{\diff}[2][1]{%
  \mathrm{d} #2\ifstrequal{#1}{1}{}{^{#1}}%
}

\newcommand*{\pdiff}[2][1]{%
  \partial #2\ifstrequal{#1}{1}{}{^{#1}}%
}

\NewDocumentCommand{\deriv}{s O{1} m m}{%
  \frac{%
    \mathrm{d}\ifstrequal{#2}{1}{}{^{#2}}\IfBooleanT{#1}{#4}
  }{%
    \diff[#2]{#3}
  }
  \IfBooleanF{#1}{#4}
}

\NewDocumentCommand{\derivat}{s O{1} m m m}{%
  \left.
    \IfBooleanTF{#1}{%
      \deriv*[#2]{#3}{#4}
    }{%
      \deriv[#2]{#3}{#4}
    }
  \right|_{#3 = #5}
}

\NewDocumentCommand{\dderiv}{m m}{%
  \partial_{#1} #2
}

\NewDocumentCommand{\dderivat}{m m o m}{%
  \IfNoValueTF{#3}{%
    \dderiv{#1}{#2} \left( #4 \right)
  }{%
    \left.
    \dderiv{#1}{#2}
    \right_{#3 = #4}
  }
}

\NewDocumentCommand{\pderiv}{s O{1} m m}{%
  \frac{%
    \partial\ifstrequal{#2}{1}{}{^{#2}}\IfBooleanT{#1}{#4}
  }{%
    \pdiff[#2]{#3}
  }
  \IfBooleanF{#1}{#4}
}

\NewDocumentCommand{\pderivat}{s O{1} m m o m}{%
  \left.
    \IfBooleanTF{#1}{%
      \pderiv*[#2]{#3}{#4}
    }{%
      \pderiv[#2]{#3}{#4}
    }
  \right|_{\IfNoValueTF{#5}{#3}{#5} = #6}
}

\NewDocumentCommand{\mpderiv}{s O{1} m m}{%
  \frac{%
    \partial\ifstrequal{#2}{1}{}{^{#2}}\IfBooleanT{#1}{#4}
  }{%
    #3
  }
  \IfBooleanF{#1}{#4}
}

\NewDocumentCommand{\mpderivat}{s O{1} m m m m}{%
  \left.
    \IfBooleanTF{#1}{%
      \mpderiv*[#2]{#3}{#4}
    }{%
      \mpderiv[#2]{#3}{#4}
    }
  \right|_{#5 = #6}
}

\NewDocumentCommand{\mipderiv}{s m o m o}{%
  \IfNoValueTF{#3}{
    \mathrm{D}^{#2}
    \IfBooleanTF{#1}{\left[ #4 \right]}{#4}
    \IfNoValueF{#5}{\left( #5 \right)}
  }{
    \IfNoValueF{#5}{\left.}
    \frac{%
      \partial^{\lvert #2 \rvert}\IfBooleanT{#1}{#4}
    }{%
      \pdiff[#2]{#3}
    }
    \IfBooleanF{#1}{#4}
    \IfNoValueF{#5}{\right\vert_{#3 = #5}}
  }
}

\NewDocumentCommand{\jacobian}{o m o}{%
  \IfNoValueTF{#1}{%
    \mathrm{D} #2 \IfNoValueF{#3}{\left( #3 \right)}
  }{%
    \left.
      \mathrm{D} #2
    \right|_{#1\IfNoValueF{#3}{= #3}}
}
}

\NewDocumentCommand{\gradient}{o m o}{%
  \IfNoValueTF{#1}{%
    \nabla #2 \IfNoValueF{#3}{\left( #3 \right)}
  }{%
    \left.
      \nabla #2
    \right|_{#1\IfNoValueF{#3}{= #3}}
}
}

\NewDocumentCommand{\divergence}{m o}{%
  \operatorname{div} \left( #1 \right) \IfNoValueF{#2}{\left( #2 \right)}
}

\NewDocumentCommand{\hessian}{o m o}{%
  \IfNoValueTF{#1}{%
    H #2 \IfNoValueF{#3}{\left( #3 \right)}
  }{%
    \left.
      H #2
    \right|_{#1\IfNoValueF{#3}{= #3}}
}
}

\NewDocumentCommand{\laplaceop}{o m o}{%
  \IfNoValueTF{#1}{%
    \Delta #2 \IfNoValueF{#3}{\left( #3 \right)}
  }{%
    \left.
      \Delta #2
    \right|_{#1\IfNoValueF{#3}{= #3}}
}
}

\newcommand*{\lintegral}[4][]{%
  \int_{#1} #4 \,\diff[1]{#2 \left( #3 \right)}%
}

\NewDocumentCommand{\rintegral}{o o m m}{%
  \int\IfNoValueF{#1}{_{#1}}\IfNoValueF{#2}{^{#2}} #4 \,\diff[1]{#3}%
}

\makeatletter

\newcommand*{\@probsymbol}{\mathrm{P}}

\newcommand*{\@given}[1]{%
  \nonscript\:#1\vert
  \allowbreak
  \nonscript\:
  \mathopen{}}

\providecommand*{\given}{}

\DeclarePairedDelimiterXPP{\@prob}[1]{\@probsymbol}{(}{)}{}{%
  \renewcommand*{\given}{\@given{\delimsize}}%
  #1}

\newcommand*{\prob}[1]{\ifblank{#1}{\@probsymbol}{\@prob*{#1}}}

\makeatother

\newcommand{\indp}[2]{#1 \perp\!\!\!\!\perp #2}

\newcommand*{\rvar}[1]{{\mathrm{#1}}}
\newcommand*{\rvec}[1]{{\bm{\mathrm{#1}}}}
\newcommand*{\rvecelem}[1]{\rvar{#1}}

\makeatletter

\DeclarePairedDelimiterX{\@condrv}[1]{.}{.}{%
  \renewcommand*{\given}{\@given{\delimsize}}%
  #1}

\NewDocumentCommand{\condrv}{som}{%
  \IfBooleanTF{#1}{%
    \@condrv*{#3}
  }{%
    \IfNoValueTF{#2}{%
      \begingroup%
      \renewcommand*{\given}{\@given{}}%
      #3%
      \endgroup%
    }{%
      \@condrv[#2]{#3}%
    }
  }
}

\makeatother

\NewDocumentCommand{\expectation}{o o m}{%
  \operatorname{\mathbb{E}}\IfNoValueF{#1}{_{#1\IfNoValueF{#2}{\sim #2}}} \left[ #3 \right]
}
\NewDocumentCommand{\covariance}{o o m m}{%
  \operatorname{Cov}\IfNoValueF{#1}{_{#1\IfNoValueF{#2}{\sim #2}}} \left[ #3, #4 \right]
}
\NewDocumentCommand{\variance}{o o m}{%
  \operatorname{\mathbb{V}}\IfNoValueF{#1}{_{#1\IfNoValueF{#2}{\sim #2}}} \left[ #3 \right]
}

\newcommand*{\gaussian}[2]{{\ensuremath{\operatorname{\mathcal{N}}\left(#1, #2\right)}}}

\newcommand*{\rproc}[1]{{\mathrm{#1}}}
\newcommand*{\morproc}[1]{{\bm{\mathrm{#1}}}}

\newcommand*{\paths}[1]{\operatorname{paths}\left(#1\right)}

\newcommand*{\gp}[2]{{\ensuremath{\operatorname{\mathcal{GP}}}\left(#1, #2\right)}}
\newcommand*{\multigp}[2]{%
  \gp{%
    \begin{pmatrix}
      #1
    \end{pmatrix}
  }{
    \begin{pmatrix}
      #2
    \end{pmatrix}
  }
}

\newcommand*{\Lk}[2]{#1 #2}
\newcommand*{\kL}[2]{#1 #2\dualop}
\NewDocumentCommand{\LkL}{m m o}{%
  #1 #2 \IfValueTF{#3}{#3}{#1}\dualop
}

\declaretheoremstyle[
  numberwithin=section,
  spaceabove=\topsep,
  spacebelow=\topsep,
  headfont=\bfseries,
  notefont=\normalfont,
  bodyfont=\itshape,
  headpunct=.,
  notebraces=(),
  postheadspace=5pt plus 1pt minus 1pt,
  headindent=0pt,
]{plain}  

\declaretheorem[style=plain]{theorem}
\declaretheorem[style=plain,sibling=theorem]{proposition}
\declaretheorem[style=plain,sibling=theorem]{lemma}
\declaretheorem[style=plain,sibling=theorem]{corollary}
\declaretheorem[style=plain,sibling=theorem]{definition}

\declaretheorem[style=plain,sibling=theorem]{remark}
\declaretheorem[style=plain]{example}

\declaretheorem[
  numberwithin=,
  heading=Notation,
  style=plain,
]{mainnotation}

\makeatletter
\renewenvironment{proof}[1][\proofname]{\par
  \pushQED{\qed}%
  \normalfont \topsep6\p@\@plus6\p@\relax
  \trivlist
  \item[\hskip\labelsep {\bf #1}]\ignorespaces
}{%
  \popQED\endtrivlist\@endpefalse
}
\makeatother

\newcommand*{\sigalg}{\mathcal{F}}  

\newcommand*{\dom}{\banachsp[D]}  
\newcommand*{\diffop}{\linop{D}}  
\newcommand*{\sol}{u}  
\newcommand*{\solvv}{\vec{u}}  
\newcommand*{\solsp}{\banachsp[U]}  
\newcommand*{\rhs}{f}  
\newcommand*{\rhssp}{\banachsp[V]}  
\newcommand*{\bfn}{g}  
\newcommand*{\truesol}{\sol^\star}  
\newcommand*{\truesolvv}{\solvv^\star}  

\newcommand*{\testfn}{v}
\newcommand*{\weakdiffop}{\diffop^w}
\newcommand*{\weakrhs}{\rhs^w}
\newcommand*{\weakbilin}{B}

\newcommand*{\mwrtestfctl}{l}
\newcommand*{\mwrtestfn}{\psi}

\newcommand*{\mwrtrialvv}{\vec{\phi}}
\newcommand*{\mwrtrial}{\phi}
\newcommand*{\mwrtrialsp}{\hat{\solsp}}
\newcommand*{\mwrtest}{\mwrtestfn}

\newcommand*{\mwrcoords}{\vec{c}}
\newcommand*{\mwrcoordselem}{\vecelem{c}}

\newcommand*{\mwrmat}{\hat{\mat{D}}}
\newcommand*{\mwrmatelem}{\hat{\matelem{D}}}
\newcommand*{\mwrrhs}{\hat{\vec{f}}}
\newcommand*{\mwrrhselem}{\hat{\vecelem{f}}}

\newcommand*{\mwrsol}{\solvv^{\mathrm{MWR}}}
\newcommand*{\mwrsolcoords}{\vec{c}^{\mathrm{MWR}}}
\newcommand*{\mwrsolcoordselem}{\vecelem{c}^{\mathrm{MWR}}}

\newcommand*{\heatcond}{\kappa}  
\newcommand*{\heatrhs}{\dot{q}_V}  
\newcommand*{\heatbfn}{\dot{q}_A}  

\newcommand*{\solvvprior}{\morproc{u}}
\newcommand*{\solprior}{\rproc{u}}
\newcommand*{\rhsprior}{\rproc{f}}
\newcommand*{\bfnprior}{\rproc{g}}

\newcommand*{\heatrhsprior}{\dot{\rproc{q}}_V}
\newcommand*{\heatbfnprior}{\dot{\rproc{q}}_A}

\newcommand*{\mwrwsdiffop}{\diffop^{(w)}}
\newcommand*{\mwrwsrhs}{\rhs^{(w)}}

\newcommand*{\mwrwsbop}{\linop{B}^{(w)}}
\newcommand*{\mwrwsbfn}{\bfn^{(w)}}

\newcommand*{\mwrwstestfctlsp}{\banachsp[L]^{(w)}}

\newcommand*{\mwrtrialproj}{\linop{P}_{\mwrtrialsp}}
\newcommand*{\mwrcoordproj}{\linfctls{P}_{\R^m}}

\newcommand*{\mwrcoordsprior}{\rvec{c}}
\newcommand*{\mwrcoordspriorelem}{\rvecelem{c}}

\NewDocumentCommand{\mwrinfoop}{m o}{\linop{I}_{#1\IfValueT{#2}{, #2}}}

\newcommand*{\mwrwsrhsprior}{\rhsprior^{(w)}}
\newcommand*{\mwrwsbfnprior}{\bfnprior^{(w)}}

\newcommand*{\gpprior}{\rproc{f}}  
\newcommand*{\mogpprior}{\morproc{f}}
\newcommand*{\mogppriorelem}{\rproc{f}}
\newcommand*{\gpidcs}{\metricsp{X}}  
\newcommand*{\gppathsp}{\banachsp[B]}  
\newcommand*{\bgrv}{\rvar{b}}
\newcommand*{\bgrvsp}{\banachsp}

\newcommand*{\wstarseqcl}{\seqcl_{w*}}

\usepackage{siunitx}
\ifdefined\unit\else
  \ifdefined\NewCommandCopy
    \NewCommandCopy\unit\si
  \else
    \NewDocumentCommand\unit{O{}m}{\si[#1]{#2}}
  \fi
\fi
\input{preamble/99_end}

\title{%
  Physics-Informed Gaussian Process Regression\\
  Generalizes Linear PDE Solvers
}
\shorttitle{Physics-Informed GP Regression Generalizes Linear PDE Solvers}

\author{%
  \name Marvin Pf\"{o}rtner$^{1}$ \email marvin.pfoertner@uni-tuebingen.de \\
  \name Ingo Steinwart$^{2}$      \email ingo.steinwart@mathematik.uni-stuttgart.de \\
  \name Philipp Hennig$^{1}$    \email philipp.hennig@uni-tuebingen.de \\
  \name Jonathan Wenger$^{1}$     \email jonathan.wenger@uni-tuebingen.de \\[0.5em]
  \addr
  $^1$University of T\"ubingen, Tübingen AI Center \\
  $^2$University of Stuttgart
}
\authorlistfull{Marvin Pf\"{o}rtner, Ingo Steinwart, Philipp Hennig and Jonathan Wenger}
\authorlistlast{Pf\"{o}rtner, Steinwart, Hennig and Wenger}

\editor{Marc Peter Deisenroth}

\jmlrmeta{
  25%
}{
  2024%
}{
  1%
}{
  04/23%
}{
  MM/YY%
}{
  23-0417
}

\jmlrsetup

\begin{document}
  \maketitle

  \begin{abstract}
    %
%
Linear partial differential equations (PDEs) are an important, widely applied class of
mechanistic models, describing physical processes such as heat transfer, electromagnetism, and wave propagation.
In practice, specialized numerical methods based on discretization are used to solve PDEs. They generally use an estimate of the unknown model parameters and, if available, physical measurements for initialization.
Such solvers are often embedded into larger scientific models with a downstream application and thus error quantification plays a key role.
%
%
However, by ignoring parameter and measurement uncertainty, classical PDE solvers may fail to produce consistent estimates of their inherent approximation error.
%
%
In this work, we approach this problem in a principled fashion by interpreting solving linear PDEs as physics-informed Gaussian process (GP) regression.
Our framework is based on a key generalization of the Gaussian process inference
theorem to observations made via an arbitrary bounded linear operator.
%
%
Crucially, this probabilistic viewpoint allows to
(1) \emph{quantify the inherent discretization error};
(2) \emph{propagate uncertainty about the model parameters to the solution}; and
(3) \emph{condition on noisy measurements}.
Demonstrating the strength of this formulation, we prove that it strictly generalizes
methods of weighted residuals, a central class of PDE solvers including collocation,
finite volume, pseudospectral, and (generalized) Galerkin methods such as finite
element and spectral methods.
This class can thus be directly equipped with a structured error estimate.
%
%
In summary, our results enable the seamless integration of mechanistic models as
modular building blocks into probabilistic models by blurring the boundaries between
numerical analysis and Bayesian inference.

  \end{abstract}


  \section{Introduction}
\label{sec:introduction}
Partial differential equations (PDEs) are powerful mechanistic models of static and
dynamic systems with continuous spatial interactions  \citep{Borthwick2018IntroPDE}.
They are widely used in the natural sciences, especially in physics, and in applied
fields like engineering, medicine and finance.
Linear PDEs form a subclass describing physical phenomena such as heat diffusion
\citep{Fourier1822Chaleur}, electromagnetism \citep{Maxwell1865Electro}, and continuum
mechanics \citep{Lautrup2005PDE}.
Additionally, they are used in applications as diverse as computer graphics
\citep{Kazhdan2006Poisson}, medical imaging \citep{Holder2005EIT}, or option pricing
\citep{Black1973Pricing}.

\paragraph{Scientific inference with PDEs}
Given a mechanistic model of a (physical) system in the form of a linear PDE
\(\linopat{D}{\solvv} = \rhs\), where $\linop{D}$ is a linear differential operator
mapping between vector spaces of functions, the system can be simulated by solving the
PDE subject to a set of linear boundary conditions (BC), given by a linear operator
\(\linop{B}\) and a function \(\bfn\) defined on the boundary of the domain, s.t.
\(\linopat{B}{\solvv} = \bfn\) \citep{Evans2010PDE}.
For instance, given all material parameters and heat sources involved, a PDE can
describe the temperature distribution in an electronic component, while the boundary
conditions describe the heat flux out of the component at the surface.
Since hardly any practically relevant PDE can be solved analytically
\citep{Borthwick2018IntroPDE}, in practice, specialized numerical methods relying on
discretization are employed.
Often such solvers are embedded into larger scientific models, where model parameters are
inferred from measurement and downstream analyses depend on the resulting simulation.
For example, we would like to model whether said electronic component hits critical
temperature thresholds during operation to assess its longevity.

\paragraph{Challenges when solving PDEs}
When performing scientific inference with PDEs via numerical simulation, one is faced
with three fundamental challenges.
\begin{enumerate}[label=(C\arabic*)]
  \item \emph{Limited computation.}
        Any numerically computed solution \(\hat{\solvv} \approx \solvv\) suffers from
        approximation error.
        In practice, a sufficiently accurate simulation often requires vast amounts of
        computational resources.
        \label{cha:limited-computation}
  \item \emph{Partially-known physics.}
        While the underlying physical mechanism is encoded in the formulation of the
        PDE, in practice, its exact parameters and boundary conditions are often
        unknown.
        For example, the position and strength of heat sources \(\rhs\) within the
        aforementioned electric component are only approximately known.
        Similarly, material parameters like thermal conductivity, which define
        \(\linop{D}\), can often only be estimated.
        Finally, the initial or boundary conditions \(\linopat{B}{\solvv}=\bfn\) are
        also only partially known.
        For example, how much heat an electrical component dissipates via its surface.
        \label{cha:partially-known-physics}
  \item \emph{Error propagation.}
        Limited computation and partially-known physics inevitably introduce error into
        the simulation.
        This resulting bias can fundamentally alter conclusions drawn from downstream
        analysis steps, in particular if these are sensitive to input variability.
        For example, an electronic component may be deemed safe based on the simulation,
        although its true internal temperature hits safety-critical levels repeatedly.
        \label{cha:error-propagation}
\end{enumerate}

\paragraph{Solving PDEs as a learning problem}
The challenges of scientific inference with PDEs are fundamentally issues of partial
information.
Here, we interpret solving a PDE as a \emph{learning problem}, specifically as
physics-informed regression, in the spirit of probabilistic numerics
\citep{Hennig2015PN,Cockayne2019BayesPNMeth,Oates2019PNRetro,Owhadi2019StatNumApprox,%
  Hennig2022PNBook}.
By leveraging the tools of Bayesian inference, we can tackle the challenges
\labelcref{cha:limited-computation,cha:partially-known-physics,cha:error-propagation}.
As illustrated in \cref{subfig:framework-illustration}, we model the solution of the PDE
with a Gaussian process, which we condition on observations of the boundary conditions,
the PDE itself and any physical measurements:

\begin{itemize}
  \item \emph{Encoding prior knowledge.}
        We can efficiently leverage any available computation by encoding inductive bias
        about the solution of the PDE.
        For example, we can identify the solution space by ``partial derivative
        counting''.
        Moreover, since PDEs typically model physical systems, expert knowledge is often
        available.
        This includes known physical properties of the system such as symmetries, as
        well as more subjective estimates from previous experience with similar systems
        or computationally cheap approximations.

  \item \emph{Conditioning on the boundary conditions.}
        The linear boundary conditions can be interpreted as measurements of the
        solution of the PDE on the boundary.
        By conditioning on (some of) these measurements, we are not limited to
        satisfying the boundary conditions exactly, but can directly model uncertain
        constraints without having to resort to point estimates.
        Instead, we propagate the uncertainty to the solution estimate.
        This also allows us to handle cases where we do not have a functional form
        \(\bfn\) of the constraints, but only a discrete set of constraints at boundary
        points.

  \item \emph{Conditioning on the PDE.}
        Conditioning a probability measure over the solution on the analytic
        ``observation'' that the PDE holds is generally intractable.
        In the spirit of classic approaches for solving PDEs, we relax the PDE-constraint
        by requiring only a finite number of projections of the associated PDE residual
        onto carefully chosen test functions to be zero.
        This choice of projections defines the discretization and allows for control
        over the amount of expended computation.
        The resulting posterior quantifies the algorithm's uncertainty within a whole
        set of solution candidates.

  \item \emph{Conditioning on measurements.}
        Finally, we can also condition on direct measurements of the solution itself.
        This is especially useful if parameters of the differential operator or boundary
        conditions are uncertain, or if the computational budget is restrictive.
\end{itemize}

The resulting posterior belief quantifies the uncertainty about the true solution
induced by limited computation and partially-known physics (see
\cref{fig:linpde-gp-uncertainty}).
By quantifying this error probabilistically, we can propagate it to any downstream
analysis or decision.
For example, to project the longevity of a newly designed electrical component, we want to simulate how likely it will hit a critical temperature threshold during
operation.
Given our posterior belief, we can simply compute the marginal probability instead of
performing Monte-Carlo sampling, which would require repeated PDE solves at significant
computational expense.

\begin{figure}
  \centering
  \begin{subfigure}[t]{\textwidth}
    \input{figures/figure_1_panels}
    \caption{%
      \emph{Learning to solve the Poisson equation.}
      A problem-specific Gaussian process prior \(\solprior\) is conditioned on
      partially-known physics, given by uncertain boundary conditions (BC) and a linear
      PDE, as well as on noisy physical measurements from experiment.
      The boundary conditions and the right-hand side of the PDE are not known but
      inferred from a small set of noise-corrupted measurements.
      The plots juxtapose the belief $\condrv{\solprior \given \cdots}$ with the true
      solution $\truesol$ of the latent boundary value problem.
    }
    \label{subfig:framework-illustration}
  \end{subfigure}\\
  \vspace{1.5em}
  \begin{subfigure}[t]{0.45\textwidth}
    \includegraphics{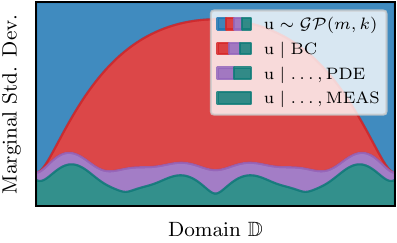}
    \caption{%
      \emph{Uncertainty quantification.}
      Marginal posterior standard deviation after conditioning on uncertain boundary
      conditions, a linear PDE, and noisy (physical) measurements.
    }
    \label{fig:linpde-gp-uncertainty}
  \end{subfigure}
  \hfill
  \begin{subfigure}[t]{0.45\textwidth}
    \includegraphics{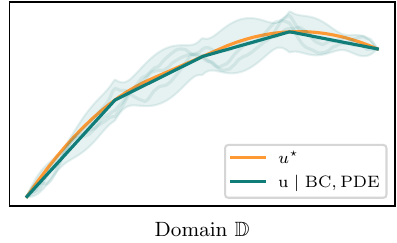}
    \caption{%
      \emph{Generalization of Classical Solvers.}
      For certain priors our framework reproduces any method of weighted residuals, e.g.~%
      the finite element method, in its posterior mean.
    }
    \label{fig:mwr-recovery}
  \end{subfigure}\\
  \caption{%
    \emph{A physics-informed Gaussian process framework for the solution of linear PDEs.}}
  \label{fig:linpde-gp-illustration}
\end{figure}

\paragraph{Contribution}
We introduce a \emph{probabilistic learning framework for the solution of (systems of)
  linear PDEs}.
Our framework can be viewed as physics-informed Gaussian process regression.
It is based on a crucial generalization of a popular result on conditioning GPs on
linear observations to observations made via an arbitrary bounded linear operator with
values in $\R^n$ (\cref{thm:gp-inference-linfctls}).
This enables \emph{combined quantification of uncertainty} from the inherent
discretization error, uncertain initial or boundary conditions, as well as noisy
measurements of the solution.
While connections between GP inference and the
solution of PDEs were made in the past (see \Cref{sec:related-work}),
corresponding methods have largely focused on estimating strong solutions by leveraging finite
difference or collocation schemes. In contrast, our framework applies to both weak and strong
formulations and generalizes a significantly broader class of existing numerical methods.
Our approach is a strict probabilistic \emph{generalization of methods of weighted
  residuals} (\cref{cor:gp-mean-mwr-recovery}), including collocation, finite volume,
(pseudo)spectral, and (generalized) Galerkin methods such as finite element methods.
The resulting probabilistic methods thus have the same convergence properties as their
classic counterparts, while providing a structured error estimate.
Moreover, the probabilistic viewpoint allows to incorporate partially-known physics and
(noisy) experimental measurements.

  \section{Background}
\label{sec:background}

\subsection{Linear Partial Differential Equations}
\label{sec:linear-pde}
A linear \emph{partial differential equation} (PDE) is an equation of the form
\begin{equation}
  \label{eqn:linear-pde}
  \linopat{D}{\solvv} = \rhs,
\end{equation}
where $\linop{D} \colon \solsp \to \rhssp$ is a linear differential operator (see
\cref{def:linear-diffop}) between a Banach space $\solsp$ of $\R^{d'}$-valued functions
and a Banach space $\rhssp$ of real-valued functions on a common open and bounded domain
$\dom \subset \R^d$, and $\rhs \in \rhssp$ is the right-hand side function.
For simplicity of exposition, we will often focus on the case $d' = 1$, in which case we
write $\sol$ instead of $\solvv$.
Systems modeled by linear PDEs are often further constrained by linear
\emph{boundary conditions} (BCs) $\linopat{B}{\solvv} = \bfn$ describing the behavior of
the system on the boundary \(\boundary{\dom}\) of the domain, where $\linop{B}$ is a
linear operator mapping functions $\solvv \in \solsp$ onto functions
$\linopat{B}{\solvv} \colon \boundary{\dom} \to \R$ defined on the boundary and
\(\bfn \colon \boundary{\dom} \to \R\).
Common types of boundary conditions for $d' = 1$ are:
\begin{itemize}[itemsep=0pt]
  \item \emph{Dirichlet}: Specify the values of the solution on the boundary, i.e.~%
        $\linopat{B}{\sol} = \sol \vert_{\boundary{\dom}}$.
  \item \emph{Neumann}: Specify the exterior normal derivative on the boundary, i.e.~%
        \(
        \linopat{B}{\sol}(\vec{x})
        \defeq \dderivat{\vec{\eta}(\vec{x})}{\sol}{\vec{x}}
        \),
        where \(\vec{\eta}(\vec{x})\) is the exterior normal vector at each point of the
        boundary.
\end{itemize}
A PDE and a set of boundary conditions is referred to as a
\emph{boundary value problem} (BVP).
A prototypical example of a linear PDE, used in thermodynamics,
electrostatics and Newtonian gravity, is the \emph{Poisson equation}
\(-\laplaceop{\sol} = \rhs\), where
$\laplaceop{\sol} = \sum_{i = 1}^d \pderiv*[2]{\vecelem{x}_i}{\sol}$
is the Laplacian.

\subsubsection{Weak Formulation}
\label{sec:weak-formulation}
Many models of physical phenomena are expressed as functions $\solvv$, which are not
(continuously) differentiable or even continuous \citep{Evans2010PDE,%
  Borthwick2018IntroPDE,Harrach2021NumPDE}.
In other words, they are not \emph{strong solutions} to any PDE.
There are also PDEs derived from established physical principles, which do not admit
strong solutions at all.
To address this, one can weaken the notion of differentiability leading to the concept
of weak solutions.
Many of the aforementioned physical phenomena are in fact weak solutions.
As an example\footnotemark, consider the weak formulation of the stationary heat
equation for non-homogeneous media
\begin{equation}
  \label{eqn:stat-heat-isotropic-non-homogeneous}
  -\divergence{\heatcond \gradient{\sol}} = \heatrhs.
\end{equation}
\footnotetext{%
  Our exposition is a strongly abbreviated version of
  \citet[Section 6.1.2]{Evans2010PDE}.
}%
Assume that $\sol \in \cdfns[2]{\dom}$, $\heatcond \in \cdfns[1]{\dom}$, and
$\heatrhs \in \cdfns[0]{\dom}$.
If $\sol$ is a solution to \cref{eqn:stat-heat-isotropic-non-homogeneous}, then we can
integrate both sides of the equation against a \emph{test function}
$\testfn \in \testfns{\dom}$, i.e.~an infinitely smooth function with compact support
(see \cref{def:test-fn}), which results in
\begin{equation*}
  -\rintegral[\dom]{\vec{x}}{\divergence{\kappa \gradient{\sol}}[\vec{x}] \testfn(\vec{x})}
  = \rintegral[\dom]{\vec{x}}{\heatrhs(\vec{x}) \testfn(\vec{x})}.
\end{equation*}
Since both $\sol$ and $\testfn$ are sufficiently differentiable, we can apply integration by
parts (Green's first identity) to the first integral to obtain
\begin{equation}
  \label{eqn:stat-heat-isotropic-non-homogeneous-weak}
  \underbrace{
    \rintegral[\dom]{\vec{x}}{\inprod{\kappa(\vec{x}) \gradient{\sol}[\vec{x}]}{\gradient{\testfn}[\vec{x}]}}
  }_{\rdefeq \evallinop{\weakbilin}{\sol, \testfn}}
  = \rintegral[\dom]{\vec{x}}{\heatrhs(\vec{x}) \testfn(\vec{x})},
\end{equation}
since $\testfn \vert_{\boundary{\dom}} = 0$.
This expression does not require $\sol$ to be twice differentiable.
Rather, $\sol$ only needs to be once \emph{weakly differentiable} (see
\citealt[Section 5.2.1]{Evans2010PDE}) with $(\gradient{\sol})_i \in \L2[\dom]$.
Intuitively speaking, a weak derivative of a (classically non-differentiable) function
``behaves like a derivative'' when integrated against a smooth test function.
These relaxed requirements on $\sol$ are exactly the defining properties of the Sobolev
space $\sobolev{1}[\dom]$, i.e.~it suffices that $\sol \in \sobolev{1}[\dom]$.
Similarly, we can weaken all other assumptions to $\testfn \in \sobolevtest{1}[\dom]$,
$\heatrhs \in \L2[\dom]$ and $\kappa \in \L\infty[\dom]$.
Then, for $\sol \in \sobolev{1}[\dom]$ and $\testfn \in \sobolevtest{1}[\dom]$,
\cref{eqn:stat-heat-isotropic-non-homogeneous-weak} is equivalent to
\begin{equation}
  \label{eqn:stat-heat-isotropic-non-homogeneous-bilinear-linear}
  \evallinop{\weakdiffop}{\sol} = \weakrhs,
\end{equation}
where
\(
\weakdiffop \colon \sobolev{1}[\dom] \to \dualsp{\sobolevtest{1}[\dom]},
u \mapsto \evallinop{\weakbilin}{\sol, \cdot}
\)
and $\weakrhs = \inprod{\heatrhs}{\cdot}[\L2[\dom]] \in \dualsp{\sobolevtest{1}[\dom]}$.
Here, $\dualsp{\sobolevtest{1}[\dom]}$ denotes the continuous dual space of
$\sobolevtest{1}[\dom]$.
We define a \emph{weak solution} of \cref{eqn:stat-heat-isotropic-non-homogeneous} as
$\sol \in \sobolev{1}[\dom]$ such that
\cref{eqn:stat-heat-isotropic-non-homogeneous-bilinear-linear}, known as the \emph{weak}
or \emph{variational formulation}, holds.
\begin{definition}
  \label{def:weak-formulation}
  A \emph{weak formulation} of a linear PDE $\linopat{D}{\solvv} = \rhs$ is an equation
  of the form
  \begin{equation}
    \label{eqn:weak-formulation}
    \evallinop{\weakdiffop}{\solvv} = \weakrhs,
  \end{equation}
  where $\weakdiffop \colon \solsp \to \dualsp{\rhssp}$ is a linear operator induced by
  the differential operator $\diffop$ and $\weakrhs \in \dualsp{\rhssp}$ is a linear
  functional induced by the right-hand side $\rhs$.
  A solution to \cref{eqn:weak-formulation} is called a \emph{weak solution} of the PDE.
  In this context, $\linopat{D}{\solvv} = \rhs$ is called the \emph{strong %
    formulation} of the PDE and any solution to it is called a \emph{strong} or
  \emph{classical solution}.
\end{definition}

\todo{Explain why we use nonstandard notation}

\subsubsection{Methods of Weighted Residuals}
\label{sec:mwr}
Unfortunately, linear PDEs both in weak and strong formulation are in general not
analytically solvable, so approximate solutions are sought instead.
\emph{Methods of weighted residuals} (MWR) constitute a large family of popular
numerical approximation schemes for linear PDEs, including \emph{collocation},
\emph{finite volume}, \emph{(pseudo)spectral}, and \emph{(generalized) Galerkin methods}
such as \emph{finite-element methods} \citep{Fletcher1984CompGalerkin}.
Loosely speaking, MWRs interpret a linear PDE as a root-finding problem for the
associated PDE \emph{residual}, i.e.~%
\(
\linopat{D}{\solvv} - \rhs = 0.
\)
Finding the solution of such a system of an uncountably infinite number of equations with infinitely
many unknowns is generally intractable.
To render the problem tractable, we reduce the number of equations by ``projecting''
onto $\R^n$ using a finite number of continuous linear \emph{test functionals}
\( \mwrtestfctl^{(1)}, \dotsc, \mwrtestfctl^{(n)} \in \dualsp{\rhssp} \), i.e.~we use
that the residual being zero implies
\begin{equation}
  \label{eqn:mwr-proj-residual}
  \evallinop{\mwrtestfctl^{(i)}}{\linopat{D}{\solvv} - \rhs}
  = \evallinop{(\mwrtestfctl^{(i)} \circ \linop{D})}{\solvv} - \evallinop{\mwrtestfctl^{(i)}}{\rhs}
  = 0
\end{equation}
for all $i = 1, \dotsc, n$.
This is a relaxation of the original problem, since the above is not an
equivalence but only an implication.\footnotemark
\footnotetext{
  This means that \cref{eqn:mwr-proj-residual} will generally have infinitely many
  solutions and needs regularization to have a unique solution.
}
A common choice for the test functionals appearing in a large class of MWRs is the
integral
\( \evallinop{\mwrtestfctl^{(i)}}{v} \defeq \rintegral[\dom]{x}{\mwrtest^{(i)}(x) v(x)}, \)
where $\mwrtest^{(i)} \in \rhssp$ is a so-called \emph{test function}.
In this case, the test functionals define a weighted average of the current residual, giving rise to the name of the
method.

To reduce the number of unknowns, MWRs also often approximate the unknown solution
function $\solvv$ via finite linear combinations of \emph{trial functions}
\( \mwrtrialvv^{(1)}, \dotsc, \mwrtrialvv^{(m)} \in \solsp, \)
i.e.~%
\begin{equation}
  \label{eqn:mwr-u-hat}
  \solvv
  \approx \hat{\solvv}
  \defeq \sum_{i = 1}^m \mwrcoordselem_i \mwrtrialvv^{(i)},
\end{equation}
where $\mwrcoords \in \R^m$ is the coordinate vector of $\hat{\vec{u}}$ in the
finite-dimensional subspace
$\mwrtrialsp \defeq \linspan{\mwrtrialvv^{(1)}, \dotsc, \mwrtrialvv^{(m)}} \subset \solsp$.
By substituting \cref{eqn:mwr-u-hat} into \cref{eqn:mwr-proj-residual}, we arrive at a linear system
\(
\mwrmat \mwrcoords = \mwrrhs,
\)
where
\(
\mwrmatelem_{ij} \defeq \evallinop{\mwrtestfctl^{(i)}}{\linopat{D}{\mwrtrialvv^{(j)}}}
\)
and
\(
\mwrrhselem_i \defeq \evallinop{\mwrtestfctl^{(i)}}{\rhs}.
\)
Hence, the approximate solution function obtained from this method is given by
\begin{equation}
  \label{eqn:mwr-solution}
  \mwrsol = \sum_{i = 1}^m \mwrsolcoordselem_i \mwrtrialvv^{(i)},
  \qquad \text{where} \qquad
  \mwrsolcoords = \mwrmat\inv \mwrrhs
\end{equation}
assuming that $\mwrmat$ is invertible.
Above, we implicitly assume that the trial functions $\mwrtrialvv^{(i)}$ satisfy the boundary conditions, i.e.~we describe so-called \emph{interior
  methods}.\footnotemark
\footnotetext{
  By stacking the residuals corresponding to the PDE and the boundary conditions, the
  approach outlined here can be used to realized \emph{mixed methods}, which
  solve the boundary value problem without requiring that $\hat{\solvv}$ fulfills the
  boundary conditions by construction.
}

The procedure outlined above can also be applied to approximate weak solutions to linear
PDEs by simply substituting $\diffop \gets \weakdiffop$, $\rhs \gets \weakrhs$, and
$\rhssp \gets \dualsp{\rhssp}$.
In this case, it is customary to employ test functionals
$\mwrtestfctl^{(i)} \in \bidualsp{\rhssp}$
induced by test functions $\mwrtest^{(i)} \in \rhssp$ such that
\(
\evallinop{\mwrtestfctl^{(i)}}{\evallinop{\weakdiffop}{\solvv}}
= \evallinop{\evallinop{\weakdiffop}{\solvv}}{\mwrtest^{(i)}}
\)
and
\(
\evallinop{\mwrtestfctl^{(i)}}{\weakrhs}
= \evallinop{\weakrhs}{\mwrtest^{(i)}}.
\)\footnotemark
\footnotetext{
  This uses the fact that there is an isometric embedding
  \(
  \iota \colon \rhssp \to \bidualsp{\rhssp},
  v \mapsto (l \mapsto \evallinop{l}{v}),
  \)
  where $\bidualsp{\rhssp}$ denotes the strong bidual of $\rhssp$
  \citep[Section IV.8]{Yosida1995FuncAna}.
}
In particular, in the example from \cref{sec:weak-formulation}, this implies
\(
\evallinop{\mwrtestfctl^{(i)}}{\evallinop{\weakdiffop}{\solvv}}
= \evallinop{\weakbilin}{\solvv, \mwrtest^{(i)}}
\)
and
\(
\evallinop{\mwrtestfctl^{(i)}}{\weakrhs}
= \inprod{\heatrhs}{\mwrtest^{(i)}}[\L2[\dom]].
\)
Following \citet{Fletcher1984CompGalerkin}, we will also refer to these methods as
methods of weighted residuals.

\Cref{tbl:mwrs-trial-test} lists the aforementioned examples of MWRs together with the
corresponding trial and test function(al)s that induce the method.

\subsection{Gaussian Processes}
\label{sec:gp}
A \emph{Gaussian process} (GP) $\gpprior$ with index set $\gpidcs$ is a family
\(
\set{\gpprior_\vec{x}}_{\vec{x} \in \gpidcs}
\)
of real-valued random variables on a common probability space $(\Omega, \sigalg, \prob{})$,
such that, for each finite set of indices $\vec{x}_1, \dotsc, \vec{x}_n \in \gpidcs$,
the joint distribution of $\gpprior_{\vec{x}_1}, \dotsc, \gpprior_{\vec{x}_n}$ is
Gaussian.
We also write $\gpprior(\vec{x}) \defeq \gpprior_\vec{x}$ and
$\gpprior(\vec{x}, \omega) \defeq \gpprior_\vec{x}(\omega)$.
The function $\vec{x} \mapsto \expectation[\prob{}]{\gpprior(\vec{x})}$ is called the
\emph{mean (function)} of $\gpprior$ and the function
$(\vec{x}_1, \vec{x}_2) \mapsto \covariance[\prob{}]{\gpprior(\vec{x}_1)}{\gpprior(\vec{x}_2)}$
is called the \emph{covariance function} or \emph{kernel} of $\gpprior$.
We write $\gpprior \sim \gp{m}{k}$ to indicate that $\gpprior$ is a Gaussian process
with mean function $m$ and covariance function $k$.
For each $\omega \in \Omega$, the function
\(
\gpprior(\cdot, \omega) \colon \gpidcs \to \R,
\vec{x} \mapsto \gpprior(\vec{x}, \omega)
\)
is called a \emph{sample} or \emph{(sample) path} of the Gaussian process.
We denote the set of all sample paths of $\gpprior$ by
\(
\paths{\gpprior}
\defeq \set{\gpprior(\cdot, \omega) \where \omega \in \Omega}
\subset \maps{\gpidcs}{\R}.
\)

The sample paths of Gaussian processes are always real-valued.
However, especially in the context of PDEs, vector-valued functions are ubiquitous, e.g.~%
when dealing with vector fields such as the electric field.
Fortunately, the index set of a Gaussian process can be chosen freely, which means that
we can ``emulate'' vector-valued GPs.
More precisely, a function $\vec{f} \colon \gpidcs \to \R^{d'}$ can be equivalently
viewed as a function $\tilde{f} \colon \{ 1, \dotsc, d' \} \times \gpidcs \to \R$ with
$\tilde{f}(i, \vec{x}) \defeq \vecelem{f}_i(\vec{x})$.
Applying this construction to a Gaussian process leads to the notion of a multi-output
Gaussian process:
A $d'$\emph{-output Gaussian process} $\morproc{f}$ with index set $\gpidcs$ is a family
$\set{\morproc{f}_\vec{x}}_{\vec{x} \in \gpidcs}$ of $\R^{d'}$-valued random variables
on $(\Omega, \sigalg, \prob{})$ such that
\(
\tilde{\rproc{f}}
\defeq \set{(\morproc{f}_\vec{x})_i}_{(i, \vec{x}) \in \{ 1, \dotsc, d \} \times \gpidcs}
\)
is a Gaussian process.
As before, we define $\morproc{f}(\vec{x}) \defeq \morproc{f}_\vec{x}$ and
$\morproc{f}(\vec{x}, \omega) \defeq \morproc{f}_\vec{x}(\omega)$.
The mean function $\vec{m} \colon \gpidcs \to \R^{d'}$ and covariance function
$\mat{k} \colon \gpidcs \times \gpidcs \to \R^{d' \times d'}$ of $\morproc{f}$ are
defined by
\begin{equation*}
  \vec{m}(\vec{x}) =
  \begin{pmatrix}
    \tilde{m}(1, \vec{x}) \\
    \vdots                \\
    \tilde{m}(d', \vec{x})
  \end{pmatrix}
  \quad \text{and} \quad
  \mat{k}(\vec{x}_1, \vec{x}_2) =
  \begin{pmatrix}
    \tilde{k}((1, \vec{x}_1), (1, \vec{x}_2))  & \hdots & \tilde{k}((1, \vec{x}_1), (d', \vec{x}_2))  \\
    \vdots                                     & \ddots & \vdots                                      \\
    \tilde{k}((d', \vec{x}_1), (1, \vec{x}_2)) & \hdots & \tilde{k}((d', \vec{x}_1), (d', \vec{x}_2)) \\
  \end{pmatrix},
\end{equation*}
where $\tilde{\rproc{f}} \sim \gp{\tilde{m}}{\tilde{k}}$,
and we write $\morproc{f} \sim \gp{\vec{m}}{\mat{k}}$.

  \section{Learning the Solution to a Linear PDE}
\label{sec:pde-solution-inference}
Consider a linear partial differential equation
\(
\linopat{D}{\solvv} = \rhs
\)
subject to linear boundary conditions \(\linopat{B}{\solvv} = \bfn\) as in
\cref{sec:linear-pde}.
Our goal is to find a solution $\solvv \in \solsp$ satisfying the PDE for (partially)
known $(\linop{D}, \rhs)$ and $(\linop{B}, \bfn)$.
In general, one cannot find a closed-form expression for the solution $\solvv$
\citep{Borthwick2018IntroPDE}.
Therefore, we aim to compute an accurate approximation \(\hat{\solvv} \approx \solvv\)
instead.
Motivated by the challenges \labelcref{cha:limited-computation,cha:partially-known-physics,%
  cha:error-propagation} of partial information inherent to numerically solving PDEs, we
approach the problem from a statistical inference perspective.
In other words, we will \emph{learn} the solution of the PDE from multiple heterogeneous
sources of information.
This way we can quantify the epistemic uncertainty about the solution at any time during
the computation, as \cref{subfig:framework-illustration} illustrates.

\paragraph{Indirectly Observing the Solution of a PDE}
Typically, we think of observations as a finite number of direct measurements
\(\solvv(\vec{x}_i) = \vec{y}_i\) of the latent function \(\solvv\).
As it turns out, we can generalize this notion of a measurement and even interpret the
PDE itself as an (indirect) observation of \(\solvv\).
As an example, consider the important case where \(\solvv\) models the state of a
physical system.
The laws of physics governing such a system are often formulated as \emph{conservation
  laws} in the language of PDEs.
For example, they may require physical quantities like mass, momentum, charge or energy
to be conserved over time.

\begin{example}[Thermal Conduction and the Heat Equation]
  \label{ex:thermal-conduction-heat-equation}
  Say we want to simulate heat conduction in a solid object with shape $\dom \subset \R^3$,
  i.e.~we want to find the time-varying temperature distribution $\sol \colon [0, T] \times \dom \to \R$.
  Neglecting radiation and convection, \(\sol(t, \vec{x})\) is described by a linear PDE
  known as the \emph{heat equation} \citep{Lienhard2020HeatTransfer}.
  Assuming spatially and temporally uniform material parameters $c_p, \rho, \heatcond \in \R$,
  it reduces to
  \begin{equation}
    \label{eqn:heat-equation-isotropic-uniform}
    \left(c_p \rho \pderiv*[1]{t}{} - \heatcond \laplaceop{}\right) \sol - \heatrhs = 0.
  \end{equation}
  Thermal conduction is described by \(-\heatcond \laplaceop{\sol}\), while $\heatrhs$
  are local heat sources, e.g.~from electric currents.
  Any energy flowing into a region due to conduction or a heat source is balanced by an
  increase in energy of the material.
  The net-zero balance shows that energy is conserved.
\end{example}

Notice how a conservation law is an \emph{observation} of the behavior of the physical
system!
To formalize this, we begin by rephrasing the classical notion of an observation at a
point \(\vec{x}_i\) as measuring the result of a specific linear operator applied to the
solution \(\solvv\):
\begin{equation*}
  \solvv(\vec{x}_i) = \vec{y}_i
  \iff \evallinop{\delta_{\vec{x}_i}}{\solvv} = \vec{y}_i
\end{equation*}
where \(\delta_{\vec{x}_i}\) is the evaluation functional.
Now, the key idea is to generalize the notion of a direct observation to collecting
information about the solution via an arbitrary linear operator \(\linfctls{L}\) with
values in $\R^n$ applied to the solution \(\solvv\), such that
\(\linfctlsat{L}{\solvv} = \vec{y} \iff \linfctlsat{L}{\solvv} - \vec{y} = \vec{0}.\)
The affine operator
\begin{equation}
  \label{eqn:linpde-information-op}
  \linfctlsat{I}{\solvv} \defeq \linfctlsat{L}{\solvv} - \vec{y}
\end{equation}
is a specific kind of \emph{information operator} \citep{Cockayne2019BayesPNMeth}.
In this setting the information operator may describe a conservation law
as in \cref{eqn:heat-equation-isotropic-uniform}, a general linear PDE of the form
\eqref{eqn:linear-pde} or an arbitrary affine operator mapping a function space into
$\R^n$.
This generalized notion of an observation turns out to be very powerful to incorporate
different kinds of mathematical, physical, or experimental properties of the solution.
Since PDEs and conservation laws are often assumed to hold exactly, we focused on
noise-free observations above.
However, generally we are not limited to this case and can also model $\vec{y}$ as random
variable, in which case the information operator \(\linopat{I}{(\solvv, \vec{y})}\) is a
(jointly) linear functional of the solution $\solvv$ \emph{and} the right-hand side
$\vec{y}$.

\subsection{Solving PDEs as a Bayesian Inference Problem}
\label{sec:solving-pdes-bayesian-inference}

One of the main challenges \labelcref{cha:limited-computation,cha:partially-known-physics,%
  cha:error-propagation} outlined in the beginning is the limited computational budget
available to us to approximate the solution.
Fortunately, in practice, the solution $\solvv$ is not hopelessly unconstrained, but we
usually a-priori have information about it.
At the very least, we know the space of functions $\solsp$ in which to search for the
solution.
Additionally, we might have expert knowledge about its rough shape and value range, or
solutions to related PDEs at our disposal.
Now, the question becomes: How do we combine this prior knowledge with indirect
observations of the solution through the information operator $\linfctls{I}$
\labelcref{eqn:linpde-information-op}?
To do so, we turn to the Bayesian inference framework.
This provides a different perspective on the numerical problem of solving a linear PDE
as a \emph{learning task}.

\paragraph{Gaussian Process Inference}
We represent our belief about the solution of the linear PDE via a (multi-output)
Gaussian process
\(
\solvvprior \sim \gp{\vec{m}}{\mat{k}}
\)
with mean function $\vec{m} \colon \dom \to \R^{d'}$ and kernel
$\mat{k} \colon \dom \times \dom \to \R^{d' \times d'}$.
Gaussian processes are well-suited for this purpose since:
\begin{enumerate}[itemsep=0pt,label=(\roman*)]
  \item For an appropriate choice of kernel, the Gaussian process defines a probability
        measure over the function space in which the PDE's solution is sought.
  \item Kernels provide a powerful modeling toolkit to incorporate prior information
        (e.g.~variability, periodicity, multi-scale effects, in- / equivariances, \dots)
        in a modular fashion.
  \item Measurement noise often follows a Gaussian distribution.
  \item Conditioning a Gaussian process on observations made via a linear map again
        results in a Gaussian process.\label{itm:linear-gaussian-inference}
\end{enumerate}
While the result in \labelcref{itm:linear-gaussian-inference} is used ubiquitously in
the literature, its general form where observations are made via \emph{arbitrary} linear
operators with values in $\R^n$ as opposed to \emph{finite-dimensional} linear maps, has
only been rigorously demonstrated for Gaussian \emph{measures} on separable Hilbert
spaces, not for the Gaussian \emph{process} perspective, to the best of our knowledge.
The two perspectives are closely related, but there are thorny technical difficulties to
consider.
We intentionally frame the problem from the Gaussian process perspective to make use of
the expressive modeling capabilities provided by the kernel.
Our framework at its core relies on this result, which we explain in detail in
\cref{sec:affine-gp-inference} and prove in \cref{sec:proof-theorem-1}.

\subsubsection{Encoding Prior Knowledge about the Solution}
\label{sec:encoding-prior-knowledge}
We can infer the solution of a linear PDE more quickly by specifying inductive biases in
the prior, which can encode both provable and approximately known properties of the
solution.\footnotemark
\footnotetext{In the special case of GP regression, if the prior smoothness matches the
  smoothness of the target function \(\solvv\), the convergence rate is optimal in the
  number of observations \cite[Thm.~5.1]{Kanagawa2018GPKernMeth}.}

\paragraph{Function Space of the Solution}
The most basic known property derived from the PDE is an appropriate choice of function
space for the solution.
For strong solutions, this can be done by inspecting the differential operator
\(\linop{D}\) and keeping track of the partial derivatives.
In fact, in implementation this can be automatically derived solely from the problem
definition, e.g.~by compositionally defining differential operators and storing
information on the necessary differentiability.
Let \(\vecelem{\beta}_i \in \N_0\) be the number of times any partial derivative in the
differential operator \(\linop{D}\) differentiates w.r.t.~the variable $\vecelem{x}_i$.%
\footnote{%
  Formally, $\vec{\beta} \in \N_0^d$ is the ``smallest'' multi-index such that
  $\vec{\alpha} \le \vec{\beta}$ for every multi-index of a partial derivative occurring
  in $\linop{D}$ (see \cref{def:multi-index,def:linear-diffop}).
}
Then a sensible choice of solution space is the space $\solsp = \bucdfns{\vec{\beta}}{\dom}$ (see \cref{sec:gp-random-function}).
To define a prior with paths in this solution space, a common choice of prior covariance function is the tensor product of one-dimensional half-integer Matérn
kernels $k_{\vecelem{\nu}_i}$ with
$\vecelem{\nu}_i = \vecelem{\beta}_i + \frac{1}{2}$ (see \cref{sec:prior-selection}).
For weak solutions, the Sobolev spaces $\solsp = \sobolev{m}[\dom]$ are prototypical choices of solution spaces.
In this case, a (multivariate) Matérn kernel with smoothness parameter $\nu = m + \frac{1}{2}$ is a useful default prior covariance function.
In both cases, a parametric kernel $k(\vec{x}_0, \vec{x}_1) = \vec{\phi}(\vec{x}_0)\T \mat{\Sigma} \vec{\phi}(\vec{x}_1)$ is also a valid choice if $\vecelem{\phi}_i \in \solsp$.
See \cref{sec:priors-weak-solutions} for a detailed account on how to choose priors for physics-informed GP regression.

\paragraph{Symmetries, In- and Equivariances}
Many solutions of PDEs exhibit a-priori known symmetries.
For example, to calculate the strength of a magnet rotated by \(\mat{R} : \R^3 \to \R^3\),
one can equivalently compute the field $\vec{B}$ of the magnet in its original position
and rotate the field, i.e.~\(\vec{B}(\mat{R} \vec{x}) = \mat{R} \vec{B}(\vec{x})\).
Inductive biases reflecting symmetries can be encoded via kernels that are \emph{invariant}
\(
\mat{k}(\mat{\rho}_g \vec{x}_0, \mat{\rho}_g \vec{x}_1)
= \mat{k}(\vec{x}_0, \vec{x}_1),
\)
or \emph{equivariant}
\(
\mat{k}(\mat{\rho}_g \vec{x}_0, \mat{\rho}_g \vec{x}_1)
= \mat{\rho}_g \mat{k}(\vec{x}_0, \vec{x}_1)\mat{\rho}_g^*,
\)
where \(\mat{\rho}_g\) is a unitary group representation.
The most commonly used kernels are stationary, i.e.~translation invariant, but one can
also construct invariant \citep{Haasdonk2007Invariant,Azangulov2022Stationary},
as well as equivariant kernels \citep{Reisert2007LearnEquiv,Holderrieth2021Equivariant} for
many other group actions.

\paragraph{Related Problems}
Given a set of solutions from related problems, the prior mean function can be set to
a combination thereof and the prior kernel can then be chosen to reflect how related the problems are.
For example, if we have an approximate solution of a PDE computed on a coarser
mesh, we can condition our function space prior on the coarse solution with a noise level reflecting the fidelity of the discretization.
Similarly, if we solved the same PDE with different parameters, we can condition on the
available solutions with a noise level chosen according to how similar the parameters
are to the one of interest.

\paragraph{Domain Expertise}
Domain experts often have approximate knowledge of what solutions can be expected,
either from experience, previous experiments, or familiarity with the physical
interpretation of the solution \(\solvv\). 
For example, an engineer who designs electrical components is likely able to give
realistic temperature ranges for a component for which we aim to simulate the temperature distribution.
This can be included by choosing the (initial) kernel hyperparameters, such as the
output- and lengthscales based on this expertise.

\subsubsection{(Indirectly) Observing the Solution}
\label{sec:indirect-observations}
From a computational perspective, the most important reason for choosing Gaussian
processes is that when conditioning on linear observations, the resulting posterior is
again a Gaussian process with closed form mean and covariance function
\citep{Bishop2006PRML}.
We extend this classic result from observations via a finite-dimensional linear map to
general $\R^n$-valued linear operators in \Cref{thm:gp-inference-linfctls}.
This is crucial to condition on the different types of observations, most importantly
the PDE itself, made via the information operator in \eqref{eqn:linpde-information-op}.
Given such an affine observation defined via a linear operator
\(\linfctls{L} \colon \solsp \to \R^n\)
and an independent Gaussian random variable
\(\rvec{\epsilon} \sim \gaussian{\vec{\mu}}{\mat{\Sigma}},\)
we can condition our prior belief using \cref{thm:gp-inference-linfctls} on the
observations to obtain a posterior of the form
\(
\condrv*{%
  \solvvprior
  \given
  (\linfctlsat{L}{\solvvprior} + \rvec{\epsilon} = \vec{y})
}
\sim \gp{%
\vec{m}^{\condrv{\solvvprior \given \vec{y}}}
}{%
\mat{k}^{\condrv{\solvvprior \given \vec{y}}}
}
\)
with mean and covariance function given by
\begin{align}
  \label{eqn:linpde-posterior-mean}
  \vecelem{m}^{\condrv{\solvvprior \given \vec{y}}}_i(\vec{x})
   & = \vecelem{m}_i(\vec{x}) +
  \linfctlsat{L}{\matelem{k}_{:,i}(\cdot, \vec{x})}\T
  \left( \LkL{\linfctls{L}}{\mat{k}} + \mat{\Sigma} \right)\pinv
  \left( \vec{y} - (\linfctlsat{L}{\vec{m}} + \vec{\mu}) \right), \\
  \label{eqn:linpde-posterior-covariance}
  \matelem{k}^{\condrv{\solvvprior \given \vec{y}}}_{i,j}(\vec{x}_1, \vec{x}_2)
   & = \matelem{k}_{i,j}(\vec{x}_1, \vec{x}_2) +
  \linfctlsat{L}{\matelem{k}_{:,i}(\cdot, \vec{x}_1)}\T
  \left( \LkL{\linfctls{L}}{\mat{k}} + \mat{\Sigma} \right)\pinv
  \linfctlsat{L}{\matelem{k}_{:,j}(\cdot, \vec{x}_2)}.
\end{align}
We will now look more closely at how we can condition on the boundary conditions, the
PDE itself and direct measurements of the solution.

\paragraph{Observing the Solution via the PDE}
The differential operator $\linop{D}$ in \cref{eqn:linear-pde} is linear and therefore
it is tempting to define the information operator
$\linopat{I}{\solvvprior} = \linopat{D}{\solvvprior} - f$
and attempt to condition on $\linopat{I}{\solvvprior} = 0$.
Under some assumptions on $\solsp$, $\linop{D}$, and $\solvvprior$, one can even show
that this is well-defined.
Unfortunately, it turns out that computing the posterior moments is then at least as
hard as solving the PDE directly and thus typically intractable in practice.
Loosely speaking, this is because $\rhs$ is a function and hence $\linopat{D}{\solvvprior} = \rhs$
corresponds to an infinite number of observations.
However, by only enforcing the PDE at a finite number of points in the domain, we can
immediately give a canonical example of an approximation to this intractable information
operator.
Concretely, we can condition $\solvvprior$ on the fact that the PDE holds at a finite
sequence of well-chosen domain points
$\mat{X}_\text{PDE} = (\vec{x}_i)_{i = 1}^n \in \dom^n,$
i.e.~we compute
\(
\condrv{%
  \solvvprior
  \given
  (\linopat{D}{\solvvprior}(\mat{X}_\text{PDE}) - f(\mat{X}_\text{PDE}) = 0)
}
\)
by choosing \(\linfctls{L}=\delta_{\mat{X}_\text{PDE}} \circ \linop{D}\) and
\(\vec{y} = f(\mat{X}_\text{PDE})\).
If the set $\mat{X}_\text{PDE}$ of domain points is dense enough, we obtain a
good approximation to the exact conditional process.
This approach, known as the \emph{probabilistic meshless method} \citep{Cockayne2017PNPDEInv},
is analogous to existing non-probabilistic approaches to solving PDEs, commonly referred
to as \emph{collocation methods}, wherein the points $\mat{X}$ are called \emph{collocation %
  points}.
Satisfying the PDE at a set of collocation points is far from the only choice within our
general framework.
For example, we can choose a set of test functions
\(\mwrtestfctl^{(i)} \in \dualsp{\rhssp},\)
which we use to observe the PDE with, such that
\(
\evallinop{\linfctlselem{L}_i}{\solvvprior}
= \evallinop{\mwrtestfctl^{(i)}}{\linopat{D}{\solvvprior}}
\)
and \( \vec{y}_i = \evallinop{\mwrtestfctl^{(i)}}{f} \).
For efficient evaluation of the differential operator we can further represent the
solution in a basis of trial functions from a subspace \(\mwrtrialsp\), resulting in
\(
\evallinop{\linfctlselem{L}_i}{\solvvprior}
= \evallinop{\mwrtestfctl^{(i)}}{\linopat{D}{\evallinop{\mwrtrialproj}{\solvvprior}}}
\).
This turns out to be very powerful and is analogous to some of the most successful
classical PDE solvers.
In fact, for certain priors and choices of subspaces, our framework recovers several
important classic solvers in the posterior mean (see \Cref{sec:gp-mean-mwr-recovery}).
The above can be applied to both time-dependent and time-independent PDEs and regardless
of the type of linear PDE (e.g.~elliptic, parabolic, hyperbolic).
Moreover, an extension to systems of linear PDEs is straightforward.

\paragraph{Observing the Solution at the Boundary}
As for the PDE, we could attempt to directly condition on the boundary conditions by
choosing \(\linfctls{L} = \linop{B}\) and \(\vec{y} = g\).
However, we are faced with the same intractability issues that we discussed above.
Instead, we observe that the boundary conditions hold at a finite set of points
\(\mat{X}_{\text{BC}} \subset \boundary{\dom}\), i.e.~%
\(\linop{L}=\delta_{\mat{X}_{\text{BC}}} \circ \mathcal{B}\)
and \(\vec{y} = g(\mat{X}_{\text{BC}})\).
In practice, sometimes the boundary conditions are only known at a finite set of points
making this a natural choice.

\paragraph{Observing the Solution Directly}
Finally, as in standard GP regression, we can directly condition on (noisy) measurements
of the solution, for example from a real world experiment, by choosing
\(\linfctls{L} = \delta_{\mat{X}_\text{MEAS}}\) and \(\vec{y} = \truesolvv(\mat{X}_\text{MEAS})\).\\

\par
\noindent In summary, the probabilistic viewpoint allows us to
\begin{itemize}[itemsep=0pt]
  \item encode prior information about the solution,
  \item condition on various kinds of (partial) information, such as the boundary condition, the PDE itself, or direct measurements, and
  \item output a structured error estimate, reflecting all obtained information and performed computation.
\end{itemize}
We will now give concrete examples for some of the possible modeling choices described above in a case study.

\subsection{Case Study: Modeling the Temperature Distribution in a CPU}
\label{sec:cpu-stationary-1d}
Central processing units (CPUs) are pieces of computing hardware that are constrained by
the vast amounts of heat they dissipate under computational load.
Surpassing the maximum temperature threshold of a CPU for a prolonged period of time can
result in reduced longevity or even permanent hardware damage \citep{Michaud2019CPUTemp}.
To counteract overheating, cooling systems are attached to the CPU, which are controlled
by digital thermal sensors (DTS).
For simplicity, assume that the CPU is under sustained computational load and that the
cooling device is controlled in a way such that the die reaches thermal equilibrium.
\begin{example}[Stationary Heat Equation]
  \label{ex:stationary-heat-equation}
  The temperature distribution of a solid at thermal equilibrium, i.e.~%
  $\pderiv*[1]{t}{\sol} = 0$ in \Cref{ex:thermal-conduction-heat-equation}, is described
  by the linear PDE
  \begin{equation}
    \label{eqn:stationary-heat-equation-isotropic-uniform-residual}
    - \heatcond \laplaceop{\sol} - \heatrhs = 0,
  \end{equation}
  known as the \emph{stationary heat equation} \citep{Lienhard2020HeatTransfer}.
  For our choice of material parameters \cref{eqn:stationary-heat-equation-isotropic-uniform-residual}
  is equivalent to the Poisson equation with $\rhs = \frac{\heatrhs}{\heatcond}$.
\end{example}
While the sensors control cooling, they only provide local, limited-precision
measurements of the CPU temperature.
This is problematic, since the chip may reach critical temperature thresholds in
unmonitored regions.
Therefore, our goal will be to infer the temperature in the entire CPU.
We will use our framework to integrate the physics of heat flow, the controlled cooling
at the boundary, and the noisy temperature measurements from the sensors.
See \cref{fig:cpu-stationary-2d} for an illustration of the result.
%
During manufacturing, the resulting belief over the temperature distribution could then
help decide whether the CPU design needs to be changed to avoid premature failure.
From here on out, we focus on a 1D slice across the CPU surface, as shown in
\cref{fig:cpu-stationary-1d-geometry-rhs} (top), to easily visualize uncertainty.

\begin{figure}[t]
  \centering
  \subcaptionbox{%
    Top: CPU die with CPU cores as heat sources and uniform cooling over the whole
    surface.\\
    Bottom: Magnitude of heat sources and sinks \(\heatrhs\) in the 1D slice in the
    upper subplot (\textcolor{orange}{\textbf{---}}).
    \label{fig:cpu-stationary-1d-geometry-rhs}
  }{%
    \includegraphics[width=0.44\textwidth]{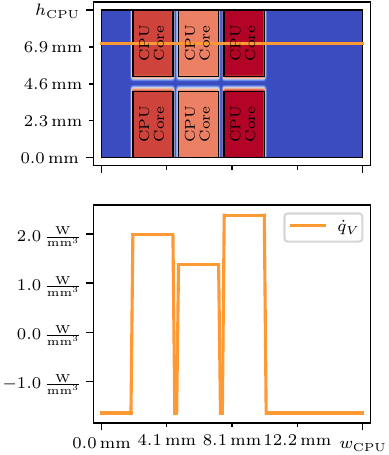}
  }
  \hfill
  \subcaptionbox{%
    Gaussian process integrating \emph{prior information} about the temperature
    distribution, a \emph{mechanistic model} of heat conduction in the form of a linear
    PDE, and \emph{empirical measurements} $(\mat{X}_\text{DTS}, \vec{y}_\text{DTS})$
    taken by limited-precision sensors (DTS).
    The plot shows the GP mean and a 1D slice illustrating the posterior uncertainty
    along with a few samples.
    \label{fig:cpu-stationary-2d}
  }{%
    \includegraphics[width=0.53\textwidth]{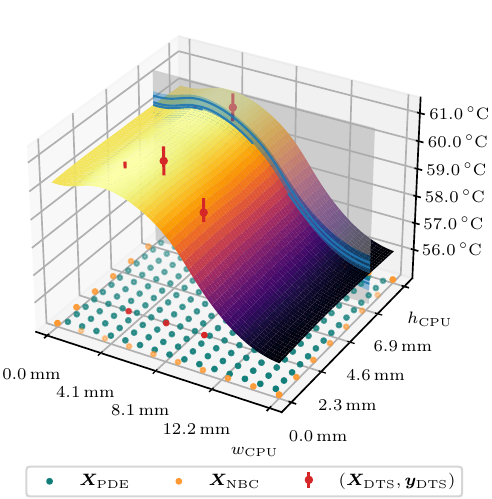}
  }
  \caption{%
    Physics-informed Gaussian process model of the stationary temperature distribution
    in an idealized hexa-core CPU die under sustained computational load.
  }
\end{figure}

\paragraph{Encoding Prior Knowledge}
By inspecting the PDE's differential operator
\(
\linop{D}
= -\heatcond \laplaceop{}
= -\heatcond \sum_{i = 1}^d \pderiv*[2]{\vecelem{x}_i}{},
\)
we can deduce that the paths of our Gaussian process need to be twice-differentiable in
every input variable $\vecelem{x}_i$.
The construction in \cref{sec:priors-gp-inference-linfctls} tells us that that a GP
prior whose covariance function is a tensor product $k_{\vec{\nu}}$ of one-dimensional
Matérn kernels $k_{\vecelem{\nu}_i}$ with \(\vecelem{\nu}_i = 2 + \frac{1}{2} = \frac{5}{2}\)
fulfills the desired path properties.
Assume we also know what temperature ranges are plausible from similar CPU
architectures, meaning we set the kernel output scale to \(\sigma_\text{out}^2=9\).
\Cref{fig:cpu-stationary-1d-prior} shows the prior process $\solprior$ on along with its
image $\linopat{D}{\solprior} \sim \gp{\linopat{D}{m}}{\sigma_\text{out}^2 \LkL{\linop{D}}{k}}$
under the differential operator.
A draw from $\linopat{D}{\solprior}$ can be interpreted as the heat sources
and sinks that generated the corresponding temperature distribution draw from $\solprior$.
\begin{figure}[t]
  \centering
  \subcaptionbox{%
    Gaussian process prior with a Mat\'ern-$\frac{5}{2}$ kernel over the temperature
    distribution of the CPU.
  }{%
    \includegraphics[width=0.48\textwidth]{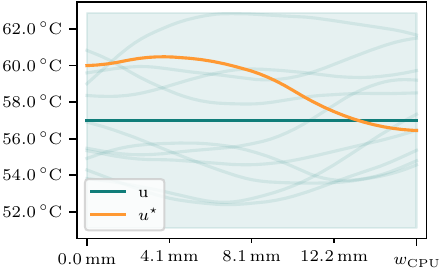}
  }
  \hfill
  \subcaptionbox{%
    Prior under the differential operator $\linop{D} = - \heatcond \laplaceop{}$
    along with heat sources and sinks $\heatrhs$.
  }{%
    \includegraphics[width=0.48\textwidth]{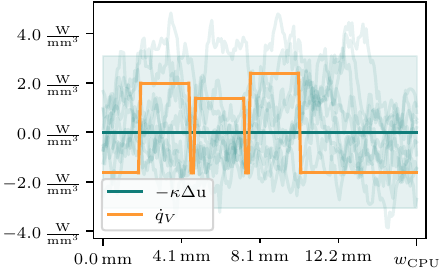}
  }
  \caption{%
    Prior model for the stationary temperature distribution of a CPU die under load.
    \label{fig:cpu-stationary-1d-prior}
  }
\end{figure}

\paragraph{Conditioning on the PDE}

We can now inform our belief about the physics of heat conduction using the mechanistic
model defined by the stationary heat equation. We choose a set of collocation points
$\mat{X}_\text{PDE} \in \dom^n$ and then condition on the observation that the PDE holds
(exactly) at these points.
In other words, we compute the physically-informed Gaussian process
\(
\solprior \mid \text{PDE}
\defeq
\solprior \mid (\evallinop{\linfctls{I}^\text{PDE}}{\solprior} = \vec{0})
\)
with
\(
\evallinop{\linfctls{I}^\text{PDE}}{\solprior}
\defeq -\heatcond \laplaceop{\solprior}[\mat{X}_{\text{PDE}}] - \heatrhs(\mat{X}_{\text{PDE}})
\)
visualized in \cref{fig:cpu-stationary-1d-cond-pde}.
\begin{figure}[t]
  \centering
  \subcaptionbox{%
    Belief about the solution after conditioning on the PDE at a set of collocation
    points.
    \label{subfig:cpu-stationary-1d-cond-pde-belief-solution}
  }{%
    \includegraphics[width=0.48\textwidth]{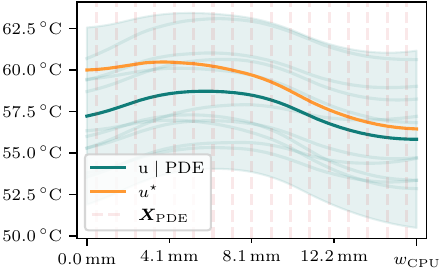}
  }
  \hfill
  \subcaptionbox{%
    Belief about heat sources and sinks after conditioning on the PDE at collocation
    points.
    \label{subfig:cpu-stationary-1d-cond-pde-belief-heat-sources-sinks}
  }{%
    \includegraphics[width=0.48\textwidth]{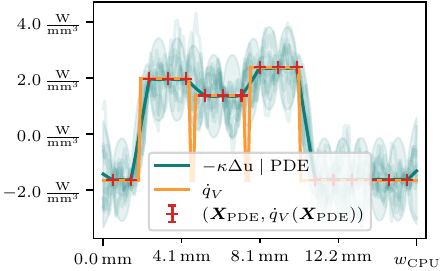}
  }
  \caption{%
  We integrate mechanistic knowledge about the system by conditioning on PDE observations
  \(
  -\heatcond \laplaceop{\solprior}[\mat{X}_{\text{PDE}}] - \heatrhs(\mat{X}_{\text{PDE}})
  = \vec{0}
  \)
  at the collocation points $\mat{X}_{\text{PDE}}$, resulting in the conditional process
  \(
  \condrv{\solprior \given \text{PDE}}
  \).
  The large remaining uncertainty in \cref{subfig:cpu-stationary-1d-cond-pde-belief-solution}
  illustrates that the PDE by itself does not identify a unique solution.
  \label{fig:cpu-stationary-1d-cond-pde}
  }
\end{figure}
We can see that the resulting conditional process indeed satisfies the PDE exactly at
the collocation points (see \cref{subfig:cpu-stationary-1d-cond-pde-belief-heat-sources-sinks}).
The remaining uncertainty in \cref{subfig:cpu-stationary-1d-cond-pde-belief-heat-sources-sinks}
is due to the approximation error introduced by only conditioning on a finite number of
collocation points.
However, while the samples from our belief about the solution in \cref{subfig:cpu-stationary-1d-cond-pde-belief-solution}
exhibit much more similarity to the mean function and less spatial variation, the
marginal uncertainty hardly decreases.
The latter is explained by the PDE not identifying a unique solution, since adding any
affine function to $\truesol$ does not alter its image under the differential operator,
i.e.~$\laplaceop{(\vec{a}\T \vec{x} + b)} = 0$.
There is an at least two-dimensional subspace of functions which can
not be observed.
This ambiguity can be resolved by introducing boundary conditions.

\paragraph{Conditioning on the Boundary Conditions}
\label{sec:boundary-conditions}
We assume that the CPU cooler extracts heat uniformly from all exposed parts of the CPU,
in particular also from the sides, rather than just from the top.
Instead of directly specifying the value of the temperature distribution at the edge
points of the CPU slice, we only know the density $\heatbfn$ of heat flowing out of each
point on the CPU's boundary based on the cooler specification.
We can use another thermodynamical law to turn this assumption into information about
the temperature distribution \(\sol\).
\begin{example}[continues=ex:thermal-conduction-heat-equation]
  Fourier's law states that the local density of heat $\heatbfn$ flowing through
  a surface with normal vector $\vec{\eta}$ is proportional to the inner product of the
  negative temperature gradient and the surface normal $\vec{\eta}$, i.e.~%
  \(
  \heatbfn = -\heatcond \inprod{\vec{\eta}}{\gradient{\sol}},
  \)
  where $\heatcond$ is the material's thermal conductivity in $\si{\watt\per\meter\kelvin}$
  \citep{Lienhard2020HeatTransfer}.
\end{example}
Assuming sufficient differentiability of $\sol$, the inner product above is equal to the
directional derivative $\dderiv{\eta}{\sol}$ of $\sol$ in direction $\eta$.
We can assign an outward-pointing vector $\eta(x)$ (almost) everywhere on the
boundary of the domain.
Since the boundary of the CPU domain is its surface, we can summarize the above in a
Neumann boundary condition $-\heatcond \dderivat{\eta(x)}{\sol}{x} = \heatbfn(x)$ for
$x \in \boundary{\dom}$.
Applying \cref{cor:gp-inference-linop-evals} once more, we can inform our estimate of
the solution about the boundary conditions by computing
\(
\condrv{\solprior \given \text{PDE}, \text{NBC}}
\defeq
\condrv*{
  \left( \condrv{\solprior \given \text{PDE}} \right)
  \given
  \evallinop{\linfctls{I}^\text{NBC}}{(\solprior, \heatbfnprior)} = \vec{0}
},
\)
where
\(
\evallinop{\linfctls{I}^\text{NBC}}{\solprior}
\defeq
- \heatcond \dderivat{\eta(\mat{X}_\text{NBC})}{\solprior}{\mat{X}_\text{NBC}}
- \heatbfnprior(\mat{X}_\text{NBC})
\)
with $\mat{X}_\text{NBC} = \set{0, w_\text{CPU}}$ is the information operator induced by
the boundary conditions.
The result is visualized in \cref{subfig:cpu-stationary-1d-cond-pde-nbc-belief-solution}.
The structure of the samples illustrates that most of the remaining uncertainty about
the solution lies in a one-dimensional subspace of $\solsp$ corresponding to constant
functions.
This is due to the fact that two Neumann boundary conditions on both sides of the
domain only determine the solution of the PDE up to an additive constant.
We need an additional source of information to address the remaining degree of
freedom.

\paragraph{Conditioning on Direct Measurements}
\label{sec:solution-observations}
Fortunately, CPUs are equipped with digital thermal sensors (DTS) located close to each
of the cores, which provide (noisy) local measurements of the core temperatures
\citep{Michaud2019CPUTemp}.
These measurements can be straightforwardly accounted for in our model by performing
standard GP regression using $\condrv{\solprior \given \text{PDE}, \text{NBC}}$ from
\cref{subfig:cpu-stationary-1d-cond-pde-nbc-belief-solution} as a prior.
The resulting belief about the temperature distribution is visualized in
\cref{subfig:cpu-stationary-1d-cond-pde-nbc-dts-belief-solution}.
\begin{figure}[t]
  \centering
  \subcaptionbox{%
    Belief about the solution after conditioning on the PDE and boundary conditions
    (BCs).
    \label{subfig:cpu-stationary-1d-cond-pde-nbc-belief-solution}
  }{%
    \includegraphics[width=0.48\textwidth]{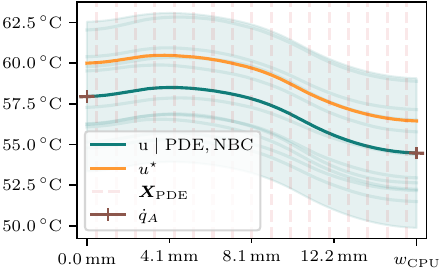}
  }
  \hfill
  \subcaptionbox{%
    Belief about the solution after conditioning on the PDE, BCs and noisy sensor data.
    \label{subfig:cpu-stationary-1d-cond-pde-nbc-dts-belief-solution}
  }{%
    \includegraphics[width=0.48\textwidth]{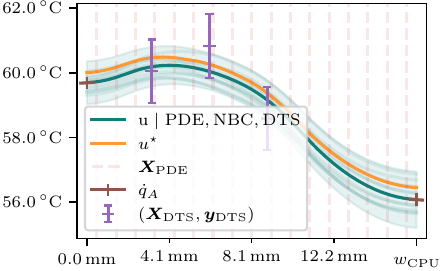}
  }
  \hfill
  \caption{%
  Neumann boundary conditions encoding mechanistic knowledge about the heat flux across
  the boundary of the CPU and a sparse set of limited-precision measurements of the
  temperature distribution made by digital thermal sensors (DTS) located at the points
  \(\mat{X}_{\text{DTS}}\) further constrain the solution of the PDE.
  The remaining uncertainty is due to measurement noise and discretization error.
  }
\end{figure}
We can see that integrating the interior measurements effectively reduces the
uncertainty due to the remaining degree of freedom, albeit not completely.
The remaining uncertainty is due to the model's consistent accounting for noise in the
thermal sensor readings, the uncertainty about the cooling, i.e.~the boundary
conditions, and the discretization error incurred by only choosing a small set of
collocation points.

\paragraph{Uncertainty in the Right-hand Side and the Boundary Function}
\label{sec:uncertain-right-hand-side}
Above, we assumed the true heat source term $\heatrhs$, i.e.~the right-hand side of the
PDE, and the boundary heat flux $\heatbfn$ to be known exactly.
However, in practice, this is rarely the case.
Fortunately, our probabilistic viewpoint admits a straightforward relaxation of this
assumption.
Namely, we can replace $\heatrhs$ and $\heatbfn$ by a joint Gaussian process prior
$(\heatrhsprior, \heatbfnprior)$, whose means are given by estimates of $\heatrhs$ and
$\heatbfn$.\footnotemark
\footnotetext{
  Technically speaking, if the right-hand-side of the PDE is given as a Gaussian
  process, the PDE turns into a stochastic partial differential equation (SPDE).
}
Above, we assumed that the cooler is controlled in such a way, that the temperature
distribution in the CPU does not change over time.
However, a naive prior $(\heatrhsprior, \heatbfnprior)$ may break this assumption.
We need to encode that the amount of heat entering the CPU is equal to the amount of
heat leaving the CPU via its boundary, i.e.~%
\begin{equation}
  \label{eqn:stationarity-constraint-2d}
  \evallinop{\linop{I}^\text{STAT}}{(\heatrhsprior, \heatbfnprior)}
  \defeq
  \rintegral[\dom]{\vec{x}}{\heatrhsprior(\vec{x})}
  - \rintegral[\boundary{\dom}]{A}{\heatbfnprior(\vec{x})}
  = 0,
\end{equation}
The (jointly) linear information operator $\linop{I}^\text{STAT}$ computes the net
amount of thermal energy that the CPU gains per unit time.
Using \cref{thm:gp-inference-linfctls} we can construct a multi-output GP prior
\((\solprior, \heatrhsprior, \heatbfnprior)\), which is consistent with the assumption
of thermal stationarity by conditioning on
$\evallinop{\linop{I}^\text{STAT}}{(\heatrhsprior, \heatbfnprior)} = 0$.
Here, we assume a-priori that $\solprior$, $\heatrhsprior$, and $\heatbfnprior$ are
pairwise independent.
In the one-dimensional model, we can simplify \cref{eqn:stationarity-constraint-2d}
by assuming that heat is drawn uniformly from the sides of the CPU.
By encoding this information in the prior $\heatbfnprior$, the information operator
corresponding to thermal stationarity resolves to
\begin{equation}
  \label{eqn:stationarity-constraint-1d}
  \evallinop{\linop{I}^\text{STAT}}{(\heatrhsprior, \heatbfnprior)}
  =
  h_\text{CPU} \rintegral[0][w_\text{CPU}]{x}{\heatrhsprior(x)}
  - h_\text{CPU} \left( \heatbfnprior(0) + \heatbfnprior(w_\text{CPU}) \right).
\end{equation}

As above, we can now use \cref{cor:gp-inference-linop-evals} to condition our
physically-consistent GP prior
\(
\condrv{(\solprior, \heatrhsprior, \heatbfnprior) \given \text{STAT}}
\defeq
\condrv*{
  (\solprior, \heatrhsprior, \heatbfnprior)
  \given
  \left( \evallinop{\linop{I}^\text{STAT}}{(\heatrhsprior, \heatbfnprior)} = 0 \right)
}
\)
on
$\evallinop{\linfctls{I}^\text{PDE}}{(\solprior, \heatrhsprior)} = \vec{0},$
as well as
$\evallinop{\linfctls{I}^\text{NBC}}{(\solprior, \heatbfnprior)} = \vec{0}$
and the noisy measurements of the temperature distribution.
Here, it is important to keep track of the cross-covariances in
$(\solprior, \heatrhsprior, \heatbfnprior),$
since the outputs in $\condrv{(\heatrhsprior, \heatbfnprior) \given \text{STAT}}$ become
correlated.
The resulting process
$\condrv{\solprior \given \text{PDE}, \text{NBC}, \text{STAT}, \text{DTS}}$
(or rather its marginals) is shown in
\cref{fig:cpu-stationary-1d-cond-pde-nbc-stat-dts}.
\begin{figure}[t]
  \centering
  \subcaptionbox{%
    Posterior belief about the temperature distribution physically consistent with the assumption of stationarity.
    \label{subfig:cpu-stationary-1d-cond-pde-nbc-stat-dts-belief-solution}
  }{%
    \includegraphics[width=0.47\textwidth]{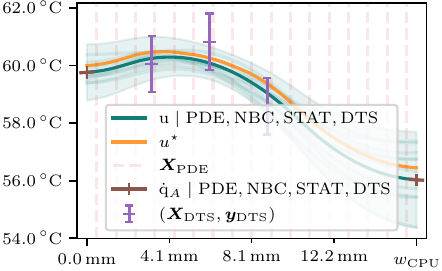}
  }
  \hfill
  \subcaptionbox{%
    Posterior belief about the heat sources and sinks after conditioning on the corresponding uncertain right-hand-side \(\heatrhsprior\) of the PDE.
    \label{subfig:cpu-stationary-1d-cond-pde-nbc-stat-dts-belief-heat-sources-sinks}
  }{%
    \includegraphics[width=0.47\textwidth]{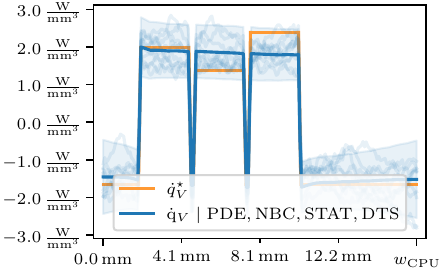}
  }
  \caption{%
    We integrate information from the joint prior
    $\condrv{(\solprior, \heatrhsprior, \heatbfnprior) \given \text{STAT}}$
    over the solution, the right-hand side of the PDE, and the values of the Neumann
    boundary conditions into our belief about the temperature distribution by
    conditioning on said PDE and boundary conditions.
    \label{fig:cpu-stationary-1d-cond-pde-nbc-stat-dts}
  }
\end{figure}
Comparing \cref{subfig:cpu-stationary-1d-cond-pde-nbc-stat-dts-belief-solution,%
  subfig:cpu-stationary-1d-cond-pde-nbc-dts-belief-solution}, we can see that, due to
the uncertainty in the right-hand side $\heatrhsprior$ of the PDE, the samples of
\(
\condrv{\solprior \given \text{PDE}, \text{NBC}, \text{STAT}, \text{DTS}}
\)
exhibit much more spatial variation.
Moreover, the samples of the GP posterior over $\heatrhsprior$ fulfill the stationarity
constraint we imposed.

\begin{figure}[t]
  \centering
  \begin{minipage}[c]{0.45\textwidth}
    \begin{tikzpicture}[
    x=1pt,
    y=1pt,
    yscale=-1,
    line width=0.75pt,
    graphvariable/.style={
        circle,
        draw=black,
        minimum size=45,
        inner sep=0,
        align=center,
      },
    unobserved/.style={
        graphvariable,
      },
    observed/.style={
        graphvariable,
        fill=lightgray,
      },
  ]

  \node[unobserved] (u) at (0.0, 0.0) {$\solprior$};
  \node[unobserved, left=15.0 of u] (f) {$\heatrhsprior$};
  \node[unobserved, right=15.0 of u] (g) {$\heatbfnprior$};

  \node[observed, below=15.0 of f, font=\footnotesize] (pde) {$\linfctls{I}^{\text{PDE}}$};
  \node[observed, below=15.0 of g, font=\footnotesize] (boundary) {$\linfctls{I}^{\text{NBC}}$};
  \node[observed, below=15.0 of u, font=\footnotesize] (u_X) {$\solprior(\mat{X}_\text{DTS})$\\$+$\\$\rvec{\epsilon}_\text{DTS}$};

  \node[observed, above=15.0 of u, font=\footnotesize] (stat) {$\linop{I}^{\text{STAT}}$};

  \draw[->] (u.south west) -- (pde.north east);
  \draw[->] (u.south east) -- (boundary.north west);
  \draw[->] (u.south) -- (u_X.north);

  \draw[->] (f.south) -- (pde.north);
  \draw[->] (f.north) to[out=-90, in=180] (stat.west);

  \draw[->] (g.south) -- (boundary.north);
  \draw[->] (g.north) to[out=-90, in=0] (stat.east);
\end{tikzpicture}
  \end{minipage}%
  \begin{minipage}[c]{0.54\textwidth}
    \caption{%
      Representation of the CPU model as a directed graphical model.
      The inference procedure described in \cref{sec:uncertain-right-hand-side} is
      equivalent to the \emph{junction tree algorithm} \citep[Section 8.4.6]{Bishop2006PRML}
      applied to the graphical model above.
      This example shows that the language of information operators is a powerful tool for
      aggregating heterogeneous sources of partial information in a joint
      probabilistic model.
    }
    \label{fig:cpu-stationary-graphical-model}
  \end{minipage}
\end{figure}
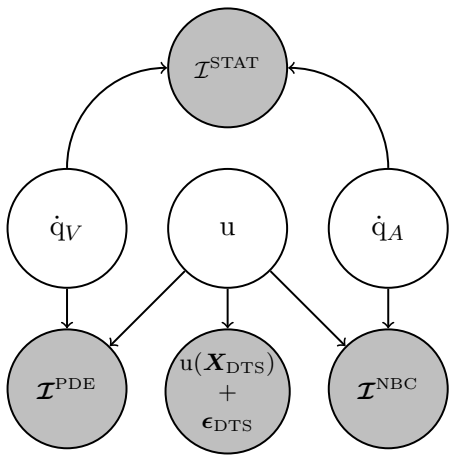

\paragraph{Summary}
Stepping back, we can view the problem of modeling the CPU under computational load as
a scientific inference problem, where we need to aggregate heterogeneous sources of
information in a joint probabilistic model.
This inference task is illustrated as a directed graphical model in \cref{fig:cpu-stationary-graphical-model}.
Our physics-informed regression framework is a local computation in the global inference procedure on the graph.
Importantly, its implementation does not change based on what happens to the solution
estimate and the input data in either upstream or downstream computations.
All this information is already handily encoded in the structured uncertainties of the
Gaussian processes.

\subsection{A General Class of Tractable Information Operators for Linear PDEs}
\label{sec:mwr-info-ops}
Recall that conditioning on the linear PDE directly via the information operator
\(
\linopat{I}{\solvvprior} = \linopat{D}{\solvvprior} - f
\)
is usually intractable.
Instead, in \cref{sec:indirect-observations} we approximated this information operator by
\( \linfctls{I} \colon \solsp \to \R^n \)
with
\(
\evallinop{\linfctlselem{I}_i}{\solvvprior}
\defeq \linopat{D}{\solvvprior}(\vec{x}_i) - f(\vec{x}_i)
\)
where $\vec{x}_i \in \dom$.
This implicitly assumes that point evaluation on both $\linopat{D}{\solvvprior}$ and
$\rhs$ is well-defined, which crucially means that this approach applies \emph{only}
to strong solutions of PDEs.
In this \lcnamecref{sec:mwr-info-ops}, we extend this approximation scheme to a
general class of tractable information operators aimed at approximating \emph{both} weak
and strong solutions to linear PDEs.
Our framework is inspired by the method of weighted residuals (MWR) (see
\cref{sec:mwr}).
In fact, in \cref{sec:gp-mean-mwr-recovery} we will show that GP inference with
information operators in this class reproduces any weighted residual method in the
posterior mean while additionally providing an estimate of the inherent approximation error.

\todo[author=Jonathan]{The paragraphs below need some cleanup. They seem quite disconnected. Give the reader some context of what they are reading next, e.g. by adding a paragraph header ``Notation'', or better connecting the different paragraphs.}
In the following, we will consider both weak and strong formulations of linear PDEs,
which is why we introduce the unifying notation
\(
\evallinop{\mwrwsdiffop}{\solvv} = \mwrwsrhs.
\)
For a strong formulation, $\mwrwsdiffop \defeq \diffop$, where
$\diffop \colon \solsp \mapsto \rhssp$
is a linear differential operator (see \cref{def:linear-diffop}), and
$\mwrwsrhs \defeq \rhs \in \rhssp$
is the right-hand side function.
In the context of a weak formulation, $\mwrwsdiffop \defeq \weakdiffop$,
where $\weakdiffop \colon \solsp \mapsto \dualsp{\rhssp}$ is the weak differential
operator, and $\mwrwsrhs \defeq \weakrhs \in \dualsp{\rhssp}$ is the right-hand side
functional (see \cref{sec:weak-formulation}).
Following \cref{sec:mwr}, we will apply linear functionals to the PDE residual.
To facilitate notation, we define the shorthand $\mwrwstestfctlsp$ for the space of
continuous linear functionals on the image space of $\mwrwsdiffop$, i.e.~%
$\mwrwstestfctlsp \defeq \dualsp{\rhssp}$ in the context of a strong formulation and
$\mwrwstestfctlsp \defeq \bidualsp{\rhssp}$ in the context of a weak formulation.
We additionally require that $\mwrtestfctl \circ \mwrwsdiffop$ is continuous for every
$\mwrtestfctl \in \mwrwstestfctlsp$.
\todo[author=Marvin]{Maybe mention why these assumptions are typically fulfilled in
  practice, e.g.~via (Lions-)Lax-Milgram assumptions}

Let $\solvvprior \sim \gp{\vec{m}}{\mat{k}}$ be a Gaussian process prior over the
solution $\solvvprior$ of the PDE, whose path space can be continuously embedded into
the solution space $\solsp$ (see \cref{sec:prior-selection} for more details on the
latter assumption).
It is intractable to condition the GP prior
on the full information provided by the PDE via the family
$\{ \mwrinfoop{\mwrtestfctl} \}_{\mwrtestfctl \in \mwrwstestfctlsp}$
of affine information operators
\(
\evallinop{\mwrinfoop{\mwrtestfctl}}{\solvvprior}
\defeq
\evallinop{(\mwrtestfctl \circ \mwrwsdiffop)}{\solvvprior} - \evallinop{\mwrtestfctl}{\mwrwsrhs},
\)
since $\mwrwstestfctlsp$ is typically infinite-dimensional.
To identify tractable families of information operators, we take inspiration from the
method of weighted residuals.
\todo[author=Jonathan]{there is some duplication here about intractability with the intro of this section}

\subsubsection{Infinite-Dimensional Trial Function Spaces}
\label{sec:mwr-info-ops-inf-trial}
Using \cref{thm:gp-inference-linfctls} we can tractably condition on a finite subfamily
\(
\{ \mwrinfoop{\mwrtestfctl^{(i)}} \}_{i = 1}^n
\subset
\{ \mwrinfoop{\mwrtestfctl} \}_{\mwrtestfctl \in \mwrwstestfctlsp}
\)
of information operators, where
\(
\set{\mwrtestfctl^{(i)}}_{i = 1}^n \subset \mwrwstestfctlsp
\)
is a finite subset of test functionals, as long as we can compute
\(
\evallinop{\mwrinfoop{\mwrtestfctl^{(i)}}}{\vec{m}},
\)
\(
\evallinop{\linop{L}_{\mwrtestfctl^{(i)}}}{\vec{k}_{:,j}(\cdot, \vec{x})},
\)
and
\(
\LkL{\linop{L}_{\mwrtestfctl^{(i)}}}{\mat{k}},
\)
where $\linop{L}_{\mwrtestfctl^{(i)}} = \mwrtestfctl^{(i)} \circ \mwrwsdiffop$.
This might not always be possible in closed-form, since $\linop{L}_{\mwrtestfctl^{(i)}}$
often involves computing integrals.
However, in these cases one could fall back to an efficient numeric quadrature method,
since the integrals are often low-dimensional (typically at most four-dimensional).
A prominent example of this approach is the \emph{probabilistic meshless method} used in
\cref{sec:pde-solution-inference}.
\begin{example}[Symmetric Collocation]
  If the PDE is in strong formulation, then
  $\mwrtestfctl^{(i)} = \delta_{\vec{x}_i} \in \dualsp{\rhssp}$
  with $\vec{x}_i \in \dom$ is a valid test functional, which induces the information
  operator
  \begin{equation*}
    \evallinop{\mwrinfoop{\mwrtestfctl^{(i)}}}{\solvvprior}
    = \linopat{D}{\solvvprior}(\vec{x}_i) - f(\vec{x}_i),
  \end{equation*}
  i.e.~we recover the \emph{probabilistic meshless method} by
  \citet{Cockayne2017PNPDEInv}. They show that the conditional mean of this approach
  reproduces \emph{symmetric collocation} \citep{Fasshauer1997SColl,Fasshauer1999SColl},
  a well-known method to approximate strong solutions of PDEs.
\end{example}
The probabilistic meshless method can only be used to approximate
strong solutions of linear PDEs, since point evaluation functionals are not well-defined
on the image space $\dualsp{\rhssp}$ of $\weakdiffop$.
However, other choices of the $\mwrtestfctl^{(i)}$ lead to approximation schemes for
weak solutions.
\begin{example}[Weak Formulations]
  \label{ex:mwr-info-ops-inf-trial-weak}
  Consider a linear PDE in weak formulation.
  As mentioned in \cref{sec:mwr}, it is customary to use test functionals
  $\mwrtestfctl^{(i)}$, which are induced by test functions $\mwrtestfn^{(i)} \in \rhssp$,
  i.e.~%
  \begin{equation}
    \label{eqn:weak-mwr-info-op-inf-dim-trial}
    \evallinop{\mwrinfoop{\mwrtestfctl^{(i)}}}{\solvvprior}
    =
    \evallinop{\evallinop{\weakdiffop}{\solvvprior}}{\mwrtestfn^{(i)}}
    - \evallinop{\weakrhs}{\mwrtestfn^{(i)}}.
  \end{equation}
  For instance, if $\dom = \ointerval{l, r} \subset \R$, then a valid set of test
  functions for the weak formulation from \cref{sec:weak-formulation} is given by
  \begin{equation}
    \label{eqn:1d-linear-lagrange-element}
    \mwrtestfn^{(i)} =
    \begin{cases}
      \frac{x - x_{i - 1}}{x_i - x_{i - 1}} & \text{if } x_{i - 1} \le x \le x_i, \\
      \frac{x_{i + 1} - x}{x_{i + 1} - x_i} & \text{if } x_i \le x \le x_{i + 1}, \\
      0                                     & \text{otherwise}.
    \end{cases}
    \quad \in \sobolevtest{1}[\ointerval{l, r}],
  \end{equation}
  where $l = x_0 < \dotsb < x_{n + 1} = r$.
  The test functions are visualized in \cref{fig:1d-linear-lagrange-elements}.
  For the weak formulation in \cref{sec:weak-formulation}, the information operator from
  \cref{eqn:weak-mwr-info-op-inf-dim-trial} is equivalent to
  \(
  \evallinop{\mwrinfoop{\mwrtestfctl^{(i)}}}{\solprior}
  =
  \evallinop{\weakbilin}{\solprior, \mwrtestfn^{(i)}} - \inprod{f}{\mwrtestfn^{(i)}}[\L2].
  \)
\end{example}

\subsubsection{Finite-Dimensional Trial Function Spaces}
\label{sec:mwr-info-ops-fin-trial}
As opposed to the methods outlined in \cref{sec:mwr}, we did not need to choose a
finite-dimensional subspace of trial functions to arrive at tractable information
operators in \cref{sec:mwr-info-ops-inf-trial}.
Nevertheless, in practice, it might still be desirable to specify a finite-dimensional
trial function basis $\mwrtrialvv^{(1)}, \dotsc, \mwrtrialvv^{(m)}$, e.g.~because
\begin{itemize}
  \item we want to reproduce the output of a classical method in the posterior mean to
        use the GP solver as an uncertainty-aware drop-in replacement (see
        \cref{cor:gp-mean-mwr-recovery});
  \item the trial basis encompasses problem-specific knowledge, which is difficult to encode
        in the prior; or
  \item we want to solve the problem in a coarse-to-fine scheme, allowing for mesh
        refinement strategies, which are informed by the GP's uncertainty estimate.
\end{itemize}
Naively, one might achieve this goal by defining the prior over $\solvvprior$ as a parametric
Gaussian process with features $\mwrtrialvv^{(i)}$.
However, this means the posterior can not quantify the inherent approximation error,
since the GP has no support outside of the finite subspace of $\solsp$ spanned by the
trial functions.
Consequently, we need to take a different approach.
Starting from a general, potentially nonparametric prior over $\solvvprior$, we consider
a bounded (potentially oblique) projection
$\mwrtrialproj \colon \solsp \to \mwrtrialsp$
onto a subspace $\mwrtrialsp \subset \solsp$, i.e.~%
$\mwrtrialproj^2 = \mwrtrialproj$,
$\norm{\mwrtrialproj} < \infty$,
and $\range{\mwrtrialproj} = \mwrtrialsp$.
In general, this subspace need not be finite-dimensional.
We apply $\mwrtrialproj$ to our GP prior over $\solvvprior$, which, by
\cref{cor:gp-inference-linop-evals}, results in another GP
\begin{equation*}
  \hat{\solvvprior}
  \defeq \evallinop{\mwrtrialproj}{\solvvprior}
  \sim \gp{
    \evallinop{\mwrtrialproj}{\vec{m}}
  }{
    \LkL{\mwrtrialproj}{\mat{k}}
  },
\end{equation*}
with sample paths in $\mwrtrialsp$.
This discards prior information about $\kernel{\mwrtrialproj}$.
Hence, especially in case $\dim{\mwrtrialsp} < \infty$, applying the information
operators $\mwrinfoop{\mwrtestfctl^{(i)}}$ from \cref{sec:mwr-info-ops-inf-trial}
directly to $\hat{\solvvprior}$ would suffer from similar problems as choosing a
parametric prior.
However,
\begin{equation*}
  \mwrinfoop{\mwrtestfctl^{(i)}}[\mwrtrialproj]
  \defeq \mwrinfoop{\mwrtestfctl^{(i)}} \circ \mwrtrialproj
  = \evallinop{(\mwrtestfctl^{(i)} \circ \mwrwsdiffop \circ \mwrtrialproj)}{\cdot} - \evallinop{\mwrtestfctl^{(i)}}{\mwrwsrhs}
\end{equation*}
is a valid information operator for $\solvvprior$, which leads to a probabilistic
generalization of the method of weighted residuals.
This is why we refer to $\mwrinfoop{\mwrtestfctl^{(i)}}[\mwrtrialproj]$ as an \emph{MWR
  information operator}.

The similarity to the method of weighted residuals is particularly prominent if we
choose a finite-dimensional subspace
$\mwrtrialsp = \linspan{\mwrtrialvv^{(1)}, \dotsc, \mwrtrialvv^{(m)}}$
as in \cref{sec:mwr}.
In this case, there is a bounded linear operator $\mwrcoordproj \colon \solsp \to \R^m$
such that
\begin{equation*}
  \evallinop{\mwrtrialproj}{\solvvprior}
  = \sum_{i = 1}^m \mwrcoordspriorelem_i \mwrtrialvv^{(i)}
  \rdefeq \evallinop{\isomorphism{\R^m}{\mwrtrialsp}}{\mwrcoordsprior},
\end{equation*}
where the $\mwrcoordsprior \defeq \evallinop{\mwrcoordproj}{\solvvprior} \in \R^m$
are the coordinates of $\evallinop{\mwrtrialproj}{\solvvprior}$ in $\mwrtrialsp$ and
$\isomorphism{\R^m}{\mwrtrialsp} \colon \R^m \to \mwrtrialsp$ is the canonical
isomorphism between $\R^m$ and $\mwrtrialsp$.
Hence, we get the factorization
\begin{equation}
  \label{eqn:factorization-trial-projection}
  \mwrtrialproj = \isomorphism{\R^m}{\mwrtrialsp} \mwrcoordproj,
\end{equation}
which implies that $\hat{\solvvprior}$ is a parametric Gaussian process.
Moreover, $\evallinop{\mwrtestfctl^{(i)}}{\mwrwsrhs} = \mwrrhselem_i$ and
\begin{equation*}
  \evallinop{(\mwrtestfctl^{(i)} \circ \mwrwsdiffop \circ \isomorphism{\R^m}{\mwrtrialsp})}{\mwrcoordsprior}
  = \sum_{i = 1}^m \mwrcoordspriorelem_i \evallinop{(\mwrtestfctl^{(i)} \circ \mwrwsdiffop)}{\mwrtrialvv^{(i)}}
  = (\mwrmat \mwrcoordsprior)_i,
\end{equation*}
where $\mwrmat$ and $\mwrrhs$ are defined as in \cref{sec:mwr}.
Consequently, the MWR information operator is given by
\(
\evallinop{\mwrinfoop{\mwrtestfctl^{(i)}}[\mwrtrialproj]}{\solvvprior}
= \evallinop{(\linfctls{I}_{\R^m} \circ \mwrtrialproj)}{\solvvprior}_i,
\)
where
\(
\evallinop{\linfctls{I}_{\R^m}}{\mwrcoordsprior} \defeq \mwrmat \mwrcoordsprior - \mwrrhs.
\)
This illustrates that we are dealing with the hierarchical model
\begin{align*}
  \solvvprior                                 & \sim \gp{\vec{m}}{\mat{k}}                           \\
  \condrv{\mwrcoordsprior \given \solvvprior} & \sim \delta_{\evallinop{\mwrcoordproj}{\solvvprior}}
\end{align*}
with observations
\(
\evallinop{\linfctls{I}_{\R^m}}{\mwrcoordsprior} = \vec{0},
\)
where
\(
\mwrcoordsprior
\sim \gaussian{\evallinop{\mwrcoordproj}{\vec{m}}}{\LkL{\mwrcoordproj}{\mat{k}}}.
\)
Inference in this model can be broken down into two steps.
First, we update our belief about the solution's coordinates in $\mwrtrialsp$ by
computing the conditional random variable
\(
\condrv{%
  \mwrcoordsprior
  \given
  \evallinop{\linfctls{I}_{\R^m}}{\mwrcoordsprior} = \vec{0}
},
\)
which is also Gaussian.
If $\mwrmat$ is invertible and $\mwrcoordsprior$ has full support on $\R^m$, then
the law of
\(
\condrv{%
  \mwrcoordsprior
  \given
  \evallinop{\linfctls{I}_{\R^m}}{\mwrcoordsprior} = \vec{0}
},
\)
is a Dirac measure whose mean is given by the coordinates of the MWR approximation
$\mwrsolcoords = \mwrmat\inv \mwrrhs$ from \cref{eqn:mwr-solution}.
Next, we can reuse precomputed quantities from the conditional moments of
\(
\condrv{%
  \mwrcoordsprior
  \given
  \evallinop{\linfctls{I}_{\R^m}}{\mwrcoordsprior} = \vec{0}
},
\)
such as the representer weights
\(
\vec{w} =
(\mwrmat \LkL{\mwrcoordproj}{\mat{k}} \mwrmat\T)\pinv
(\mwrrhs - \mwrmat \evallinop{\mwrcoordproj}{\vec{m}})
\)
to efficiently compute the conditional random process
\begin{equation*}
  (\condrv{\solvvprior \given \evallinop{(\linfctls{I}_{\R^m} \circ \mwrcoordproj)}{\mwrcoordsprior} = \vec{0}})
  = (\condrv{\solvvprior \given \{ \evallinop{\mwrinfoop{\mwrtestfctl^{(i)}}[\mwrtrialproj]}{\solvvprior} = 0 \}_{i = 1}^n}),
\end{equation*}
i.e.~the main quantity of interest.
Assuming once more that $\mwrmat$ is invertible and $\mwrcoordsprior$ has full support
on $\R^m$, the remaining uncertainty of the conditional process lies in the kernel of
$\proj{\mwrtrialsp}$, since the law of
\(
\condrv{%
  \mwrcoordsprior
  \given
  \evallinop{\linfctls{I}_{\R^m}}{\mwrcoordsprior} = \vec{0}
},
\)
is a Dirac measure and
\begin{equation*}
  (\condrv{\evallinop{\mwrtrialproj}{\solvvprior} \given \{ \evallinop{\mwrinfoop{\mwrtestfctl^{(i)}}[\mwrtrialproj]}{\solvvprior} = 0 \}_{i = 1}^n})
  = (\condrv{\evallinop{\isomorphism{\R^m}{\mwrtrialsp}}{\mwrcoordsprior} \given \evallinop{\linfctls{I}_{\R^m}}{\mwrcoordsprior} = \vec{0}}).
\end{equation*}
Thus, all remaining uncertainty must be due to
\(
\condrv{%
\evallinop{(\id[\solsp] - \mwrtrialproj)}{\solvvprior}
\given
\{ \evallinop{\mwrinfoop{\mwrtestfctl^{(i)}}[\mwrtrialproj]}{\solvvprior} = 0 \}_{i = 1}^n
}.
\)
Note the striking similarity of this property to the notion of \emph{Galerkin
  orthogonality} \citep[Equation 2.63]{Logg2012FEniCS}.

A canonical choice for $\mwrtrialproj$ would arguably be an orthogonal
projection w.r.t. the RKHS inner product of the sample space of $\solvvprior$ (see e.g.~%
\citealt{Kanagawa2018GPKernMeth}).
However, this inner product is generally difficult to compute.
Fortunately, we can use the $\L2$ inner products or Sobolev inner products on the
samples to induce a (usually non-orthogonal) projection $\mwrtrialproj$.
\begin{example}
  \label{ex:l2-projection}
  If the elements of $\solsp$ are square-integrable, then the linear operator
  \begin{equation*}
    \evallinop{\mwrcoordproj}{\solvvprior}
    \defeq \mat{P}\inv \left( \rintegral[\dom]{\vec{x}}{\inprod{\mwrtrialvv^{(i)}(\vec{x})}{\solvvprior(\vec{x})}[\R^d]} \right)_{i = 1}^m,
  \end{equation*}
  where
  \begin{equation*}
    \matelem{P}_{ij} \defeq \rintegral[D]{\vec{x}}{\inprod{\mwrtrialvv^{(i)}(\vec{x})}{\mwrtrialvv^{(j)}(\vec{x})}[\R^d]},
  \end{equation*}
  induces a projection $\mwrtrialproj = \isomorphism{\R^m}{\mwrtrialsp} \mwrcoordproj$
  onto $\mwrtrialsp \subset \solsp$, even if $\solsp$ is not a Hilbert space with inner
  product $\inprod{\cdot}{\cdot}[\L2[\dom]]$.
\end{example}
At first glance, information operators restricting $\mwrtrialsp$ to be
finite-dimensional might seem fundamentally inferior to the information operators from
\cref{sec:mwr-info-ops-inf-trial}.
However, the conditional mean of a Gaussian process prior conditioned on
$\{ \evallinop{\mwrinfoop{\mwrtestfctl^{(i)}}}{\solvvprior} = 0 \}_{i = 1}^n$ is updated
by a linear combination of $n$ functions, while the covariance function receives an at
most rank $n$ downdate.
This means that, implicitly, conditioning a Gaussian process on an information operator
with $\mwrtrialproj = \id[\solsp]$ also constructs a finite-dimensional trial function
space, which depends on the test function basis, the bilinear form $\weakbilin$ and the
prior covariance function $\mat{k}$.

MWR information operators with finite-dimensional trial function bases can be used to
realize a GP-based analogue of the finite element method.
\begin{example}[A 1D Finite Element Method]
  \label{ex:fem}
  Finite element methods are (generalized) Galerkin methods, where
  the functions in the test and trial bases have compact support, i.e.~they are nonzero
  only in a highly localized region of the domain.
  The archetype of a finite element method for the weak formulation from
  \cref{sec:weak-formulation} uses linear \emph{Lagrange elements}
  \citep[Section 3.3.1]{Logg2012FEniCS} as test and trial functions, i.e.~%
  $\mwrtrial^{(i)}(x) = \mwrtest^{(i)}(x)$ and $m = n$.
  Linear Lagrange elements are piecewise linear on a triangulation of the domain.
  For instance, on a one-dimensional domain $\dom = \ointerval{-1, 1}$, the linear
  Lagrange elements are given by \cref{eqn:1d-linear-lagrange-element} from
  \cref{ex:mwr-info-ops-inf-trial-weak}.
  Multiplying a coordinate vector $\mwrcoords \in \R^m$ with these basis functions leads
  to a piecewise linear interpolation between the points
  \[
    (x_0, 0), (x_1, \mwrcoordselem_1), \dotsc, (x_n, \mwrcoordselem_n), (x_{n + 1}, 0),
  \]
  since, for $x \in [x_i, x_{i + 1}]$,
  \begin{equation*}
    \sum_{i = 1}^m \mwrcoordselem_i \mwrtrial^{(i)}(x)
    = \mwrcoordselem_i \frac{x_{i + 1} - x}{x_{i + 1} - x_i}
    + \mwrcoordselem_{i + 1} \frac{x - x_i}{x_{i + 1} - x_i}
    = \left( 1 - \frac{x - x_i}{x_{i + 1} - x_i} \right) \mwrcoordselem_i
    + \left( \frac{x - x_i}{x_{i + 1} - x_i} \right) \mwrcoordselem_{i + 1}.
  \end{equation*}
  The basis functions and an element in their span are visualized in
  \cref{fig:1-fem-basis}.
  The Lagrange elements at the boundary of the domain can also be easily modified such
  that arbitrary piecewise linear boundary conditions are fulfilled by construction.
  The effect of MWR information operators based on this set of test and trial functions
  is visualized in \cref{fig:mwr-info-ops-fin-trial-fem-matern}.
\end{example}
\begin{figure}[t]
  \subcaptionbox{%
    Test/trial functions $\mwrtrial^{(i)} = \mwrtest^{(i)}$.
    \label{fig:1d-linear-lagrange-elements}
  }{%
    \includegraphics[width=0.49\textwidth]{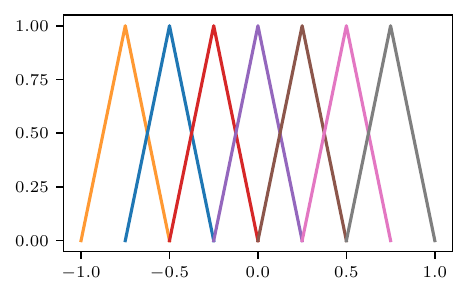}%
  }
  \hfill
  \subcaptionbox{%
    The trial functions $\mwrtrial^{(i)}$ span the space of piecewise linear functions on the
    given grid.
  }{%
    \includegraphics[width=0.49\textwidth]{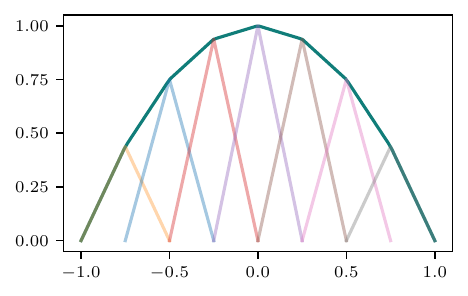}%
  }
  \caption{
    Linear Lagrange test and trial functions as used by the finite element method.
  }
  \label{fig:1-fem-basis}
\end{figure}
\begin{figure}[t]
  \subcaptionbox{%
    Posterior process corresponding to a Mat\'ern-$\nicefrac{3}{2}$ prior.
    The sample paths of the process embed continuously into the Sobolev space
    $\sobolev{1}[\dom]$ (see \cref{sec:prior-selection}).
    \label{fig:mwr-info-ops-fin-trial-fem-matern}
  }{%
    \includegraphics[width=0.49\textwidth]{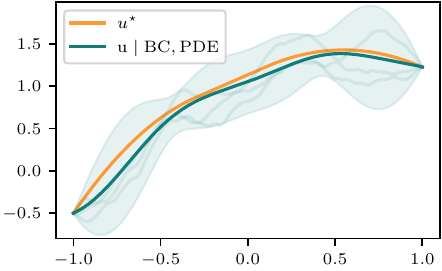}%
  }
  \hfill
  \subcaptionbox{%
    Posterior process corresponding to an MWR Recovery Prior constructed from a
    Mat\'ern-$\nicefrac{3}{2}$ prior via \cref{prop:mwr-recovery-prior}.
    The posterior mean corresponds to the point estimate produced by the classical MWR.
    \label{fig:mwr-info-ops-fin-trial-fem-recovery}
  }{%
    \includegraphics[width=0.49\textwidth]{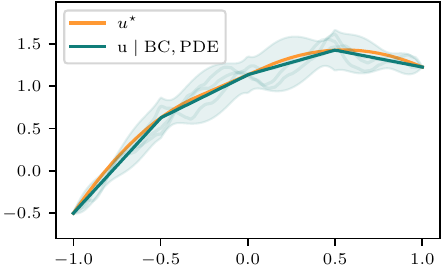}%
  }
  \caption{
    Conditioning two different Gaussian process priors on the MWR information operators
    $\{ \linop{I}_{\mwrtest^{(i)},\mwrtrialproj} \}_{i = 1}^n$
    corresponding to the weak formulation of the Poisson equation, i.e.~%
    \cref{eqn:stat-heat-isotropic-non-homogeneous-weak}, and $m = 3$ linear Lagrange
    elements as test functions $\mwrtest^{(i)}$ and trial functions $\mwrtrial^{(i)}$ (see
    \cref{ex:fem}).
    The trial functions $\mwrtrial^{(1)}$ and $\mwrtrial^{(m)}$ were modified to fulfill
    the non-zero boundary conditions exactly.
  }
  \label{fig:mwr-info-ops-fin-trial-fem}
\end{figure}

\subsubsection{MWR Information Operators}
Even though the class of information operators introduced above is constructed for
linear PDEs, it can naturally be applied to the weak form of an arbitrary operator
equation.
In particular, we can use MWR information operators for the boundary conditions in an
(I)BVP.
Moreover, it is straightforward to extend $\mwrinfoop{\mwrtestfctl}[\mwrtrialproj]$ to a
joint GP prior over $(\solvvprior, \mwrwsrhsprior)$ if the right-hand side $\mwrwsrhs$
of the operator equation is unknown as in \cref{sec:linear-pde}.
In this case, $\mwrinfoop{\mwrtestfctl}[\mwrtrialproj]$ is jointly linear in
$(\solvvprior, \mwrwsrhsprior)$.
Summarizing \cref{sec:mwr-info-ops-inf-trial,sec:mwr-info-ops-fin-trial} and
incorporating the extensions discussed here, we define an MWR information operator as follows:
\begin{definition}[MWR Information Operator]
  \label{def:mwr-info-op}
  Let $\evallinop{\mwrwsdiffop}{\solvvprior} = \mwrwsrhs$ be an operator equation in
  strong or weak formulation.
  An \emph{MWR information operator} for said operator equation is a continuous affine
  functional
  \begin{equation*}
    \mwrinfoop{\mwrtestfctl}[\mwrtrialproj]
    \defeq
    \evallinop{(\mwrtestfctl \circ \mwrwsdiffop \circ \mwrtrialproj)}{\cdot}
    - \evallinop{\mwrtestfctl}{\mwrwsrhs}
  \end{equation*}
  parameterized by a \emph{test functional} $\mwrtestfctl \in \mwrwstestfctlsp$ and a
  bounded (potentially oblique) projection $\mwrtrialproj$ onto a subspace
  $\mwrtrialsp \subset \solsp$.
  We also write $\mwrinfoop{\mwrtestfctl} \defeq \mwrinfoop{\mwrtestfctl}[\id[\solsp]]$.
  The input of $\mwrinfoop{\mwrtestfctl}[\mwrtrialproj]$ can be extended to the
  right-hand side $\mwrwsrhsprior$ of the operator equation, i.e.~%
  \begin{equation*}
    \evallinop{\mwrinfoop{\mwrtestfctl}[\mwrtrialproj]}{(\solvvprior, \mwrwsrhsprior)}
    \defeq
    \evallinop{(\mwrtestfctl \circ \mwrwsdiffop \circ \mwrtrialproj)}{\solvvprior}
    - \evallinop{\mwrtestfctl}{\mwrwsrhsprior},
  \end{equation*}
  which is jointly linear in $(\solvvprior, \mwrwsrhsprior)$.
\end{definition}

\begin{table}[t]
  \centering
  \small
  \begin{tabular}{c|p{0.17\textwidth}|p{0.28\textwidth}|p{0.38\textwidth}}
  \toprule
   & Method
   & Trial Functions $\mwrtrialvv^{(i)}$
   & Test Functionals $\mwrtestfctl^{(i)}$ / Functions $\mwrtest^{(i)}$
  \\
  \midrule
  \multirow{9}{*}{\rotatebox[origin=c]{90}{Strong Solutions}}
   & Collocation
   & arbitrary
   & $\mwrtestfctl^{(i)} = \delta_{\vec{x}_i}$ for $\vec{x}_i \in \dom$
  \newline $\Rightarrow \evallinop{(\mwrtestfctl^{(i)} \circ \diffop)}{\solvvprior} = \linopat{D}{\solvvprior}(\vec{x}_i)$
  \\
  \cmidrule{2-4}
   & Subdomain \newline (Finite Volume)
   & arbitrary
   & $\mwrtest^{(i)} = \chi_{\dom_i}$ for $\dom_i \subset \dom$
  \newline $\Rightarrow \evallinop{(\mwrtestfctl^{(i)} \circ \diffop)}{\solvvprior} = \rintegral[\dom_i]{\vec{x}}{\linopat{D}{\solvvprior}(\vec{x})}$
  \\
  \cmidrule{2-4}
   & Pseudospectral
   & orthogonal and globally supported (e.g.~Fourier basis or Chebychev polynomials)
   & $\mwrtestfctl^{(i)} = \delta_{\vec{x}_i}$ for $\vec{x}_i \in \dom$
  \newline $\Rightarrow \evallinop{(\mwrtestfctl^{(i)} \circ \diffop)}{\solvvprior} = \linopat{D}{\solvvprior}(\vec{x}_i)$
  \\
  \midrule
  \multirow{10}{*}{\rotatebox[origin=c]{90}{Weak \& Strong Solutions}}
   & Generalized Galerkin
   & arbitrary
   & arbitrary, but in general $\mwrtest^{(i)} \neq \mwrtrial^{(i)}$
  \\
  \cmidrule{2-4}
   & Finite Element
   & locally supported (e.g.~piecewise polynomial)
   & same class as trial functions, but in general $\mwrtest^{(i)} \neq \mwrtrial^{(i)}$
  \\
  \cmidrule{2-4}
   & Spectral (Galerkin)
   & orthogonal and globally supported (e.g.~Fourier basis or Chebychev polynomials)
   & same class as trial functions, but in general $\mwrtest^{(i)} \neq \mwrtrial^{(i)}$
  \\
  \cmidrule{2-4}
   & (Ritz-)Galerkin
   & arbitrary
   & $\mwrtest^{(i)} = \mwrtrial^{(i)}$
  \\
  \bottomrule
\end{tabular}
  \caption{%
    Trial and test function(al)s defining commonly used methods of weighted
    residuals. If used as part of an MWR information operator, the GP posterior mean recovers the corresponding classic method (see \cref{cor:gp-mean-mwr-recovery}).
  }
  \label{tbl:mwrs-trial-test}
\end{table}

\subsubsection{Recovery of Classical Methods}
\label{sec:gp-mean-mwr-recovery}
In this section we will show that, under certain assumptions, the posterior mean of a
GP prior conditioned on a set of MWR information operators is identical to the
approximation generated by the corresponding traditional method of weighted residuals,
examples of which are given in \Cref{tbl:mwrs-trial-test}.
More precisely, we will show that there is a flexible family of GP priors
$\solvvprior \sim \gp{\vec{m}}{\mat{k}}$ whose posterior means after conditioning on
\(
\{ \mwrinfoop{\mwrtestfctl^{(i)}}[\mwrtrialproj] \}_{i = 1}^m
\)
are identical to the corresponding classical MWR approximation $\mwrsolcoords$ to the
solution of the same weak form linear PDE, where we use the same trial functions
$\mwrtrialvv^{(1)}, \dotsc, \mwrtrialvv^{(m)}$ and test functionals
$\mwrtestfctl^{(1)}, \dotsc, \mwrtestfctl^{(n)}$ in both cases, i.e.~%
$\mwrtrialsp = \linspan{\mwrtrialvv^{(1)}, \dotsc, \mwrtrialvv^{(m)}}$.
As in \cref{sec:mwr}, we assume that the trial functions are already constructed in such
a way that the boundary conditions are fulfilled.
However, it is possible to extend the results below to the general case by adding MWR
information operators corresponding to the boundary conditions and using
\begin{equation*}
  \mwrsolcoords =
  \begin{pmatrix}
    \mwrmat_\text{PDE} \\
    \mwrmat_\text{BC}
  \end{pmatrix}\inv
  \begin{pmatrix}
    \mwrrhs_\text{PDE} \\
    \mwrrhs_\text{BC}
  \end{pmatrix}
\end{equation*}
as coordinates for the reference solution generated by the traditional MWR.

\begin{proposition}
  \label{prop:gp-mwr-info-op-fin-trial}
  If $\mwrmat \in \R^{n \times m}$ and
  \(
  \mat{\Sigma}_\mwrcoordsprior
  \defeq \LkL{\mwrcoordproj}{\mat{k}}
  \in \R^{m \times m}
  \)
  are invertible, then
  \begin{equation*}
    \condrv{\mwrcoordsprior \given \mwrmat \mwrcoordsprior - \mwrrhs = \vec{0}}
    \sim
    \delta_{\mwrsolcoords}
  \end{equation*}
  and the conditional mean $\vec{m}^{\condrv{\solvvprior \given \mwrmat, \mwrrhs}}$ of
  \(
  \condrv{\solvvprior \given \mwrmat \evallinop{\mwrcoordproj}{\solvvprior} - \mwrrhs = \vec{0}}
  \)
  admits a unique additive decomposition
  \begin{equation}
    \vec{m}^{\condrv{\solvvprior \given \mwrmat, \mwrrhs}}
    = \mwrsol + \solvv_{\kernel{\mwrtrialproj}}
  \end{equation}
  with $\mwrsol \in \mwrtrialsp$ and
  $\solvv_{\kernel{\mwrtrialproj}} \in \kernel{\mwrtrialproj}$.
\end{proposition}
\begin{corollary}[MWR Generalization]
  \label{cor:gp-mean-mwr-recovery}
  If, additionally, $\vec{m} \in \mwrtrialsp$ and
  $\LkL{\proj{\kernel{\mwrtrialproj}}}{\mat{k}}[\mwrcoordproj] = \vec{0}$,
  then the conditional mean function $\vec{m}^{\condrv{\solvvprior \given \mwrmat, \mwrrhs}}$
  is equal to the MWR approximation, i.e.~%
  \begin{equation*}
    \vec{m}^{\condrv{\solvvprior \given \mwrmat, \mwrrhs}} = \mwrsol.
  \end{equation*}
\end{corollary}

It turns out that it is possible to transform any admissible GP prior over the (weak)
solution of the PDE into a prior that fulfills the assumptions of
\cref{cor:gp-mean-mwr-recovery}.
\begin{proposition}[MWR Recovery Prior]
  \label{prop:mwr-recovery-prior}
  Let $\tilde{\solvvprior} \sim \gp{\tilde{\vec{m}}}{\tilde{\mat{k}}}$ with mean and
  sample paths in $\solsp$.
  Then \( \solvvprior \sim \gp{\vec{m}}{\mat{k}} \) with
  $\vec{m} \defeq \evallinop{\mwrtrialproj}{\tilde{\vec{m}}}$
  and
  \begin{align*}
    \mat{k}
     & \defeq \LkL{\mwrtrialproj}{\tilde{\mat{k}}} + \LkL{\proj{\kernel{\mwrtrialproj}}}{\tilde{\mat{k}}} \\
     & = \LkL{\mwrtrialproj}{\tilde{\mat{k}}} + \LkL{(\id[\solsp] - \mwrtrialproj)}{\tilde{\mat{k}}}      \\
     & = \tilde{\mat{k}}
    - \Lk{\mwrtrialproj}{\tilde{\mat{k}}}
    - \kL{\tilde{\mat{k}}}{\mwrtrialproj}
    + 2 \LkL{\mwrtrialproj}{\tilde{\mat{k}}}
  \end{align*}
  has sample paths in $\solsp$, $\vec{m} \in \mwrtrialsp$,
  and $\LkL{\proj{\kernel{\mwrtrialproj}}}{\mat{k}}[\mwrcoordproj] = \vec{0}$.
\end{proposition}
\Cref{fig:mwr-info-ops-fin-trial-fem-recovery} visualizes how a prior of this form
reproduces a 1D finite element method in the posterior mean and
\cref{fig:mwr-info-ops-fin-trial-fem} as a whole contrasts the difference between
$\tilde{\solvvprior}$ and $\solvvprior$.
Intuitively speaking, the construction for the covariance from
\cref{prop:mwr-recovery-prior} enforces statistical independence between the subspaces
$\mwrtrialsp$ and $\kernel{\mwrtrialproj}$ of the GP's path space.
This way, an observation of the GP prior in the subspace $\mwrtrialsp$ gains no information
about $\kernel{\mwrtrialproj}$, which means that the posterior process will not be
updated along $\kernel{\mwrtrialproj}$.
Since $\vec{m} \in \mwrtrialsp$, i.e.~$\evallinop{\proj{\kernel{\mwrtrialproj}}}{\vec{m}}$,
it follows that the posterior mean will also lie in $\mwrtrialsp$.
Even though this choice of prior is somewhat restrictive, there are good reasons to use
it in practice, arguably the most important of which is that the uncertainty
quantification provided by the GP can be added on top of traditional MWR solvers in
existing pipelines in a plug-and-play fashion.
This is because given the MWR recovery prior, the mean estimate of the probabilistic numerical method agrees with the point
estimate produced by the classical solver.

\subsection{Algorithm}
\label{sec:algorithm}
\Cref{alg:linpde-gp} summarizes our framework from an algorithmic standpoint.
It outlines how a GP prior can be conditioned on heterogeneous sources of information
such as mechanistic knowledge given in the form of a linear boundary value problem,
and noisy measurement data by leveraging the notion of a linear information operator.
All GP posteriors in this article were computed by this algorithm with different choices
of prior, PDE, boundary conditions and policy.
\begin{algorithm}[ht]
  \caption{Solving PDEs via Gaussian Process Inference\label{alg:linpde-gp}}
  \textbf{Input:}
Joint GP prior \( (\solvvprior, \mwrwsrhsprior, \mwrwsbfnprior, \rvec{\epsilon}) \sim \gp{\vec{m}}{\mat{k}} \),
linear PDE \( (\mwrwsdiffop, \mwrwsrhsprior) \),
boundary conditions \( (\mwrwsbop, \mwrwsbfnprior) \),
(noisy) measurements \( (\mat{X}_\text{MEAS}, \vec{y}_\text{MEAS}) \), $\dotsc$ \\
\textbf{Output:} GP posterior \( \gp{\vec{m}^{(i)}}{\mat{k}^{(i)}} \)
\begin{algorithmic}[1]
  \Procedure{\textsc{LinPDE-GP}}{$\vec{m}, \mat{k}, \mwrinfoop{\cdot}[\cdot]^\text{PDE}, \mwrinfoop{\cdot}[\cdot]^\text{BC}, \mat{X}_\text{MEAS}, \vec{y}_\text{MEAS}$}
  \State \( i \gets 0 \)
  \State \( (\vec{m}^{(0)}, \mat{k}^{(0)}) \gets (\vec{m}, \mat{k}) \)
  \State \( \vec{w}^{(0)} \gets ()\)
  \State \( \mat{G}^{(0)} \gets ()\)
  \While{\textbf{not} \textsc{StoppingCriterion}()}
  \State \( i \gets i + 1 \)
  \State \(
  (
  \mwrtestfctl_\text{PDE}^{(i)},
  \mwrtestfctl_\text{BC}^{(i)},
  \mwrtrialproj^{(i)},
  \dotsc,
  \vec{v}_\text{MEAS}^{(i)}
  )
  \gets \textsc{Policy}(\vec{m}^{(i)}, \mat{k}^{(i)})
  \)
  \Comment{Action}
  \State \(
  \linfctls{I}^{(i)}
  \gets
  (\solvv, \mwrwsrhs, \mwrwsbfn, \vec{\epsilon})
  \mapsto
  \begin{pmatrix}
    \evallinop{\linfctls{I}_{\mwrtestfctl_\text{PDE}^{(i)}, \mwrtrialproj^{(i)}}^\text{PDE}}{(\solvv, \mwrwsrhs)} \\
    \evallinop{\linfctls{I}_{\mwrtestfctl_\text{BC}^{(i)}, \mwrtrialproj^{(i)}}^\text{BC}}{(\solvv, \mwrwsbfn)}   \\
    \vdots                                                                                                        \\
    \inprod{\vec{v}_\text{MEAS}^{(i)}}{\solvv(
      \mat{X}_\text{MEAS}) + \vec{\epsilon}}
  \end{pmatrix}
  \)
  \Comment{Information operator}
  \State \( \vec{y}^{(i)} \gets \pvec{0 & 0 & \hdots & \inprod{\vec{v}_\text{MEAS}^{(i)}}{\vec{y}_\text{MEAS}}}\T \)
  \Comment{Observations}
  \State \(
  \mat{G}^{(i)}
  \gets
  \begin{pmatrix}
    \mat{G}^{(i - 1)}                                           & \LkL{\linfctls{I}^{(1:i-1)}}{\mat{k}}[(\linfctls{I}^{(i)})] \\
    \LkL{\linfctls{I}^{(i)}}{\mat{k}}[(\linfctls{I}^{(1:i-1)})] & \LkL{\linfctls{I}^{(i)}}{\mat{k}}[(\linfctls{I}^{(i)})]
  \end{pmatrix}
  \)
  \Comment{Update Gram matrix}
  \State \( \vec{w}^{(i)} \gets (\mat{G}^{(i)})\pinv (\vec{y}^{(1:i)} - \evallinop{\linfctls{I}^{(1:i)}}{\vec{m}}) \)
  \Comment{Update representer weights}
  \State \( \vecelem{m}^{(i)}_j \gets \vec{x} \mapsto \vecelem{m}_j(\vec{x}) + \evallinop{\linfctls{I}^{(1:i)}}{\vec{k}_{:,j}(\cdot, \vec{x})}\T \vec{w}^{(i)} \)
  \Comment{Belief Update}
  \State \( \matelem{k}^{(i)}_{j_1,j_2} \gets (\vec{x}_1, \vec{x}_2) \mapsto \matelem{k}_{j_1,j_2}(\vec{x}_1, \vec{x}_2) - \evallinop{\linfctls{I}^{(1:i)}}{\vec{k}_{:,j_1}(\cdot, \vec{x}_1)}\T (\mat{G}^{(i)})\pinv \evallinop{\linfctls{I}^{(1:i)}}{\vec{k}_{:,j_2}(\cdot, \vec{x}_2)} \)
  \EndWhile
  \State \Return \( \gp{\vec{m}^{(i)}}{\mat{k}^{(i)}} \)
  \EndProcedure
\end{algorithmic}
\end{algorithm}
Modeling uncertainty over the right-hand side $\mwrwsrhs$ of the PDE and the boundary
function(al) $\mwrwsbfn$ is achieved by specifying a joint prior
$(\solvvprior, \mwrwsrhsprior, \mwrwsbfnprior, \rvec{\epsilon})$.
Therefore, \Cref{alg:linpde-gp} also returns a multi-output Gaussian process posterior
over the same objects.
This means that our method can be used to solve PDE-constrained Bayesian inverse
problems for the right-hand side $\mwrwsrhs$ and the boundary function $\mwrwsbfn$,
while computing a consistent distributional estimate for the corresponding solution
$\solvv$ of the forward problem.
This is a generalization of a linear \emph{latent force model} \citep{Alvarez2009LFM}.
If $\mwrwsrhs$ and $\mwrwsbfn$ are not uncertain, the corresponding covariance functions
in the joint prior can simply be set to 0, which (in the absence of measurements)
reduces the joint prior to a simple prior over the solution $\solvvprior$.
To condition the GP on the PDE and the boundary conditions, we make use of MWR
information operators (see \cref{def:mwr-info-op}), where the test functions and
projections are chosen by an arbitrary policy in each iteration of the method.
An example of such a policy which reproduces \cref{fig:mwr-recovery} chooses
$\mwrtrialproj$ as the $\L2$ projection onto the basis from \cref{ex:fem} in every
iteration, the test functions $\mwrtestfctl_\text{BC} \in \{ \delta_{-1}, \delta_1 \}$,
and $\mwrtestfctl_\text{PDE} = 0$ in the first two iterations; and
$\mwrtestfctl_\text{PDE}$ is induced by $\mwrtestfn^{(i - 2)} = \mwrtrial^{(i - 2)}$
(and $\mwrtestfctl_\text{BC} = 0$) from iteration 3 onward.
The ellipses in the information operator $\linfctls{I}^{(i)}$ and the observations
$\vec{y}^{(i)}$ indicate that adding additional information operators is possible in the
same fashion.
For instance, adding additional PDE information operators enables the solution of
systems of linear PDEs.

\paragraph{Performance Considerations}
Instead of naively conditioning the previous conditional process on the
new observation in each iteration, \cref{alg:linpde-gp} always conditions the prior on
the accumulated observations.
This is because the naive expressions for the conditional moments become more
and more complex over time.
While, in principle, it is possible to use automatic differentiation (AD) to compute
$\evallinop{\linfctls{I}^{(i)}}{\vec{m}^{(i)}}$,
$\evallinop{\linfctls{I}^{(i)}}{\vec{k}^{(i - 1)}_{:, j}(\cdot, \vec{x})}$, and
$\LkL{\linfctls{I}^{(i)}}{\mat{k}^{(i - 1)}}[(\linfctls{I}^{(i)})]$ in each iteration and then
evaluate \cref{eqn:gp-linop-evals-posterior-mean,eqn:gp-linop-evals-posterior-cov}
naively, we found that this is detrimental to the performance of the algorithm.
In \cref{alg:linpde-gp}, we only need to compute
$\evallinop{\linfctls{I}^{(i)}}{\vec{m}}$,
$\evallinop{\linfctls{I}^{(i)}}{\vec{k}_{:,j}(\cdot, \vec{x})}$, and
$\LkL{\linfctls{I}^{(i)}}{\mat{k}}[(\linfctls{I}^{(i)})]$
on the prior moments, which are much less complex and cheaper to evaluate.
For maximum efficiency, for many information operator / kernel combinations one can
compute optimized closed-form expressions for these terms, alleviating the need for
automatic differentiation or quadrature.
We can avoid unnecessary recomputation of the representer weights at every iteration of
the method by means of block-matrix inversion.
For instance, if a Cholesky decomposition is used to invert the Gramian $\mat{G}^{(i)}$,
we can use a variant of the block Cholesky decomposition \citep{Golub2013Matrix} to
update the Cholesky factor of $\mat{G}^{(i - 1)}$.

\paragraph{Code}
A Python implementation of \cref{alg:linpde-gp} based on ProbNum
\citep{Wenger2021ProbNum} and JAX \citep{Johnson2018JAX} is available at:
\begin{center}
  \url{https://github.com/marvinpfoertner/linpde-gp}
\end{center}

\subsection{Related Work}
\label{sec:related-work}
%
%
The area of physics-informed machine learning \citep{Karniadakis2021PhysInfML}
aims at augmenting machine learning models with mechanistic knowledge about physical
phenomena, mostly in the form of ordinary and partial differential equations.
Recently, there has been growing interest in deep learning--based approaches \citep{Raissi2019PINNs,Li2020NeurOp,Li2021FNO}.
However, this model choice makes it inherently difficult to quantify the uncertainty about the solution induced by noise-corrupted input data and inevitable approximation error.
%
%
Instead, we approach the problem through the lens of \emph{probabilistic numerics}
\citep{Hennig2015PN,Cockayne2019BayesPNMeth,Oates2019PNRetro,Owhadi2019StatNumApprox,%
  Hennig2022PNBook}, which frames numerical problems as statistical estimation tasks.
%
%
Probabilistic numerical methods for the solution of PDEs
are predominantly based on Gaussian process priors.
Our work builds upon and extends these works.
Many existing methods aim to find a strong solution to a linear PDE using a collocation
scheme (e.g.~\citealt{Graepel2003LinOpEqGP,Cockayne2017PNPDEInv,Raissi2017LinDEGP}).
Unfortunately, many practically relevant (linear) PDEs only admit weak solutions.
Our framework extends existing collocation approaches to weak formulations.
Probabilistic numerical methods approximating weak formulations are primarily based on discretization.
For example, \citet{Cockayne2019BayesCG,Wenger2020ProbLinSolve} apply a probabilistic linear solver to the linear system arising from discretization.
\citet{Girolami2021statFEM} propose a statistical version of the finite element method (statFEM), which uses a specific parametric GP prior.
However, these approaches do not quantify the inherent discretization error -- often the largest source of uncertainty about the solution.
In contrast, our framework models this error and additionally admits a broader class of discretizations.
\citet{Wang2021BayesNonlinPDE,Kraemer2022PNMOL} propose GP-based solvers for strong
formulations of time-dependent nonlinear PDEs by leveraging finite-difference
approximations to the differential operator and linearization-based approximate
inference.
While it is possible to apply such methods to linear PDEs, the finite difference
approximation of the differential operator introduces additional estimation error.
In contrast, the evaluation of the differential operator in our method is exact.
\citet{Cockayne2017PNPDEInv,Raissi2017LinDEGP,Girolami2021statFEM} also apply their
methods to solve PDE-constrained (Bayesian) inverse problems.
\citet{Sarkka2011LinOpsSPDEGP} directly infers the right-hand side of a linear PDE in
strong formulation by observing measurements of the solution through the associated
Green's function.
Our approach also builds a belief over an unknown right-hand side without requiring
access to a Green's function.
%
The aforementioned methods use the closure of Gaussian processes under conditioning on
observations of the sample paths through a linear operator without proof.
\citet{Owhadi2018CondGaussHilbert} show how to condition Gaussian \emph{measures} on an
orthogonal direct sum of separable Hilbert spaces on observations of one of the
summands.
However, this result does not apply to separable \emph{Banach} spaces such as Hölder spaces, which are ubiquitous in the study of strong solutions of linear PDEs.
Furthermore, when it can be applied, it does not translate to Gaussian \emph{processes} without significant effort.\footnote{The theoretical results of an earlier version of this work were based on the result by \citet{Owhadi2018CondGaussHilbert}. In order to generalize our framework to Banach spaces, we've adopted a different proof strategy.}
Our work therefore provides the theoretical basis for conditioning Gaussian
\emph{processes} on observations of their sample paths made through an arbitrary bounded
linear operator with values in $\R^n$.
Recent results about the sample spaces of GPs
\citep{Steinwart2019SamplePathProps,Kanagawa2018GPKernMeth} ensure the applicability of
our work to practical GP regression problems.
From a practitioner's perspective this allows the modeling flexibility of Gaussian processes via the kernel, while ensuring that conditioning on observations of the sample paths through a linear operator is possible.
To our knowledge this is the first complete proof of this widely used property of GPs.
Thus, \cref{thm:gp-inference-linfctls} provides the theoretical basis for physics-informed GP regression, including
the aforementioned methods for the solution of PDEs.
In our work, it enables conditioning on information operators constructed from e.g.~PDEs,
boundary conditions and general integral equations.

  \section{Gaussian Process Inference with Linear Operator Observations}
\label{sec:affine-gp-inference}
Our framework fundamentally relies on the fact that when a Gaussian process prior is
conditioned on linear observations of its paths, one obtains a closed-form posterior.
This section provides the theoretical foundation for this result.
While this property is used widely in the literature
(see e.g.~\citet{%
  Graepel2003LinOpEqGP,%
  Rasmussen2006GPML,%
  Sarkka2011LinOpsSPDEGP,%
  Sarkka2013InfDimFiltSmooth,%
  Cockayne2017PNPDEInv,%
  Raissi2017LinDEGP,%
  Agrell2019GPLinOpIneq,%
  Albert2019GPLinDE,%
  Kraemer2022PNMOL%
}),
no proof of its general form where observations are made via \emph{bounded linear %
  operators} mapping a separable Banach function spaces into $\R^n$, instead of
\emph{finite-dimensional linear maps} on a finite number of point evaluations exists, to
the best of our knowledge.
\citet{Owhadi2018CondGaussHilbert} give a proof of a related property for Gaussian
\emph{measures} on separable Hilbert spaces.
Here, we extend their results to the case of Gaussian \emph{processes}.
While these perspectives are closely related, significant technical attention needs to
be paid for this result to transfer to the GP case.
For our framework this is essential such that we can leverage the modeling capabilities
provided by specifying a kernel as described in \cref{sec:encoding-prior-knowledge}.

To state the result, let $\mogpprior \sim \gp{\vec{m}}{\mat{k}}$ be a (multi-output) GP
prior with index set $\gpidcs$, $\linfctls{L} \colon \paths{\mogpprior} \to \R^n$ a
linear operator acting on the paths of $\mogpprior$, and
$\rvec{\epsilon} \sim \gaussian{\vec{\mu}}{\mat{\Sigma}}$
a Gaussian random vector in $\R^n$ with $\indp{\rvec{\epsilon}}{\mogpprior}$.
We need to compute the conditional random process
\begin{equation*}
  \condrv{\mogpprior \given \linfctlsat{L}{\mogpprior} + \rvec{\epsilon} = \vec{y}}
\end{equation*}
for some $\vec{y} \in \R^n$.
This object is defined as the family
\(
\left( \condrv*{\mogpprior \given \linfctlsat{L}{\mogpprior} + \rvec{\epsilon} = \vec{y}} \right)
\defeq
\{ \condrv{\mogpprior(x, \cdot) \given E} \}_{x \in \gpidcs},
\)
of conditional random variables\footnotemark, where
$(\Omega, \borelsigalg{\Omega}, \prob{})$
is the probability space on which both $\mogpprior$ and $\rvec{\epsilon}$ are defined,
$E$ is the event
$E \defeq \preim{\rvec{h}}(\set{\vec{y}}) \in \borelsigalg{\Omega}$,
and $\rvec{h}$ is the random variable
\begin{equation*}
  \rvec{h} \colon \Omega \to \R^n,
  \omega \mapsto \linfctlsat{L}{\mogpprior(\cdot, \omega)} + \rvec{\epsilon}(\omega).
\end{equation*}
\footnotetext{
  Here, we need to work with \emph{regular conditional probability measures}
  \citep{Klenke2014Probability}, since the event $E$ typically has probability
  0.
}%
We refer to \cref{app:proofs-gp-inference} for definitions of the objects mentioned above.
For instance, in \cref{sec:pde-solution-inference}, we use
$\linfctls{L} \defeq \left( \linopat{D}{\cdot}(\vec{x}_i) \right)_{i = 1}^n$, where
$\linop{D}$ is a linear differential operator,
as well as $\linfctlsat{L}{\mogpprior} \defeq (\mogpprior(\vec{x}_i))_{i = 1}^n$,
and, in \cref{sec:cpu-stationary-1d}, we additionally use
\begin{equation*}
  \linfctlsat{L}{\mogpprior} = \rintegral[\gpidcs]{\vec{x}}{\mogpprior(\vec{x})}.
\end{equation*}
It is well-known that $\rvec{h}$ is a Gaussian random vector
\(
\rvec{h} \sim \gp{%
  \linfctlsat{L}{\vec{m}} + \vec{\mu}
}{%
  \LkL{\linfctls{L}}{\mat{k}} + \mat{\Sigma}
},
\)
where $\LkL{\linfctls{L}}{\mat{k}} \in \R^{n \times n}$ with
\(
\left( \LkL{\linfctls{L}}{\mat{k}} \right)_{i_1,i_2}
= \evallinop{\linfctlselem{L}_{i_1}}{%
  \vec{x} \mapsto (\evallinop{\linfctlselem{L}_{i_2}}{\vec{k}_{j_1,:}(\vec{x}, \cdot)})_{j_1 = 1}^n
},
\)
and that the conditional random process is a (multi-output) Gaussian process
\begin{equation*}
  \condrv*{\mogpprior \given \linfctlsat{L}{\mogpprior} + \rvec{\epsilon} = \vec{y}}
  \sim
  \gp{\vec{m}^{\condrv{\mogpprior \given \vec{y}}}}{\mat{k}^{\condrv{\mogpprior \given \vec{y}}}}
\end{equation*}
with conditional moments given by
\begin{align*}
  \vecelem{m}^{\condrv{\mogpprior \given \vec{y}}}_i(\vec{x})
   & = \vecelem{m}_i(\vec{x}) +
  \linfctlsat{L}{\vec{k}_{:,i}(\cdot, \vec{x})}\T
  \left( \LkL{\linfctls{L}}{\mat{k}} + \mat{\Sigma} \right)\inv
  \left( \vec{y} - (\linfctlsat{L}{\vec{m}} + \vec{\mu}) \right),
  \qquad \text{and}                                  \\
  \matelem{k}^{\condrv{\mogpprior \given \vec{y}}}_{i_1,i_2}(\vec{x}_1, \vec{x}_2)
   & = \matelem{k}_{i_1,i_2}(\vec{x}_1, \vec{x}_2) +
  \linfctlsat{L}{\vec{k}_{:,i_1}(\cdot, \vec{x}_1)}\T
  \left( \LkL{\linfctls{L}}{\mat{k}} + \mat{\Sigma} \right)\inv
  \linfctlsat{L}{\vec{k}_{:,i_2}(\cdot, \vec{x}_2)}.
\end{align*}
Since the above are nontrivial claims about potentially ill-behaved
infinite-dimensional objects, a proof is important, be it just to
identify a precise set of assumptions about the objects at play, ensuring the result holds.
For instance, the statement that $\rvec{h}$ is a random vector, i.e.~a measurable
function, is highly nontrivial.
To remedy this situation, a major contribution of this work are
\cref{thm:gp-inference-linfctls,cor:gp-inference-linop-evals}
and their proof in \cref{app:proofs-gp-inference}, which prove the claims above under
realistic assumptions.
Hence, besides being the theoretical basis for this work,
\cref{thm:gp-inference-linfctls,cor:gp-inference-linop-evals}
also provide theoretical backing for many of the publications cited above.
Our results identify a set of mild assumptions, which are easy to verify and
widely-applicable in practical applications.
\Cref{asm:gp-gm} constitutes the common set of assumptions shared by
\cref{thm:gp-inference-linfctls,cor:gp-inference-linop-evals}.

\begin{restatable}{mainassumption}{LinOpGPInferenceAssumptions}
  \label{asm:gp-gm}
  Let $\gpprior \sim \gp{m}{k}$ be a Gaussian process prior with index set $\gpidcs$
  on the probability space $(\Omega, \sigalg, \prob{})$,
  whose paths lie in a real separable reproducing kernel Banach space (RKBS) $\gppathsp \subset \maps{\gpidcs}{\R}$
  such that $\omega \mapsto \gpprior(\cdot, \omega)$ is a $\gppathsp$-valued Gaussian
  random variable.
\end{restatable}

For instance, for a 1D domain $\dom \subset \R$, a GP prior with half-integer Matérn
kernel with smoothness parameter $\nu = p + \frac{1}{2}$ fulfills \cref{asm:gp-gm}
with $\gppathsp = \bucdfns{p}{\dom}$, i.e.~the space of $p$-times differentiable
functions with bounded and uniformly continuous derivatives.
Similar results hold in multiple dimensions and for other kernels.
See \cref{sec:prior-selection} for more information on prior selection.

\begin{table}
  \caption{
    \cref{thm:gp-inference-linfctls} provides the theoretical basis to condition on
    (affine) observations of a Gaussian process.
    While results like conditioning on derivative evaluations are used ubiquitously
    (e.g.~for monotonic GPs, Bayesian optimization, probabilistic numerical PDE solvers,
    \dots) a complete proof does not exist in the literature, to the best of our
    knowledge.
  }
  \centering
  \begin{tabular}{llcc}
  \toprule
  Observation                    & Information operator                                                                                     & Proof known? & Reference                           \\
  \midrule
  Point evaluation               & \( \mogpprior(\vec{x}) \)                                                                                & \cmark       & \cite{Bishop2006PRML}               \\
  Finite-dim. affine map         & \( \mat{A} \mogpprior(\vec{X}) + \vec{b} \)                                                              & \cmark       & \cite{Bishop2006PRML}               \\
  Point evaluation of derivative & \( \mpderivat[\abs{\vec{\alpha}}]{\pdiff[\vec{\alpha}]{\vec{x}}}{\mogppriorelem_i}{\vec{x}}{\vec{x}'} \) & \xmark       & \cref{cor:gp-inference-linop-evals} \\
  Integral                       & \( \lintegral[\gpidcs]{\mu}{\vec{x}}{\inprod{\vec{\psi}(\vec{x})}{\mogpprior(\vec{x})}} \)               & \xmark       & \cref{thm:gp-inference-linfctls}    \\
  General affine functionals     & \( \linfctlsat{L}{\mogpprior} + \vec{b} \)                                                               & \xmark       & \cref{thm:gp-inference-linfctls}    \\
  \bottomrule
\end{tabular}
  \label{tab:linop-gp-inference}
\end{table}

\Cref{thm:gp-inference-linfctls} enables affine observations, in which the GP sample
paths enter through one or multiple continuous linear functionals.
For example, we used \cref{thm:gp-inference-linfctls} in \cref{sec:cpu-stationary-1d} to
condition on observations of an integral of a GP's paths and in \cref{sec:mwr-info-ops}
to condition on projections of the paths.
To state the result conveniently, we introduce some notation.
\begin{mainnotation}
  \label{not:L1kL2adj-linfunctls}
  Let \cref{asm:gp-gm} hold and let $\linfctls{L} \colon \gppathsp \to \R^n$
  and $\tilde{\linfctls{L}} \colon \gppathsp \to \R^{\tilde{n}}$ be bounded linear
  operators.
  By $\LkL{\linfctls{L}}{k}[\tilde{\linfctls{L}}] \in \R^{n_1 \times n_2}$ we denote
  the matrix with entries
  \begin{equation*}
    (\LkL{\linfctls{L}}{k}[\tilde{\linfctls{L}}])_{ij}
    \defeq
    \evallinop{\linfctls{L}}{%
      \vec{x} \mapsto \evallinop{\tilde{\linfctls{L}}}{k(\vec{x}, \cdot)}_j
    }_i.
  \end{equation*}
\end{mainnotation}
The order in which the operators $\linfctls{L}$, $\tilde{\linfctls{L}}$ are
applied to the arguments of $k$ does not matter, i.e.~%
\(
(\LkL{\linfctls{L}}{k}[\tilde{\linfctls{L}}])_{ij}
= \evallinop{\linfctls{L}}{\vec{x} \mapsto \evallinop{\tilde{\linfctls{L}}}{k(\vec{x}, \cdot)}_j}_i
= \evallinop{\tilde{\linfctls{L}}}{\vec{x} \mapsto \evallinop{\linfctls{L}}{k(\cdot, \vec{x})}_i}_j.
\)
This motivates the parenthesis-free notation $\LkL{\linfctls{L}}{k}[\tilde{\linfctls{L}}]$
introduced above.
\begin{restatable}{maintheorem}{ThmGPInferenceLinFunctls}
  \label{thm:gp-inference-linfctls}
  Let \cref{asm:gp-gm} hold and let $\linfctls{L} \colon \gppathsp \to \R^n$ be a
  bounded linear operator.
  Then
  \begin{equation}
    \label{eqn:gp-linfunctls-predictive}
    \linfctlsat{L}{\gpprior} \sim \gaussian{
      \linfctlsat{L}{m}
    }{
      \LkL{\linfctls{L}}{k}
    }.
  \end{equation}
  Let $\rvec{\epsilon} \sim \gaussian{\vec{\mu}}{\mat{\Sigma}}$ be an $\R^n$-valued
  Gaussian random vector with $\indp{\rvec{\epsilon}}{\gpprior}$.
  Then, for any $\vec{y} \in \R^n$,
  \begin{equation}
    \label{eqn:gp-linfunctls-posterior}
    \condrv{\gpprior \given \linfctlsat{L}{\gpprior} + \rvec{\epsilon} = \vec{y}}
    \sim \gp{
      m^{\condrv{\gpprior \given \vec{y}}}
    }{
      k^{\condrv{\gpprior \given \vec{y}}}
    },
  \end{equation}
  with conditional mean and covariance function given by
  \begin{equation}
    \label{eqn:gp-linfunctls-posterior-mean}
    m^{\condrv{\gpprior \given \vec{y}}}(\vec{x})
    = m(\vec{x}) +
    \linfctlsat{L}{k(\vec{x}, \cdot)}\T
    \left( \LkL{\linfctls{L}}{k} + \mat{\Sigma} \right)\pinv
    \left( \vec{y} - \left( \linfctlsat{L}{m} + \vec{\mu} \right) \right),
  \end{equation}
  and
  \begin{equation}
    \label{eqn:gp-linfunctls-posterior-cov}
    k^{\condrv{\gpprior \given \vec{y}}}(\vec{x}_1, \vec{x}_2)
    = k(\vec{x}_1, \vec{x}_2) -
    \linfctlsat{L}{k(\vec{x}_1, \cdot)}\T
    \left( \LkL{\linfctls{L}}{k} + \mat{\Sigma} \right)\pinv
    \linfctlsat{L}{k(\cdot, \vec{x}_2)}.
  \end{equation}
\end{restatable}

Finally, we turn to \cref{cor:gp-inference-linop-evals}, which is the result that is
most widely-used throughout the literature \citep{Graepel2003LinOpEqGP,Sarkka2011LinOpsSPDEGP,Sarkka2013InfDimFiltSmooth,Cockayne2017PNPDEInv,Raissi2017LinDEGP,Agrell2019GPLinOpIneq,Albert2019GPLinDE,Kraemer2022PNMOL}.
It shows how Gaussian processes can be conditioned on point evaluations of the image of
their paths under a linear operator, provided that the linear operator is bounded and
maps into a separable Banach function space, on which point evaluation is continuous.
Moreover, it shows that, under these conditions, the image of the GP under the linear
operator is itself a Gaussian process.
Again, we introduce some notation to facilitate stating the result.
\begin{restatable}{mainnotation}{NotLkLadjLinOpEvals}
  \label{not:L1kL2adj-linops}
  Let \cref{asm:gp-gm} hold and let $\linop{L}_i \colon \gppathsp \to \gppathsp_i$ for
  $i = 1, 2$ be bounded linear operators mapping into real separable RKBSs
  $\gppathsp_i \subset \maps{\gpidcs_i}{\R}$, respectively.
  In analogy to \cref{not:L1kL2adj-linfunctls}, we define the bivariate functions
  \begin{alignat}{7}
    \kL{k}{\linop{L}_2}
    \colon & \setsym{X}   & \times & \setsym{X}_2 &  & \to \R,\, &  & (\vec{x},\,   &  & \vec{x}_2 & ) & \mapsto \evallinop{\linop{L}_2}{k(\vec{x}, \cdot)}(\vec{x}_2),                       \\
    \Lk{\linop{L}_1}{k}
    \colon & \setsym{X}_1 & \times & \setsym{X}   &  & \to \R,\, &  & (\vec{x}_1,\, &  & \vec{x}   & ) & \mapsto \evallinop{\linop{L}_1}{k(\cdot, \vec{x})}(\vec{x}_1), \qquad \text{and}     \\
    \LkL{\linop{L}_1}{k}[\linop{L}_2]
    \colon & \setsym{X}_1 & \times & \setsym{X}_2 &  & \to \R,\, &  & (\vec{x}_1,\, &  & \vec{x}_2 & ) & \mapsto \evallinop{\linop{L}_1}{(\kL{k}{\linop{L}_2})(\cdot, \vec{x}_2)}(\vec{x}_1).
  \end{alignat}
\end{restatable}
\begin{restatable}{maincorollary}{CorrGPInferenceLinOpEvals}
  \label{cor:gp-inference-linop-evals}
  Let \cref{asm:gp-gm} hold and let $\linop{L} \colon \gppathsp \to \tilde{\gppathsp}$
  be a linear operator mapping into a real vector space
  $\tilde{\gppathsp} \subset \maps{\tilde{\gpidcs}}{\R}$
  such that $\delta_\vec{\tilde{x}} \circ \linop{L}$ is bounded for all
  $\tilde{x} \in \tilde{\gpidcs}$.
  Then
  \begin{equation}
    \label{eqn:gp-linop-evals-predictive}
    \linopat{L}{\gpprior} \sim \gp{
      \linopat{L}{m}
    }{
      \LkL{\linop{L}}{k}
    }.
  \end{equation}
  Let $\rvec{\epsilon} \sim \gaussian{\vec{\mu}}{\mat{\Sigma}}$ with values in $\R^n$
  and $\indp{\rvec{\epsilon}}{\gpprior}$.
  Then, for $\tilde{\mat{X}} = (\tilde{\vec{x}}_i)_{i = 1}^n \in \tilde{\gpidcs}^n$
  and $\vec{y} \in \R^n$,
  \begin{equation}
    \label{eqn:gp-linop-evals-posterior}
    \condrv{%
      \gpprior
      \given
      \linopat{L}{\gpprior}(\tilde{\mat{X}}) + \rvec{\epsilon} = \vec{y}
    }
    \sim \gp{
      m^{\condrv{\gpprior \given \vec{y}}}
    }{
      k^{\condrv{\gpprior \given \vec{y}}}
    }
  \end{equation}
  with
  \begin{equation}
    \label{eqn:gp-linop-evals-posterior-mean}
    m^{\condrv{\gpprior \given \vec{y}}}(\vec{x})
    \defeq
    m(\vec{x}) +
    (\kL{k}{\linop{L}})(\vec{x}, \tilde{\mat{X}})\T
    \left((\LkL{\linop{L}}{k})(\tilde{\mat{X}}, \tilde{\mat{X}}) + \mat{\Sigma} \right)\pinv
    \left( \vec{y} - \left( \linopat{L}{m}(X) + \vec{\mu} \right) \right)
  \end{equation}
  and
  \begin{equation}
    \label{eqn:gp-linop-evals-posterior-cov}
    k^{\condrv{\gpprior \given \vec{y}}}(\vec{x}_1, \vec{x}_2)
    \defeq k(\vec{x}_1, \vec{x}_2) -
    (\kL{k}{\linop{L}})(\vec{x}_1, \tilde{\mat{X}})\T
    \left((\LkL{\linop{L}}{k})(\tilde{\mat{X}}, \tilde{\mat{X}}) + \mat{\Sigma} \right)\pinv
    (\Lk{\linop{L}}{k})(\tilde{\mat{X}}, \vec{x}_2)
  \end{equation}
  If additionally $\tilde{\gpidcs} = \gpidcs$, then
  \begin{equation}
    \label{eqn:gp-linop-evals-joint}
    \begin{pmatrix}
      \gpprior \\
      \linopat{L}{\gpprior}
    \end{pmatrix}
    \sim \multigp{%
      m \\
      \linopat{L}{m}
    }{%
      k                 & \kL{k}{\linop{L}}                 \\
      \Lk{\linop{L}}{k} & \LkL{\linop{L}}{k}
    }.
  \end{equation}
\end{restatable}
\begin{remark}
  The assumptions about $\linop{L}$ from \cref{cor:gp-inference-linop-evals}
  are fulfilled if $\tilde{\gppathsp}$ is an RKBS and $\linop{L}$ is bounded.
  However, these conditions are not necessary.
\end{remark}
\Cref{cor:gp-inference-linop-evals} is the theoretical basis for most of
\cref{sec:cpu-stationary-1d}.
For $\linop{L} = \id[\gppathsp]$, we recover standard GP regression as a special case.
Finally, both \Cref{thm:gp-inference-linfctls} and \Cref{cor:gp-inference-linop-evals}
apply also to vector-valued Gaussian processes.

\begin{remark}[Multi-Output Gaussian Processes]
  \label{rmk:gp-inference-multi-output}
  \Cref{thm:gp-inference-linfctls,cor:gp-inference-linop-evals} also apply to
  multi-output GPs $\mogpprior$.
  In this case, we interpret the sample paths
  $\mogpprior(\cdot, \omega) \colon \gpidcs \to \R^{d'}$
  of the multi-output GP as sample paths
  \(
  \tilde{\gpprior}(\cdot, \omega) \colon I \times \gpidcs \to \R,\
  \tilde{\gpprior}((i, \vec{x}), \omega) \defeq \mogppriorelem_i(\vec{x}, \omega)
  \)
  of a regular GP with index set $I \times \gpidcs \to \R$, where
  $I = \set{1, \dotsc, d'}$ (see \cref{sec:gp}).
  We also generalize notation like $\LkL{\linfctls{L}}{\mat{k}}$ accordingly.
\end{remark}

  \section{Conclusion}
In this work, we developed a probabilistic framework for the solution of (systems of)
linear partial differential equations, which can be interpreted as physics-informed
Gaussian process regression.
It enables the seamless fusion of (1) a-priori known, provable properties of the system
of interest, (2) exact and partial mechanistic information, (3) subjective domain
expertise, as well as, (4) noisy empirical measurements into a unified scientific model.
This model outputs a consistent uncertainty estimate, which quantifies the inherent approximation error in addition to the uncertainty arising from partially-known physics, as well as limited-precision measurements.
Our framework fundamentally relies on the closure of Gaussian processes under conditioning on observations of their sample paths through an arbitrary bounded linear operator.
While this result has been used ubiquitously in the literature, a rigorous proof for linear operator observations, as needed in the PDE setting, did not exist prior to this work to the best of our knowledge.
Our work generalizes and unifies several related formulations of GP-PDE inference.
Importantly, our formulation extends these ideas to virtually all popular methods for PDE
simulation, revealing them to be a form of Gaussian process inference and in turn clarifying the underlying (probabilistic) assumptions.
More specifically, by choosing a specific prior and information operator in our framework, it recovers methods of weighted residuals, a popular family of numerical methods for the solution of (linear) PDEs, which includes generalized Galerkin methods such as finite element and spectral methods.
This demonstrates that classical linear PDE solvers can be generalized in their functionality to include approximate input data and equipped with a structured uncertainty estimate.
Our work outlines a general framework for the integration of mechanistic building blocks
in the form of information operators derived from e.g.~linear PDEs into probabilistic
models.
Our case study shows that the language of information operators is a powerful toolkit
for aggregating heterogeneous sources of partial information in a joint probabilistic
model, especially in the context of physics-informed machine learning.
This opens up several interesting lines of research.
For example, the choice of prior and information operator are not fixed and can be specifically chosen for the problem at hand.
The design of adaptive information operators, which actively collect information based on the current belief about the solution could prove to be a promising research direction.
Further, the uncertainty estimate about the solution could be used to inform experimental design choices.
For example, in the case study from \Cref{sec:cpu-stationary-1d}, the posterior belief can be used to optimize the locations of the digital thermal sensors in future CPU designs.
Finally, it remains an open question whether this framework can be adapted to nonlinear partial differential equations in a similar manner to how many classic methods solve a sequence of linearized problems to approximate the solution of a nonlinear PDE.

  \acks{%
    MP, PH and JW gratefully acknowledge financial support by
    the European Research Council through ERC StG Action 757275 / PANAMA;
    the DFG Cluster of Excellence ``Machine Learning - New Perspectives for Science'',
    EXC 2064/1, project number 390727645;
    the German Federal Ministry of Education and Research (BMBF) through the Tübingen AI
    Center (FKZ: 01IS18039A);
    and funds from the Ministry of Science, Research and Arts of the State of
    Baden-Württemberg.
    The authors thank the International Max Planck Research School for Intelligent
    Systems (IMPRS-IS) for supporting MP and JW.
    IS thanks the Deutsche Forschungsgemeinschaft (DFG, German Research Foundation) for
    supporting this work by funding EXC 2075-390740016 under Germany's Excellence
    Strategy. IS also acknowledges support by the Stuttgart Center for Simulation Science
    (SimTech).

    Finally, the authors are grateful to Natha\"el Da Costa and Filip Tronarp for many invaluable discussions concerning the theoretical part of this work.
  }

  \appendix

  \section{Proofs for Section~\ref*{sec:mwr-info-ops}}
\begin{proof}[Proof of \cref{prop:gp-mwr-info-op-fin-trial}]
  By \cref{thm:gp-inference-linfctls}, we have
  \begin{align*}
    \vecelem{m}^{\condrv{\solvvprior \given \mwrmat, \mwrrhs}}_i(\vec{x})
     & = \vecelem{m}_i(\vec{x}) +
    \evallinop{(\mwrmat \mwrcoordproj)}{\vec{k}_{:,i}(\cdot, \vec{x})}\T
    \left( \LkL{(\mwrmat \mwrcoordproj)}{\mat{k}} \right)\inv
    \left( \mwrrhs - \evallinop{\mwrmat \mwrcoordproj}{\vec{m}} \right)             \\
     & = \vecelem{m}_i(\vec{x}) +
    \evallinop{\mwrcoordproj}{\vec{k}_{:,i}(\cdot, \vec{x})}\T \mwrmat\T
    \left( \mwrmat \mat{\Sigma}_\mwrcoordsprior \mwrmat\T \right)\inv
    \mwrmat \left( \mwrmat\inv \mwrrhs - \evallinop{\mwrcoordproj}{\vec{m}} \right) \\
     & = \vecelem{m}_i(\vec{x}) +
    \evallinop{\mwrcoordproj}{\vec{k}_{:,i}(\cdot, \vec{x})}\T
    \mat{\Sigma}_\mwrcoordsprior\inv
    \left( \mwrmat\inv \mwrrhs - \evallinop{\mwrcoordproj}{\vec{m}} \right).
  \end{align*}
  Since $\mwrtrialproj$ is a bounded projection, we have
  \(
  \solsp
  = \range{\mwrtrialproj} \oplus \kernel{\mwrtrialproj}
  = \mwrtrialsp \oplus \kernel{\mwrtrialproj}
  \)
  \citep[see][Section 5.16]{Rudin1991FuncAna},
  where each $\solvv \in \solsp$ decomposes uniquely into $\solvv = \solvv_{\mwrtrialsp} + \solvv_{\kernel{\mwrtrialproj}}$
  with $\solvv_{\mwrtrialsp} \in \mwrtrialsp$ and $\solvv_{\kernel{\mwrtrialproj}} \in \kernel{\mwrtrialproj}$.
  It is clear that \( \solvv_{\mwrtrialsp} = \evallinop{\mwrtrialproj}{\solvv}, \) and
  \(
  \solvv_{\kernel{\mwrtrialproj}}
  = \evallinop{\left(\id - \mwrtrialproj\right)}{\solvv}
  = \evallinop{\proj{\kernel{\mwrtrialproj}}}{\solvv}.
  \)
  This implies
  \begin{align*}
    \evallinop{\mwrcoordproj}{\vec{m}^{\condrv{\solvvprior \given \mwrmat, \mwrrhs}}}
     & = \evallinop{\mwrcoordproj}{\vec{m}}
    + \underbrace{\LkL{\mwrcoordproj}{\mat{k}}}_{= \mat{\Sigma}_\mwrcoordsprior}
    \mat{\Sigma}_\mwrcoordsprior\inv
    \left( \mwrmat\inv \mwrrhs - \evallinop{\mwrcoordproj}{\vec{m}} \right)                            \\
     & = \evallinop{\mwrcoordproj}{\vec{m}} + \mwrmat\inv \mwrrhs - \evallinop{\mwrcoordproj}{\vec{m}} \\
     & = \mwrmat\inv \mwrrhs
    = \mwrsolcoords.
  \end{align*}
  Hence, we have
  \begin{align*}
    \evallinop{\mwrtrialproj}{\vec{m}^{\condrv{\solvvprior \given \mwrmat, \mwrrhs}}}
    = \sum_{i = 1}^m \left( \evallinop{\proj{\R^n}}{\vec{m}^{\condrv{\solvvprior \given \mwrmat, \mwrrhs}}} \right)_i \mwrtrialvv^{(i)}
    = \sum_{i = 1}^m \mwrsolcoordselem_i \mwrtrialvv^{(i)}
    = \mwrsol
    \in \mwrtrialsp
  \end{align*}
  and since $\solsp = \mwrtrialsp \oplus \kernel{\mwrtrialproj}$, the statement follows.
  Moreover, $\evallinop{\proj{\R^n}}{\vec{m}^{\condrv{\solvvprior \given \mwrmat, \mwrrhs}}}$
  is the mean of $\condrv{\mwrcoordsprior \given \mwrmat \mwrcoordsprior - \mwrrhs = \vec{0}}$
  and its covariance matrix is given by
  \begin{align*}
    \mat{\Sigma}^{\condrv{\mwrcoordsprior \given \mwrmat, \mwrrhs}}
     & = \mat{\Sigma}_\mwrcoordsprior - \mat{\Sigma}_\mwrcoordsprior \mwrmat\T \left( \mwrmat \mat{\Sigma}_\mwrcoordsprior \mwrmat\T \right)\inv \mwrmat \mat{\Sigma}_\mwrcoordsprior \\
     & = \mat{\Sigma}_\mwrcoordsprior - \mat{\Sigma}_\mwrcoordsprior \mwrmat\T (\mwrmat\T)\inv \mat{\Sigma}_\mwrcoordsprior\inv \mwrmat\inv \mwrmat \mat{\Sigma}_\mwrcoordsprior      \\
     & = \mat{\Sigma}_\mwrcoordsprior - \mat{\Sigma}_\mwrcoordsprior \mat{\Sigma}_\mwrcoordsprior\inv \mat{\Sigma}_\mwrcoordsprior
    = \mat{0}.
  \end{align*}
  Consequently, $\condrv{\mwrcoordsprior \given \mwrmat \mwrcoordsprior - \mwrrhs = \vec{0}} \sim \delta_{\mwrsolcoords}$.
\end{proof}

\begin{proof}[Proof of \cref{cor:gp-mean-mwr-recovery}]
  \begin{align*}
         & \evallinop{\proj{\kernel{\mwrtrialproj}}}{\vec{m}^{\condrv{\solvvprior \given \mwrmat, \mwrrhs}}}(\vec{x})                 \\
    = \  & \underbrace{\evallinop{\proj{\kernel{\mwrtrialproj}}}{\vec{m}}}_{= 0}(\vec{x}) +
    \LkL{(\delta_\vec{x} \circ \proj{\kernel{\mwrtrialproj}})}{\mat{k}}[(\mwrmat \mwrcoordproj)]
    \left( \LkL{(\mwrmat \mwrcoordproj)}{\mat{k}} \right)\inv
    \left( \mwrrhs - \evallinop{\mwrmat \mwrcoordproj}{\vec{m}} \right)                                                               \\
    = \  & \evallinop{\delta_\vec{x}}{\underbrace{\LkL{\proj{\kernel{\mwrtrialproj}}}{\mat{k}}[\mwrcoordproj]}_{= \mat{0}}} \mwrmat\T
    \left( \LkL{(\mwrmat \mwrcoordproj)}{\mat{k}} \right)\inv
    \left( \mwrrhs - \evallinop{\mwrmat \mwrcoordproj}{\vec{m}} \right)
    = \mat{0}
  \end{align*}
\end{proof}

\begin{proof}[Proof of \cref{prop:mwr-recovery-prior}]
  The process $\solvvprior$ can be constructed as the sum of independent samples from
  the processes $\evallinop{\mwrtrialproj}{\tilde{\solvvprior}}$ and
  $\evallinop{\mwrtrialproj}{\tilde{\solvvprior} - \tilde{\vec{m}}}$.
  This proves that the sample paths lie in $\solsp$.
  Since $\mwrtrialproj$ is idempotent, we have
  \(
  \proj{\kernel{\mwrtrialproj}} \mwrtrialproj
  = \mwrtrialproj - \mwrtrialproj^2
  = \mwrtrialproj - \mwrtrialproj
  = 0
  \)
  and
  \(
  \proj{\R^n} \mwrtrialproj
  = (\isomorphism{\R^m}{\mwrtrialsp})\inv \mwrtrialproj^2
  = (\isomorphism{\R^m}{\mwrtrialsp})\inv \mwrtrialproj
  = \proj{\R^n}.
  \)
  It follows that
  \begin{align*}
    \proj{\kernel{\mwrtrialproj}} \mat{k} \proj{\R^n}\adj
     & = \proj{\kernel{\mwrtrialproj}} \tilde{\mat{k}} \proj{\R^n}\adj
    - \underbrace{\proj{\kernel{\mwrtrialproj}} \mwrtrialproj}_{= 0} \tilde{\mat{k}}                                     \\
     & \qquad - \proj{\kernel{\mwrtrialproj}} \tilde{\mat{k}} \mwrtrialproj\adj \proj{\R^n}\adj
    + 2 \underbrace{\proj{\kernel{\mwrtrialproj}} \mwrtrialproj}_{= 0} \tilde{\mat{k}} \mwrtrialproj\adj \proj{\R^n}\adj \\
     & = \proj{\kernel{\mwrtrialproj}} \tilde{\mat{k}} \proj{\R^n}\adj
    - \proj{\kernel{\mwrtrialproj}} \tilde{\mat{k}} \underbrace{\left( \proj{\R^n} \mwrtrialproj \right)\adj}_{= \proj{\R^n}}\adj
    = 0.
  \end{align*}
\end{proof}

  \section{Proofs for Section~\ref*{sec:affine-gp-inference}}
\label{app:proofs-gp-inference}
Using the rules of linear-Gaussian inference \citep{Bishop2006PRML}, we can easily see
that
\begin{align*}
  \gpprior
   & \sim \gp{m}{k}                                                            \\
  \mat{A} \gpprior(\mat{X})
   & \sim \gaussian{\mat{A} m(\mat{X})}{\mat{A} k(\mat{X}, \mat{X}) \mat{A}\T} \\
  \condrv{\gpprior \given \mat{A} \gpprior(\mat{X}) + \rvec{b} = \vec{y}}
   & \sim \gp{m^{\gpprior \mid \vec{y}}}{k^{\gpprior \mid \vec{y}}},
\end{align*}
where $\mat{A} \in \R^{m \times n}$, $\mat{X} = (\vec{x}_i)_{i = 1}^n \in \gpidcs$,
$\rvec{b} \sim \gaussian{\vec{\mu}}{\mat{\Sigma}}$ with $\indp{\rvec{b}}{\gpprior}$ and
\begin{align*}
  m^{\gpprior \mid \vec{y}}(\vec{x})
   & \defeq m(\vec{x}) + k(\vec{x}, \mat{X}) \mat{A}\T (\mat{A} k(\mat{X}, \mat{X}) \mat{A}\T + \mat{\Sigma})\pinv (\vec{y} - (\mat{A} m + \vec{\mu}))           \\
  k^{\gpprior \mid \vec{y}}(\vec{x}_1, \vec{x}_2)
   & \defeq k(\vec{x}_1, \vec{x}_2) - k(\vec{x}_1, \mat{X}) \mat{A}\T (\mat{A} k(\mat{X}, \mat{X}) \mat{A}\T + \mat{\Sigma})\pinv \mat{A} k(\mat{X}, \vec{x}_2).
\end{align*}
It is tempting to think that the above also extends to more general linear
transformations of $\gpprior$ such as differentiation at a point $\vec{x} \in \gpidcs$
and integration.
Unfortunately, this is not the case, since the result from \citep{Bishop2006PRML}
heavily uses the fact that, by definition, evaluations of the Gaussian process at a
\emph{finite} set of points follow a joint Gaussian distribution.
However, differentiation at a point and integration are examples of linear functionals,
i.e.~linear maps from a vector space of functions to the real numbers, which operate on
an (uncountably) infinite subset of random variables.

To generalize the result above to general linear operators $\linfctls{L}$ mapping the
paths of $\gpprior$ into $\R^n$, we will take the following route:
\begin{enumerate}
  \item In \cref{sec:gp-random-function}, we will show that, under certain conditions on
        $\gpprior$ and $\gpidcs$, the function $\omega \mapsto \gpprior(\cdot, \omega)$
        is a Gaussian random variable with values in a separable Banach space
        $\gppathsp$ of real-valued functions on $\gpidcs$.
        We introduce Gaussian random variables on separable Banach spaces and their
        essential properties in \cref{sec:gm}.
  \item Under the assumption that $\linfctls{L}$ is continuous, we can use the transformation
        properties of Gaussian random variables on
        separable Banach spaces (see \cref{lem:gm-affine-transform}) to show that $\linfctlsat{L}{\gpprior}$ and for $\mat{X} \in \gpidcs^m$ also
        \begin{equation*}
          \begin{pmatrix}
            \gpprior(\mat{X}) \\
            \linfctlsat{L}{\gpprior}
          \end{pmatrix}
        \end{equation*}
        are Gaussian random variables on $\R^n$ and
        $\R^{m + n}$, respectively.
  \item Finally, in \cref{sec:proof-theorem-1}, we can then use the well-known
        linear-Gaussian inference theorem \citep{Bishop2006PRML} to show that
        $\condrv*{\gpprior \given \linfctlsat{L}{\gpprior} = \vec{y}}$
        is again a Gaussian process.
\end{enumerate}

In the following, $\borelsigalg{\gppathsp}$ denotes the Borel $\sigma$-algebra generated
by the norm topology on a Banach space $\gppathsp$.

\subsection{Gaussian Measures on Separable Banach Spaces}
\label{sec:gm}
As stated before, in many cases, the function $\omega \mapsto \gpprior(\cdot, \omega)$
will often turn out to be a Gaussian random variable with values in an infinite-dimensional
separable Banach space $\gppathsp \supseteq \paths{\gpprior}$ of real-valued functions
on $\gpidcs$ (see \cref{prop:gp-to-gm}).
\begin{definition}
  \label{def:gm-hilbert-space}
  Let $\bgrvsp$ be a real separable Banach space.
  A Borel probability measure $\mu$ on $(\bgrvsp, \borelsigalg{\bgrvsp})$ is called
  \emph{Gaussian} if every continuous linear functional $l \in \dualsp{\bgrvsp}$ is a
  univariate Gaussian random variable.
  A $\bgrvsp$-valued random variable is called Gaussian if its law is Gaussian.
\end{definition}

Just as for Gaussian random variables on Euclidean vector space $\R^n$, we can define a
mean and covariance (operator) for their counterparts on general separable Banach spaces.
\begin{proposition}
  \label{prop:gm-mean-cov}
  Let $\bgrv$ be a Gaussian random variable on $(\Omega, \sigalg, \prob{})$ with
  values in a real separable Banach space $\bgrvsp$.
  Then there is a unique $m_\bgrv \in \bgrvsp$ such that
  \(
  \evallinop{l}{m_\bgrv} = \expectation[\bgrv]{\evallinop{l}{\bgrv}}
  \)
  for any $l \in \dualsp{\bgrvsp}$.
  We refer to $m_\bgrv$ as the \emph{mean (vector)} of $\bgrv$.
  Moreover, there is a unique bounded linear operator
  $\linop{C}_\bgrv \colon \dualsp{\bgrvsp} \to \bgrvsp$
  such that
  \(
  \evallinop{l_2}{\evallinop{\linop{C}_\bgrv}{l_1}}
  = \covariance[\bgrv]{\evallinop{l_1}{\bgrv}}{\evallinop{l_2}{\bgrv}}
  \)
  for any $l_1, l_2 \in \dualsp{\bgrvsp}$, the so-called \emph{covariance operator} of
  $\bgrv$.
\end{proposition}
\begin{proof}
  Fernique's theorem \citep[Theorem 2.7]{DaPrato1992StochEq} implies that
  $\norm{\bgrv}[\bgrvsp] \in \L{p}(\Omega, \prob{})$ for all $p \in \N_{\ge 1}$.
  By assumption, $\bgrv$ is measurable and $\bgrvsp \supset \range{\bgrv}$ is
  separable, which means that $\bgrv$ is strongly measurable
  \citep[Section V.4, Pettis' Theorem]{Yosida1995FuncAna}.
  Since $\norm{\bgrv}[\bgrvsp] \in \L1(\Omega, \prob{})$, it follows that $\bgrv$ is
  Bochner integrable \citep[Section V.5, Theorem 1]{Yosida1995FuncAna}.
  Let $l \in \dualsp{\bgrvsp}$.
  By Corollary 2 in \citet[Section V.5]{Yosida1995FuncAna} we then have
  \begin{align*}
    \expectation[\bgrv]{\evallinop{l}{\bgrv}}
    = \lintegral[\Omega]{\prob{}}{\omega}{\evallinop{l}*{\bgrv(\omega)}}
    = \evallinop{l}[\bigg]{%
      \underbrace{\lintegral[\Omega]{\prob{}}{\omega}{\bgrv(\omega)}}_{\rdefeq m_\bgrv}
    }.
  \end{align*}
  Now assume that there is another $\tilde{m}_\bgrv \in \bgrvsp$ with
  \(
  \evallinop{l}{\tilde{m}_\bgrv} = \expectation[\bgrv]{\evallinop{l}{\bgrv}}.
  \)
  Then
  \begin{equation*}
    0
    = \expectation[\bgrv]{\evallinop{l}{\bgrv}} - \expectation[\bgrv]{\evallinop{l}{\bgrv}}
    = \evallinop{l}{\tilde{m}_\bgrv} - \evallinop{l}{m_\bgrv}
    = \evallinop{l}{\tilde{m}_\bgrv - m_\bgrv}
  \end{equation*}
  and since this holds for all $l \in \dualsp{\bgrvsp}$, it follows that
  $\tilde{m}_\bgrv = m_\bgrv$.

  Let $l_1 \in \dualsp{\bgrvsp}$. Then
  \(
  \omega \mapsto \evallinop{l_1}{\bgrv(\omega) - m_\bgrv} (\bgrv(\omega) - m_\bgrv)
  \)
  is clearly weakly measurable and, since $\bgrvsp$ is separable, also strongly
  measurable \citep[Section V.4, Pettis' Theorem]{Yosida1995FuncAna}.
  As above, Fernique's theorem shows that $\norm{\bgrv}[\bgrvsp] \in \L2(\Omega, \prob{})$.
  By the triangle inequality in $\bgrvsp$ and the fact that $\prob{}$ is a probability
  measure, we also have $\norm{\bgrv(\cdot) - m_\bgrv}[\bgrvsp] \in \L2(\Omega, \prob{})$.
  It follows that
  \begin{align*}
    \lintegral[\Omega]{\prob{}}{\omega}{\norm{\evallinop{l_1}{\bgrv(\omega) - m_\bgrv} (\bgrv(\omega) - m_\bgrv)}[\bgrvsp]}
     & =   \lintegral[\Omega]{\prob{}}{\omega}{\abs*{\evallinop{l_1}{\bgrv(\omega) - m_\bgrv}} \norm{\bgrv(\omega) - m_\bgrv}[\bgrvsp]}                      \\
     & \le \norm{l_1}[\dualsp{\bgrvsp}] \lintegral[\Omega]{\prob{}}{\omega}{\norm{\bgrv(\omega) - m_\bgrv}[\bgrvsp] \norm{\bgrv(\omega) - m_\bgrv}[\bgrvsp]} \\
     & = \norm{l_1}[\dualsp{\bgrvsp}] \norm[\big]{\omega \mapsto \norm{\bgrv(\omega) - m_\bgrv}[\bgrvsp]}[\L2(\Omega, \prob{})]^2 < \infty,
  \end{align*}
  where $\norm{l_1}[\dualsp{\bgrvsp_1}] < \infty$, since $l_1$ is continuous.
  Let $l_2 \in \dualsp{\bgrvsp}$.
  Again by Corollary 2 in \citet[Section V.5]{Yosida1995FuncAna}, we find that
  \begin{align*}
    \covariance{\evallinop{l_1}{\bgrv}}{\evallinop{l_2}{\bgrv}}
     & = \lintegral[\Omega]{\prob{}}{\omega}{
      \evallinop{l_1}{\bgrv(\omega) - m_\bgrv}
      \evallinop{l_2}{\bgrv(\omega) - m_\bgrv}
    }                                         \\
     & = \evallinop{l_2}[\bigg]{
    \lintegral[\Omega]{\prob{}}{\omega}{
    \underbrace{
    \evallinop{l_1}{\bgrv(\omega) - m_\bgrv}
    (\bgrv(\omega) - m_\bgrv)
    }_{\rdefeq \linopat{C}{l_1}}
    }
    }.
  \end{align*}
  $\linop{C}_\bgrv$ is bounded, since
  \begin{align*}
    \norm{\evallinop{\linop{C}_\bgrv}{l_1}}[\bgrvsp]
     & \le \lintegral[\Omega]{\prob{}}{\omega}{\norm{\evallinop{l_1}{\bgrv(\omega) - m_\bgrv} (\bgrv(\omega) - m_\bgrv)}[\bgrvsp]}                                                 \\
     & \le \norm{l_1}[\dualsp{\bgrvsp}] \underbrace{\norm[\big]{\omega \mapsto \norm{\bgrv(\omega) - m_\bgrv}[\bgrvsp]}[\L2(\Omega, \prob{})]^2}_{\rdefeq \norm{\linop{C}_\bgrv}},
  \end{align*}
  where $\norm{\linop{C}_\bgrv} < \infty$ because
  $\norm{\bgrv(\cdot) - m_\bgrv}[\bgrvsp] \in \L2(\Omega, \prob{})$.
  Uniqueness follows from an argument analogous to the one used to prove uniqueness of
  the mean.
\end{proof}
\begin{remark}
  One can show that the mean and the covariance operator of a Gaussian random variable
  with values in a separable Banach space identify its law uniquely.
  Hence, we often write $\gaussian{m}{\linop{C}}$ to denote Gaussian measures on
  separable Banach spaces.
\end{remark}

\subsubsection{Continuous Affine Transformations}
\label{sec:gm-affine-transforms}
Just as their finite-dimensional counterparts, Gaussian random variables with values in
separable Banach spaces are closed under linear transformations as long as they are
continuous.
Moreover, the expressions for the transformed mean and covariance operator are analogous
to the finite-dimensional case.
For instance, we will use this result to show that $\linfctlsat{L}{\gpprior}$ is an
$\R^n$-valued random variable.
\begin{lemma}
  \label{lem:gm-affine-transform}
  Let $\bgrv \sim \gaussian{m}{\linop{C}}$ be a Gaussian random variable on
  $(\Omega, \sigalg, \prob{})$ with values in a real separable Banach space $\bgrvsp$.
  Let $\linop{L} \colon \bgrvsp \to \tilde{\bgrvsp}$ be a bounded linear operator
  mapping into another real separable Banach space $\tilde{\bgrvsp}$.
  Then
  \(
  \linfctlsat{L}{\bgrv} \sim \gaussian{\linopat{L}{m}}{\linop{L} \linop{C} \linop{L}\dualop}.
  \)
\end{lemma}
\begin{proof}
  $\linop{L}$ is continuous and hence Borel measurable, which means that
  $\linopat{L}{\bgrv}$ is a $\tilde{\bgrvsp}$-valued random variable.
  Moreover, for $\tilde{l} \in \dualsp{\tilde{\bgrvsp}}$, we have
  $l \defeq \tilde{l} \circ \linop{L} \in \dualsp{\bgrvsp}$ and hence
  \(
  \evallinop{l}{\bgrv}
  = \evallinop{\tilde{l}}{\linopat{L}{\bgrv}}
  \)
  is Gaussian.
  This implies that $\linopat{L}{\bgrv}$ is a $\tilde{\bgrvsp}$-valued Gaussian
  random variable.
  Moreover, we have
  \(
  \expectation[\bgrv]{\evallinop{l}{\bgrv}}
  = \evallinop{l}{m}
  = \evallinop{\tilde{l}}{\linopat{L}{m}},
  \)
  i.e.~$\linopat{L}{m}$ is the mean of $\linopat{L}{\bgrv}$.
  Let $\tilde{l}_1, \tilde{l}_2 \in \dualsp{\bgrvsp}$ and define
  $l_i \defeq \tilde{l}_i \circ \linop{L} \in \dualsp{\tilde{\bgrvsp}}$ for $i = 1, 2$.
  Then
  \begin{equation*}
    \covariance[\bgrv]{\evallinop{l_1}{\bgrv}}{\evallinop{l_2}{\bgrv}}
    = \evallinop{l_2}{\linopat{C}{l_1}}
    = \evallinop{(\tilde{l}_2 \circ \linop{L})}{\linopat{C}{\tilde{l}_1 \circ \linop{L}}}
    = \evallinop{\tilde{l}_2}{\evallinop{(\linop{L} C \linop{L}\dualop)}{\tilde{l}_1}},
  \end{equation*}
  i.e.~$\linop{L} C \linop{L}\dualop$ is the covariance operator of $\linopat{L}{\bgrv}$.
\end{proof}

\subsection{Gaussian Processes as Gaussian Random Functions}
\label{sec:gp-random-function}
We now introduced all necessary preliminaries to show that, under certain assumptions on
$\gpprior$ and $\gpidcs$, the function $\omega \mapsto \gpprior(\cdot, \omega)$ is a
Gaussian random variable with values in a special kind of separable Banach space, namely
a \emph{reproducing kernel Banach space (RKBS)}:
\begin{definition}[{\citealt[Definition 2.1]{Lin2022RKBS}}]
  A \emph{reproducing kernel Banach space (RKBS)} $(\banachsp, \norm{\cdot}[\banachsp])$
  is a Banach space of real-valued functions on a nonempty set $\metricsp{X}$, on which
  the point evaluation functionals $\delta_\vec{x}$ for $\vec{x} \in \metricsp{X}$ are
  continuous.
\end{definition}

We formulate \cref{thm:gp-inference-linfctls,cor:gp-inference-linop-evals}, under
the following set of assumptions:
\LinOpGPInferenceAssumptions*

Generalizing an observation from \citet[Remark 1]{Rajput1972GPMeasures}, it becomes clear that \cref{asm:gp-gm} is often not about the GP $\gpprior$ itself, but rather about the topology of the space $\gppathsp$.
Denote by $\wstarseqcl(\setsym{L}) \defeq \set{\ell \in \dualsp{\gppathsp} \suchthat \exists \set{\ell_k}_{k \in \N} \subset \setsym{L} \colon \ell_k \to_{w*} \ell}$ the weak-* \emph{sequential} closure\footnotemark of a subset $\setsym{L} \subset \dualsp{\gppathsp}$ of the continuous dual space of a Banach space $\gppathsp$.
Moreover, $\wstarseqcl^n(\setsym{L}) = \wstarseqcl^{n - 1}(\setsym{L})$ and $\wstarseqcl^0(\setsym{L}) = \setsym{L}$.
\footnotetext{The weak-* sequential closure of $\setsym{L}$ is not to be confused with the weak-* closure of $\setsym{L}$. The two notions coincide if $\dualsp{\gppathsp}$ equipped with the weak-* topology is a sequential space, but for many of the dual spaces considered in this work this does not hold.}
\begin{theorem}
  \label{thm:gp-gm-w*-sequential-closure}
  Let $\gppathsp \subset \maps{\gpidcs}{\R}$ be a real separable RKBS and define $\setsym{L}_\delta \defeq \operatorname{span} \set{\delta_\vec{x}}_{\vec{x} \in \gpidcs} \subset \dualsp{\gppathsp}$.
  If there is an $n \in \N_0$ such that $\dualsp{\gppathsp} = \wstarseqcl^n(\setsym{L}_\delta)$, then \cref{asm:gp-gm} holds for any GP $\gpprior$ with paths in $\gppathsp$.
\end{theorem}
\begin{proof}
  First, we show by induction on $n \in \N_0$ that $\omega \mapsto \evallinop{l}{\gpprior(\cdot, \omega)}$ is a Gaussian random variable for every $l \in \wstarseqcl^n(\setsym{L}_\delta)$.
  For $n = 0$, we have $\wstarseqcl^n(\setsym{L}_\delta) = \setsym{L}_\delta$.
  Hence, for every $l \in \wstarseqcl^n(\setsym{L}_\delta) = \setsym{L}_\delta$, there are $m \in \N$, $\set{\alpha_k}_{k = 1}^m \subset \R$, and $\set{\vec{x}_k}_{k = 1}^m \subset \gpidcs$ such that
  \begin{equation*}
    \evallinop{l}{\gpprior(\cdot, \omega)}
    = \sum_{k = 1}^{m} \alpha_k \gpprior(\vec{x}_k, \omega).
  \end{equation*}
  Since $\gpprior$ is a GP, $\gpprior(\vec{x}_1, \cdot), \dotsc, \gpprior(\vec{x}_m, \cdot)$ is jointly Gaussian, which implies that $\omega \mapsto \evallinop{l}{\gpprior(\cdot, \omega)}$ is a Gaussian random variable.
  Now let $n > 0$ and fix $l \in \wstarseqcl^n(\setsym{L}_\delta)$.
  Then there is a sequence $\set{l_k}_{k \in \N} \subset \wstarseqcl^{n - 1}(\setsym{L}_\delta)$ such that $\evallinop{l_k}{f} \to \evallinop{l}{f}$ as $k \to \infty$ for every $f \in \gppathsp$.
  By the inductive hypothesis, we know that $\omega \mapsto \evallinop{l_k}{\gpprior(\cdot, \omega)}$ is Gaussian for every $k \in \N$.
  It follows that the function
  \(
  \omega
  \mapsto \evallinop{l}{\gpprior(\cdot, \omega)}
  = \lim_{k \to \infty} \evallinop{l_k}{\gpprior(\cdot, \omega)}
  \)
  is measurable \cite[Theorem 1.92]{Klenke2014Probability}.
  Moreover, as the pointwise limit of Gaussian random variables is again a Gaussian random variable, we conclude that $\omega \mapsto \evallinop{l}{\gpprior(\cdot, \omega)}$ is Gaussian.

  Under the assumption that $\dualsp{\gppathsp} = \wstarseqcl^n(\setsym{L}_\delta)$ for some $n \in \N_0$, the above shows that $\omega \mapsto \evallinop{l}{\gpprior(\cdot, \omega)}$ is a Gaussian random variable for every $l \in \dualsp{\gppathsp}$.
  In particular, the map $\omega \mapsto \evallinop{l}{\gpprior(\cdot, \omega)}$ is $\mathcal{F}$-$\borelsigalg{\R}$-measurable for all $l \in \dualsp{\gppathsp}$, i.e. $\omega \mapsto \gpprior(\cdot, \omega)$ is weakly measurable.
  By the separability of $\gppathsp$, it follows that $\omega \mapsto \gpprior(\cdot, \omega)$ is $\mathcal{F}$-$\borelsigalg{\banachsp}$-measurable \citep[Theorem A.3.7]{Bogachev1998GaussianMeasures}.
  This shows that $\omega \mapsto \gpprior(\cdot, \omega)$ is a $\gppathsp$-valued Gaussian random variable.
\end{proof}
\begin{corollary}
  \label{cor:gp-gm-w*-sequentially-dense}
  Let $\gppathsp \subset \maps{\gpidcs}{\R}$ be a real separable RKBS.
  If $\setsym{L}_\delta \defeq \operatorname{span} \set{\delta_\vec{x}}_{\vec{x} \in \gpidcs}$ lies weak-* \emph{sequentially} dense\footnotemark in $\dualsp{\gppathsp}$, then \cref{asm:gp-gm} holds for any GP $\gpprior$ with paths in $\gppathsp$.
  \footnotetext{As before, a weak-* sequentially dense set is not to be confused with a weak-* dense set, since the dual spaces considered in this work are not necessarily sequential w.r.t. the weak-* topology.}
\end{corollary}

In the following, we show that \cref{thm:gp-gm-w*-sequential-closure} applies to three important classes of Banach spaces, which often arise in the study of Gaussian Processes and PDEs.

\begin{proposition}
  \label{prop:gp-rkhs-paths-gm}
  The set $\setsym{L}_\delta \defeq \operatorname{span} \set{\delta_\vec{x}}_{\vec{x} \in \gpidcs}$ is weak-* sequentially dense in the dual $\dualsp{\rkhs{k}}$ of any separable RKHS $\rkhs{k} \subset \maps{\gpidcs}{\R}$.
\end{proposition}
\begin{proof}
  Let $l \in \dualsp{\rkhs{k}}$.
  By the Riesz representation theorem, there is $h \in \rkhs{k}$ such that $l = \inprod{h}{\cdot}[\rkhs{k}]$.
  Since $\operatorname{span} \set{k(\cdot, \vec{x})}_{x \in \gpidcs}$ lies dense in $\rkhs{k}$ \citep[Theorem 4.21]{Steinwart2008SVMs}, there is $\set{h_i}_{i \in \N} \subset \operatorname{span} \set{k(\cdot, \vec{x})}_{x \in \gpidcs}$ such that $h_i \to h$.
  For every $i \in \N$, define $l_i \defeq \inprod{h_i}{\cdot}[\rkhs{k}]$ and note that $\set{l_i}_{i \in \N} \subset \setsym{L}_\delta$ by the reproducing property.
  By the continuity of the inner product it follows that
  \begin{equation*}
    \evallinop{l}{f}
    = \inprod{h}{f}[\rkhs{k}]
    = \inprod*{\lim_{i \to \infty} h_i}{f}[\rkhs{k}]
    = \lim_{i \to \infty} \inprod{h_i}{f}[\rkhs{k}]
    = \lim_{i \to \infty} \evallinop{l_i}{f}
  \end{equation*}
  for every $f \in \rkhs{k}$.
\end{proof}

\begin{proposition}
  \label{prop:gp-cpaths-gm}
  The set $\setsym{L}_\delta \defeq \operatorname{span} \set{\delta_\vec{x}}_{\vec{x} \in \gpidcs}$ is weak-* sequentially dense in the dual $\dualsp{\gppathsp}$ of the space $\gppathsp = \cfns{\gpidcs}$ of real-valued continuous functions on a compact metric space $(\gpidcs, d_\gpidcs)$.
\end{proposition}
\begin{proof}
  Let $l \in \dualsp{\cfns{\gpidcs}}$.
  By the Riesz-Markov theorem \citep[Corollary 14.15]{Aliprantis2006InfAna}, there is a signed Borel measure $\lambda$ on $\gpidcs$ such that $\evallinop{l}{f} = \lintegral[\gpidcs]{\lambda}{\vec{x}}{f(\vec{x})}$.
  We need to show that there are $\set{l_k}_{k \in \N} \subset \setsym{L}_\delta$ such that $\evallinop{l}{f} \to \evallinop{l_k}{f}$ for every $f \in \cfns{\gpidcs}$.
  To do so, we modify a construction from \citet[Paragraph 4.22]{Alt2012FuncAna}.
  For $S \subset \gpidcs$, define $\operatorname{diam}(S) \defeq \sup_{\vec{x}, \vec{x}_0 \in S} d_\gpidcs(\vec{x}, \vec{x}_0)$.
  Since $\gpidcs$ is compact, there is a finite open cover $\{ \tilde{S}^{(k)}_i \}_{i = 1}^{n^{(k)}}$ of $\gpidcs$ with $\operatorname{diam}(\tilde{S}^{(k)}) < \frac{1}{k}$ for every $k \in \N$.
  Then $\{ S^{(k)}_i \}_{i = 1}^{n^{(k)}} \subset \gpidcs$ with
  \begin{equation*}
    S^{(k)}_i \defeq \tilde{S}^{(k)}_i \setminus \bigcup_{j < i} \tilde{S}^{(k)}_j
  \end{equation*}
  is also a cover of $\gpidcs$ and $S^{(k)}_i \in \borelsigalg{\gpidcs}$.
  Now choose $\vec{x}^{(k)}_i \in S^{(k)}_i$ for all $i = 1, \dotsc, n^{(k)}$ (w.l.o.g.~$S^{(k)}_i \neq \emptyset$).
  For any $f \in \cfns{\gpidcs}$, we define
  \begin{equation*}
    f_k \defeq \sum_{i = 1}^{n^{(k)}} f(\vec{x}^{(k)}_i) \chi_{S^{(k)}_i}.
  \end{equation*}
  Note that, since the $f_k$ are simple, we have
  \begin{equation*}
    \evallinop{l}{f_k}
    = \lintegral[\gpidcs]{\lambda}{\vec{x}}{f_k(\vec{x})}
    = \sum_{i = 1}^{n^{(k)}} f(\vec{x}^{(k)}_i) \lambda(S^{(k)}_i)
    \rdefeq \evallinop{l_k}{f},
  \end{equation*}
  where $l_k \in \setsym{L}_\delta$.
  We will now show that $f_k \to f$ uniformly, since this implies $\evallinop{l_k}{f} = \evallinop{l}{f_k} \to \evallinop{l}{f}$ by the dominated convergence theorem.

  By the Heine-Cantor theorem, $f \in \cfns{\gpidcs}$ is uniformly continuous.
  Thus, for every $\epsilon > 0$, there is $\delta(\epsilon) > 0$ such that $\abs{f(\vec{x}) - f(\vec{x}_0)} < \epsilon$ for $\vec{x}, \vec{x}_0 \in \gpidcs$ with $d_\gpidcs(\vec{x}, \vec{x}_0) < \delta(\epsilon)$.
  Now fix $\epsilon > 0$.
  Then for $k > \frac{1}{\delta(\epsilon)}$ and any $\vec{x} \in S^{(k)}_i$, we have
  \begin{equation*}
    d_\gpidcs(\vec{x}, \vec{x}^{(k)}_i)
    < \operatorname{diam}(S^{(k)}_i)
    < \frac{1}{k}
    < \delta(\epsilon)
  \end{equation*}
  and thus $\abs{f(\vec{x}) - f(\vec{x}^{(k)}_i)} < \epsilon$.
  Consequently,
  \begin{equation*}
    \norm{f - f_k}[\infty]
    = \max_{i = 1, \dotsc, n^{(k)}} \sup_{\vec{x} \in S^{(k)}_i} \abs{f(\vec{x}) - f(\vec{x}^{(k)}_i)}
    < \max_{i = 1, \dotsc, n^{(k)}} \epsilon
    = \epsilon.
  \end{equation*}
\end{proof}

The Banach spaces $\bucdfns{k}{\gpidcs}$ of $k$-times differentiable functions on an
open and bounded domain $\gpidcs \subset \R^d$ with bounded and uniformly continuous
partial derivatives and their subspaces (in particular the Hölder spaces) appear
naturally in the study of strong solutions to linear PDEs.
However, to allow for a flexible prior construction, we define a generalization of these
spaces.

\begin{definition}
  \label{def:bucdfns-bounded-domain}
  Let $\gpidcs \subset \R^d$ be open and bounded and let $A \subset \N_0^d$ be a
  non-empty downward closed set of multi-indices, i.e. $\vec{\beta} \in A$ implies
  $\vec{\alpha} \in A$ for every $\vec{\alpha} \in \N_0^d$ with
  $\vec{\alpha} \le \vec{\beta}$.
  We define $\bucdfns{A}{\gpidcs}$ as the space of real-valued functions $f$ on
  $\gpidcs$, for which all partial derivatives $\mipderiv{\vec{\alpha}}{f}$ with
  $\vec{\alpha} \in A$ are bounded and uniformly continuous.
\end{definition}

\begin{remark}
  \label{rmk:bucdfns-bounded-domain}
  One can show that $\bucdfns{A}{\gpidcs}$ equipped with the norm
  \begin{equation*}
    \norm{f}[\bucdfns{A}{\gpidcs}]
    \defeq
    \max_{\vec{\alpha} \in A}
    \sup_{\vec{x} \in \gpidcs}
    \abs{\mipderiv{\vec{\alpha}}{f}[\vec{x}]}
  \end{equation*}
  is a separable Banach space.
  Since every partial derivative $\mipderiv{\vec{\alpha}}{f}$ with $\vec{\alpha} \in A$
  is bounded and uniformly continuous, it has a unique, bounded, continuous extension to
  the closure $\closure{\gpidcs}$ of $\gpidcs$ \citep{Adams2003Sobolev}.
\end{remark}

\begin{proposition}
  \label{prop:gp-cdpaths-gm}
  Let $\bucdfns{A}{\gpidcs}$ be the separable Banach space of real-valued functions on an open and bounded domain $\gpidcs \subset \R^d$ with bounded and uniformly continuous partial derivatives $\mipderiv{\vec{\alpha}}{f}$ for all $f \in \bucdfns{A}{\gpidcs}$ and $\vec{\alpha} \in A$.
  Then $\dualsp{\bucdfns{A}{\gpidcs}} = \wstarseqcl^{m + 1}(\setsym{L}_\delta)$, where $\setsym{L}_\delta \defeq \operatorname{span} \set{\delta_\vec{x}}_{\vec{x} \in \gpidcs} \subset \dualsp{\bucdfns{A}{\gpidcs}}$ and $m = \max_{\vec{\alpha} \in A} \abs{\vec{\alpha}}$.
\end{proposition}
\begin{proof}
  In the following, we adapt the proof of Theorem 3.9 in \citep{Adams2003Sobolev}.
  We choose an arbitrary ordering $\vec{\alpha}_1, \dotsc, \vec{\alpha}_n$ of the
  multi-indices in $A$, i.e.~$A = \{ \vec{\alpha}_1, \dotsc, \vec{\alpha}_n \}$,
  where $n = \card{A}$.
  Let
  \begin{equation*}
    \banachsp[W]
    \defeq \set{%
      (\mipderiv{\vec{\alpha}_1}{f}, \dotsc, \mipderiv{\vec{\alpha}_n}{f})
      \where
      f \in \bucdfns{A}{\gpidcs}
    }
    \subset \cfns{\closure{\gpidcs}}^n,
  \end{equation*}
  where we interpret $\mipderiv{\vec{\alpha}_i}{f}$ as a function defined on the closure
  $\closure{\gpidcs}$ of $\gpidcs$ by the unique continuous extension mentioned in
  \cref{rmk:bucdfns-bounded-domain}.
  We equip $\cfns{\closure{\gpidcs}}^n$ and $\banachsp[W] \subset \cfns{\closure{\gpidcs}}^n$ with the norm
  \begin{equation*}
    \norm{\vec{f}}[\cfns{\closure{\gpidcs}}^n] \defeq \max_{i = 1, \dotsc, n} \norm{\vecelem{f}_i}[\cfns{\closure{\gpidcs}}].
  \end{equation*}
  Then $(\cfns{\closure{\gpidcs}}^n, \norm{\cdot}[\cfns{\closure{\gpidcs}}^n])$ is a separable Banach space \citep[Theorem 1.23]{Adams2003Sobolev}.
  Let $\linop{I} \colon \bucdfns{A}{\gpidcs} \to \banachsp[W]$ be the linear operator defined by
  $\linopat{I}{f}_i = \mipderiv{\vec{\alpha}_i}{f}$.
  The operator $\linop{I}$ is surjective and norm-preserving, and hence an isometric isomorphism.
  It follows that $\banachsp[W]$ is a closed subspace of $\cfns{\closure{\gpidcs}}^n$.

  Fix $l \in \dualsp{\bucdfns{A}{\gpidcs}}$.
  Then $l \circ \linop{I}\inv \in \dualsp{\banachsp[W]}$.
  By the Hahn-Banach theorem, there is a continuous extension $\tilde{l} \in \dualsp{(\cfns{\closure{\gpidcs}}^n)}$ of
  $l \circ \linop{I}\inv$ to $\cfns{\closure{\gpidcs}}^n$.
  This means that there are $\tilde{l}_1, \dotsc, \tilde{l}_n \in \dualsp{\cfns{\closure{\gpidcs}}}$,
  such that
  \begin{equation*}
    l
    = (l \circ \linop{I}\inv) \circ \linop{I}
    = \tilde{l} \circ \linop{I}
    = \sum_{i = 1}^n \tilde{l}_i \circ \mipderiv{\vec{\alpha}_i}{}.
  \end{equation*}
  By \cref{prop:gp-cpaths-gm}, there are $\set{\tilde{l}_{ij}}_{j \in \N} \subset \setsym{L}_\delta$ such that $\evallinop{\tilde{l}_{ij}}{f} \to \evallinop{\tilde{l}_i}{f}$ as $j \to \infty$ for all $f \in \cfns{\closure{\gpidcs}}$.
  Consequently, there are $n_{ij} \in \N$, $\set{c_{ijk}}_{k = 1}^{n_{ij}} \subset \R$, and $\set{\vec{x}_{ijk}}_{k = 1}^{n_{ij}} \subset \closure{\gpidcs}$ such that
  \begin{equation*}
    \evallinop{l}{f}
    = \lim_{j \to \infty} \sum_{i = 1}^n \evallinop{\tilde{l}^{(i)}_j}{\mipderiv{\vec{\alpha}_i}{f}}
    = \lim_{j \to \infty} \sum_{i = 1}^n \sum_{k = 1}^{n_{ij}} c^{(i)}_{jk} \mipderiv{\vec{\alpha}_i}{f}(\vec{x}_{ijk})
  \end{equation*}
  for all $f \in \bucdfns{A}{\gpidcs}$.
  We will detail the remainder of the proof only for $\abs{\vec{\alpha}} \le 1$ for all $\vec{\alpha} \in A$, since the proof of the general statement is a straightforward yet laborious extension of this special case.
  Assume without loss of generality that $\vec{\alpha}_1 = \vec{0}$ and $\vec{\alpha}_i = \vec{e}_{i - 1}$ for $2 \le i \le n \le d + 1$.
  Then,
  \begin{align*}
    \evallinop{l}{f}
     & = \lim_{j_0 \to \infty} \sum_{i = 1}^n \sum_{k = 1}^{n_{ij_0}} c_{ij_0k} \mipderiv{\vec{\alpha}_i}{f}(\vec{x}_{ij_0k})                                                                                                                                                          \\
     & = \lim_{j_0 \to \infty} \sum_{k = 1}^{n_{ij_0}} c_{1j_0k} f(\vec{x}_{1j_0k}) + \sum_{i = 2}^n c_{ij_0k} \mipderiv{\vec{e}_{i - 1}}{f}(\vec{x}_{ij_0k})                                                                                                                          \\
     & = \lim_{j_0 \to \infty} \lim_{j_1 \to \infty} \underbrace{\sum_{k = 1}^{n_{ij_0}} c_{1j_0k} f(\vec{x}_{1j_0k}) + \sum_{i = 2}^n \frac{c_{ij_0k}}{h_{j_1}} \left( f(\vec{x}_{ij_0k} + h_{j_1} \vec{e}_{i - 1}) - f(\vec{x}_{ij_0k}) \right)}_{\rdefeq \evallinop{l_{j_0j_1}}{f}}
  \end{align*}
  for any null sequence $\set{h_j}_{j \in \N} \subset \R$ and all $f \in \bucdfns{A}{\gpidcs}$.
  Since $l_{j_0j_1} \in \setsym{L}_\delta$, it follows that $l \in \wstarseqcl^2(\setsym{L}_\delta)$.
  The general case, i.e. $l \in \wstarseqcl^{m + 1}(\setsym{L}_\delta)$ with $m \defeq \max_{i = 1, \dotsc, n} \abs{\vec{\alpha}_i}$, can be shown by induction on $m \in \N_0$.
  Hence, $\dualsp{\bucdfns{A}{\gpidcs}} = \wstarseqcl^{m + 1}(\setsym{L}_\delta)$.
\end{proof}

Having investigated conditions under which $\omega \mapsto \gpprior(\cdot, \omega)$ is
a Gaussian random variable, it remains to show what its mean and covariance operator are.
Perhaps unsurprisingly, it turns out that they are strongly related to the mean and
covariance function of the GP.
\begin{proposition}
  \label{prop:gp-to-gm}
  Let \cref{asm:gp-gm} hold.
  Then $m \in \gppathsp$, $k(\vec{x}, \cdot) \in \gppathsp$ for all $\vec{x} \in \gpidcs$, and the mean and
  covariance operator of $\omega \mapsto \gpprior(\cdot, \omega)$ are given by $m$ and
  \begin{equation}
    \linop{C}_k \colon \dualsp{\gppathsp} \to \gppathsp, \quad
    l \mapsto \evallinop{\linop{C}_k}{l}(\vec{x}) = \evallinop{l}{k(\vec{x}, \cdot)},
  \end{equation}
  respectively.
\end{proposition}
\begin{proof}
  Since $\omega \mapsto \gpprior(\cdot, \omega)$ is Gaussian, its mean and covariance
  operator $m_\gpprior$ and $\linop{C}_\gpprior$ exist by
  \cref{prop:gm-mean-cov} and we have
  \(
  m(\vec{x})
  = \expectation[\prob{}]{\gpprior(\vec{x})}
  = \expectation[\gpprior]{\evallinop{\delta_\vec{x}}{\gpprior}}
  = \evallinop{\delta_\vec{x}}{m_\gpprior}
  \)
  for all $\vec{x} \in \gpidcs$ and
  \begin{equation*}
    k(\vec{x}_1, \vec{x}_2)
    = \covariance[\prob{}]{\gpprior(\vec{x}_1)}{\gpprior(\vec{x}_2)}
    = \covariance[\gpprior]{\evallinop{\delta_{\vec{x}_1}}{\gpprior}}{\evallinop{\delta_{\vec{x}_2}}{\gpprior}}
    = \evallinop{\linop{C}_\gpprior}{\delta_{\vec{x}_1}}(\vec{x}_2)
  \end{equation*}
  for all $\vec{x}_1, \vec{x}_2 \in \gpidcs$, since all point evaluation functionals are
  continuous on $\gppathsp$.
  Hence, $m = m_\gpprior \in \gppathsp$ and
  $k(\vec{x}, \cdot) = \evallinop{\linop{C}_\gpprior}{\delta_\vec{x}} \in \gppathsp$
  for all $\vec{x} \in \gpidcs$.
  Additionally, for all $l \in \dualsp{\gppathsp}$ and $\vec{x} \in \gpidcs$,
  \begin{align*}
    \evallinop{\linop{C}_\gpprior}{l}(\vec{x})
    = \covariance[\gpprior]{\evallinop{l}{\gpprior}}{\evallinop{\delta_\vec{x}}{\gpprior}}
    = \covariance[\gpprior]{\evallinop{\delta_\vec{x}}{\gpprior}}{\evallinop{l}{\gpprior}}
    = \evallinop{l}{\evallinop{\linop{C}_\gpprior}{\delta_\vec{x}}}
    = \evallinop{l}{k(\vec{x}, \cdot)}
    = \evallinop{\linop{C}_k}{l}(\vec{x}).
  \end{align*}
  This shows that $\linop{C}_\gpprior = \linop{C}_k$.
\end{proof}

\subsection{Proof of Theorem~\ref{thm:gp-inference-linfctls}}
\label{sec:proof-theorem-1}
Using the results from \cref{sec:gm,sec:gp-random-function}, particularly
\cref{prop:gp-to-gm,lem:gm-affine-transform}, we can now conduct the proof of
\cref{thm:gp-inference-linfctls,cor:gp-inference-linop-evals} as outlined in the
beginning of \cref{app:proofs-gp-inference}.

The main theorem deals with the case in which we observe the GP through a finite number
of linear functionals.
This happens when conditioning on integral observations or on (Galerkin) projections as
in \cref{sec:cpu-stationary-1d,sec:mwr-info-ops}.
\ThmGPInferenceLinFunctls*
\begin{proof}
  By \cref{lem:gm-affine-transform}, $\linfctlsat{L}{\gpprior}$ is a Gaussian random
  variable with mean $\linfctlsat{L}{m}$ and covariance matrix $\mat{\Sigma}$ with
  \begin{equation*}
    \matelem{\Sigma}_{ij}
    = \evallinop{\linfctlselem{L}_i}{
      \linopat{C}{\linfctlselem{L}_j}
    }
    = \evallinop{\linfctlselem{L}_i}{
      x \mapsto \evallinop{\linfctlselem{L}_j}{k(x, \cdot)}
    }
    = (\LkL{\linfctls{L}}{k})_{ij},
  \end{equation*}
  where we used \cref{prop:gm-mean-cov,prop:gp-to-gm}.
  This proves \cref{eqn:gp-linfunctls-predictive}.
  Now let $\mat{X} = (\vec{x}_i)_{i = 1}^m \in \gpidcs^m$ and consider
  \begin{equation*}
    \tilde{L} \colon U \to \R^{m + n},
    f \mapsto \pvec{f(\mat{X}) \\ \linfctlsat{L}{f}}.
  \end{equation*}
  $\tilde{L}$ is linear and bounded.
  Hence, by \cref{lem:gm-affine-transform} and the stability properties of independent
  Gaussian random variables on $\R^{m + n}$,
  \begin{equation*}
    \pvec{\gpprior(\mat{X}) \\ \linfctlsat{L}{\gpprior} + \rvec{\epsilon}}
    = \evallinop{\tilde{\linfctls{L}}}{\gpprior} + \pvec{\mat{0}_{m \times n} \\ \mat{I}_n} \rvec{\epsilon}
    \sim \gaussian{
      \pvec{m(\mat{X}) \\ \linfctlsat{L}{m} + \vec{\mu}}
    }{
      \begin{pmatrix}
        k(\mat{X}, \mat{X})                  & \mat{\Sigma}^{\mat{X}, \linfctls{L}} \\
        \mat{\Sigma}^{\linfctls{L}, \mat{X}} & \LkL{\linfctls{L}}{k} + \mat{\Sigma}
      \end{pmatrix}
    },
  \end{equation*}
  where
  \begin{equation*}
    \matelem{\Sigma}^{\mat{X}, \linfctls{L}}_{i,j}
    = \evallinop{\delta_{\vec{x}_i}}{\linopat{C}{\linfctlselem{L}_j}}
    = \linopat{C}{\linfctlselem{L}_j}(\vec{x}_i)
    = \evallinop{\linfctlselem{L}_j}{k(\vec{x}_i, \cdot)}
    = (\LkL{\delta_{\mat{X}}}{k}[\linfctls{L}])_{i, j}
  \end{equation*}
  and $\mat{\Sigma}^{\linfctls{L}, \mat{X}} = (\mat{\Sigma}^{\mat{X}, \linfctls{L}})\T$.
  By the well-known conditioning theorem for Gaussian random variables in $\R^{m + n}$,
  we arrive at
  \begin{equation*}
    \condrv{\gpprior(\mat{X}) \given \linfctlsat{L}{\gpprior} + \rvec{\epsilon} = \vec{y}}
    \sim \gaussian{
    \vec{\mu}^{\condrv{\gpprior(\mat{X}) \given \vec{y}}}
    }{
    \mat{\Sigma}^{\condrv{\gpprior(\mat{X}) \given \vec{y}}}
    },
  \end{equation*}
  with
  \begin{align*}
    \vec{\mu}^{\condrv{\gpprior(\mat{X}) \given \vec{y}}}
     & = m(\mat{X}) +
    (\LkL{\delta_{\mat{X}}}{k}[\linfctls{L}])
    (\LkL{\linfctls{L}}{k} + \mat{\Sigma})\pinv
    (\vec{y} - (\linfctlsat{L}{m} + \vec{\mu})) \\
    \intertext{and}
    \mat{\Sigma}^{\condrv{\gpprior(\mat{X}) \given \vec{y}}}
     & = k(\mat{X}, \mat{X}) -
    (\LkL{\delta_{\mat{X}}}{k}[\linfctls{L}])
    (\LkL{\linfctls{L}}{k} + \mat{\Sigma})\pinv
    (\LkL{\linfctls{L}}{k}[\delta_{\mat{X}}]).
  \end{align*}
  This shows that $\gpprior = \{ \omega \mapsto \gpprior(\vec{x}, \omega) \}_{\vec{x} \in \gpidcs}$
  is a Gaussian process on the probability space
  \begin{equation*}
    (
    \Omega,
    \sigalg,
    \prob{\cdot \given \linfctlsat{L}{\gpprior} + \rvec{\epsilon} = \vec{y}}
    )
  \end{equation*}
  where $\prob{\cdot \given \linfctlsat{L}{\gpprior} + \rvec{\epsilon} = \vec{y}}$ is a
  regular conditional probability whose existence is guaranteed, since $\R^n$ is
  Polish.
  The mean and covariance function of the conditional process evaluated at $\vec{x}_i$
  are given by $\vecelem{\mu}^{\condrv{\gpprior(\mat{X}) \given \vec{y}}}_i$ and
  $\matelem{\Sigma}^{\condrv{\gpprior(\mat{X}) \given \vec{y}}}_{i,i}$.
  Since the points $\mat{X}$ were chosen arbitrarily, this holds for any
  $\vec{x} \in \gpidcs$, which proves \cref{eqn:gp-linfunctls-posterior-mean,%
    eqn:gp-linfunctls-posterior-cov}.
\end{proof}
Finally, we address the archetypical case, in which both the prior $\gpprior$ and the
prior predictive $\linopat{L}{\gpprior} + \rvec{\epsilon}$ are Gaussian processes.
This happens if the linear operator maps into a function space, in which point
evaluation is continuous.
In this article, this case occurred in \cref{sec:solving-pdes-bayesian-inference,%
  sec:cpu-stationary-1d}, where we inferred the strong solution of a PDE from
observations of the PDE residual at a finite number of domain points.
\CorrGPInferenceLinOpEvals*
\begin{proof}
  The linear operator
  \(
  \linopat{L}{\cdot}(\tilde{\mat{X}}) \colon \gppathsp \to \R^n
  \)
  is bounded, since
  \(
  \linopat{L}{\cdot}(\tilde{\mat{X}})_i = \delta_{\tilde{\vec{x}}_i} \circ \linop{L}
  \)
  is bounded by assumption.
  Hence, \cref{eqn:gp-linop-evals-predictive,eqn:gp-linop-evals-posterior,%
    eqn:gp-linop-evals-posterior-mean,eqn:gp-linop-evals-posterior-cov} follow directly
  from \cref{thm:gp-inference-linfctls}.
  Now let $\mat{X} = (\vec{x}_i)_{i = 1}^{m + m'} \in \gpidcs^{m + m'}$.
  Then the linear operator
  \begin{equation*}
    \linfctls{L}_\mat{X} \colon \gppathsp \to \R^{m + m'},
    f \mapsto
    \begin{pmatrix}
      f(\vec{x}_1)
       & \hdots
       & f(\vec{x}_m)
       & \linopat{L}{f}(\vec{x}_{m + 1})
       & \hdots
       & \linopat{L}{f}(\vec{x}_{m + m'})
    \end{pmatrix}\T
  \end{equation*}
  is bounded and $\evallinop{\linfctls{L}_\mat{X}}{\gpprior}$ is Gaussian by
  \cref{thm:gp-inference-linfctls}.
  This implies that
  $\set{\omega \mapsto \morproc{f}^\linop{L}(\vec{x}, \omega)}_{\vec{x} \in \gpidcs}$
  with
  \begin{equation*}
    \morproc{f}^\linop{L}(\vec{x}, \omega)
    \defeq
    \begin{pmatrix}
      f(\vec{x}, \omega) \\
      \linopat{L}{f(\cdot, \omega)}(\vec{x})
    \end{pmatrix}
  \end{equation*}
  is a 2-output Gaussian process.
  By \cref{lem:gm-affine-transform,prop:gm-mean-cov,prop:gp-to-gm}, its mean
  function is given by
  \begin{equation*}
    \vec{m}^\linop{L}(\vec{x})
    =
    \begin{pmatrix}
      \expectation[\prob{}]{\evallinop{\delta_\vec{x}}{\gpprior}} \\
      \expectation[\prob{}]{\evallinop{(\delta_\vec{x} \circ \linop{L})}{\gpprior}}
    \end{pmatrix}
    =
    \begin{pmatrix}
      m(\vec{x}) \\
      \linopat{L}{m}(\vec{x})
    \end{pmatrix}
  \end{equation*}
  and its covariance function is given by
  \begin{align*}
    \mat{k}^\linop{L}(\vec{x}_1, \vec{x}_2)
     & =
    \begin{pmatrix}
      \covariance[\prob{}]{\evallinop{\delta_{\vec{x}_1}}{\gpprior}}{\evallinop{\delta_{\vec{x}_2}}{\gpprior}}
       & \covariance[\prob{}]{\evallinop{\delta_{\vec{x}_1}}{\gpprior}}{\evallinop{(\delta_{\vec{x}_2} \circ \linop{L})}{\gpprior}}                   \\
      \covariance[\prob{}]{\evallinop{(\delta_{\vec{x}_1} \circ \linop{L})}{\gpprior}}{\evallinop{\delta_{\vec{x}_2}}{\gpprior}}
       & \covariance[\prob{}]{\evallinop{(\delta_{\vec{x}_1} \circ \linop{L})}{\gpprior}}{\evallinop{(\delta_{\vec{x}_2} \circ \linop{L})}{\gpprior}}
    \end{pmatrix} \\
     & =
    \begin{pmatrix}
      \evallinop{\delta_{\vec{x}_1}}{\evallinop{\linop{C}_k}{\delta_{\vec{x}_2}}}
       & \evallinop{\delta_{\vec{x}_1}}{\evallinop{\linop{C}_k}{\delta_{\vec{x}_2} \circ \linop{L}}}                   \\
      \evallinop{(\delta_{\vec{x}_1} \circ \linop{L})}{\evallinop{\linop{C}_k}{\delta_{\vec{x}_2}}}
       & \evallinop{(\delta_{\vec{x}_1} \circ \linop{L})}{\evallinop{\linop{C}_k}{\delta_{\vec{x}_2} \circ \linop{L}}}
    \end{pmatrix}             \\
     & =
    \begin{pmatrix}
      k(\vec{x}_1, \vec{x}_2)                   & (\kL{k}{\linop{L}})(\vec{x}_1, \vec{x}_2)  \\
      (\Lk{\linop{L}}{k})(\vec{x}_1, \vec{x}_2) & (\LkL{\linop{L}}{k})(\vec{x}_1, \vec{x}_2)
    \end{pmatrix}.
  \end{align*}
  This proves \cref{eqn:gp-linop-evals-joint}.
\end{proof}

\subsection{On Prior Selection}
\label{sec:prior-selection}
In order to apply \cref{thm:gp-inference-linfctls} in practice, we need to construct our
GP prior $\gpprior$ such that
\begin{enumerate}
  \item $\omega \mapsto \gpprior(\cdot, \omega)$ is a Gaussian random variable on some
        suitably chosen RKBS $\gppathsp$, and
  \item $\linfctls{L} \colon \gppathsp \to \R^n$ is bounded.
\end{enumerate}
Luckily, we can use existing results about the path spaces of Gaussian processes together with \cref{thm:gp-gm-w*-sequential-closure,prop:gp-rkhs-paths-gm,prop:gp-cpaths-gm,prop:gp-cdpaths-gm} to verify these assumptions.
It is tempting to choose $\gppathsp = \rkhs{k}$, i.e.~the RKHS of the GP's kernel $k$.
However, this is only valid if $\rkhs{k}$ is finite-dimensional.
\begin{remark}
  \label{rmk:gp-path-space-rkhs}
  Let $\gpprior \sim \gp{m}{k}$ be a Gaussian process with index set $\gpidcs$ and let
  $\rkhs{k}$ be the RKHS of the covariance function $k$.
  If $\dim \rkhs{k} = \infty$, then the sample paths of $\gpprior$ almost surely do not
  lie in $\rkhs{k}$.
  We refer to \cite[Section 4]{Kanagawa2018GPKernMeth} and
  \cite{Steinwart2019SamplePathProps} for more details on RKHS sample spaces.
\end{remark}
In the following, we will give example constructions of appropriate priors for GP
regression tasks with linear operator observations.

\subsubsection{Priors for GP Regression with Linear Operator Observations}
\label{sec:priors-gp-inference-linfctls}
The spaces $\bucdfns{A}{\gpidcs}$ from \cref{def:bucdfns-bounded-domain}, particularly
$\bucdfns{\vec{\beta}}{\gpidcs}$ with
$A \defeq \set{\vec{\alpha} \in \N_0^d \suchthat \vec{\alpha} \le \vec{\beta}}$
and $\bucdfns{k}{\gpidcs}$ with
$A \defeq \set{\vec{\alpha} \in \N_0^d \suchthat \abs{\vec{\alpha}} \le k}$,
are useful sample spaces for many GP regression tasks, since a large number of
practically relevant observation functionals, including point evaluations of the paths
and their partial derivatives, as well as integrals of the paths, are bounded on these
spaces.
Even though the functions in $\bucdfns{A}{\gpidcs}$ are, technically speaking, only
defined on the open and bounded set $\gpidcs \subset \R^d$, we can treat them as
functions on the closure $\closure{\gpidcs}$ of $\gpidcs$ by continuous extension (see
\cref{rmk:bucdfns-bounded-domain} for more details).
In other words, we can evaluate functions in $\bucdfns{A}{\gpidcs}$ on the
boundary $\boundary{\gpidcs}$ of $\gpidcs$.

To fulfill \cref{asm:gp-gm}, it remains to verify that the sample paths of a given GP prior (almost surely) lie in $\bucdfns{A}{\gpidcs}$.
Under the assumption that $m \in \bucdfns{A}{\gpidcs}$, \citet{DaCosta2023SamplePathRegularity} show that this can be done by studying the regularity of the covariance function $k$.
They also provide readily applicable results for a wide variety of covariance functions used in practice.

\begin{example}[Tensor Products of Matérn Covariances]
  Tensor products of 1D Matérn covariance functions
  \begin{equation*}
    k_{\vec{\nu}}(\vec{x}_1, \vec{x}_2)
    = \prod_{i = 1}^d k_{\vecelem{\nu}_i}(\vecelem{x}_{1,i}, \vecelem{x}_{2,i})
  \end{equation*}
  are a particularly convenient choice of prior covariance function, since their hyperparameters directly control the differentiability of the sample paths independently for each input dimension.
  For an open and bounded domain $\gpidcs \subset \R^d$, Propositions 10 and 21 in \citet{DaCosta2023SamplePathRegularity} imply that samples from a Gaussian process with mean function $m$ and covariance function $k_{\vec{\nu}}$ lie in $\bucdfns{\vec{\beta}}{\gpidcs}$ (with probability 1) if $\vecelem{\nu}_i > \vecelem{\beta}_i$ and
  $m \in \bucdfns{\vec{\beta}}{\gpidcs}$.
  Any point-evaluated partial derivative $f \mapsto \mipderiv{\vec{\alpha}}{f}[\vec{x}]$
  with $\vec{\alpha} \le \vec{\beta}$ and $\vec{x} \in \closure{\gpidcs}$ is continuous
  on $\bucdfns{\vec{\beta}}{\gpidcs}$.
\end{example}

In \cref{sec:cpu-stationary-1d}, we use tensor products of Matérn covariance functions
to construct the GP priors.
In particular, we choose $\vecelem{\nu}_i = \frac{5}{2} = 2 + \frac{1}{2}$, which
implies that the sample paths of the prior lie in $\bucdfns{(2, 2)}{\gpidcs}$ and that
all point-evaluated differential operators of order $\le 2$ are continuous linear
functionals on the sample space.
Hence, the assumptions of \cref{cor:gp-inference-linop-evals} are fulfilled, which
means that the inference procedure used in this \lcnamecref{sec:cpu-stationary-1d} is
supported by our theoretical results above.

The sample paths of Gaussian processes with multivariate Matérn covariance functions $k_\nu$ (almost surely) lie in the Banach space $\gppathsp = \bucdfns{p}{\gpidcs}$ if $\nu > p$ \citep[Proposition 10]{DaCosta2023SamplePathRegularity}.

For Gaussian processes with smooth covariance functions like the \emph{Gaussian}/\emph{exponentiated quadratic} or the \emph{rational quadratic} covariance functions, \cref{asm:gp-gm} holds for $\gppathsp = \bucdfns{p}{\gpidcs}$ for all $p \in \N_0$ \citep[Corollary 13]{DaCosta2023SamplePathRegularity}.
Informally speaking, for the Gaussian covariance function, this can be seen as a limit of the argument above, since a Matérn covariance function approaches the Gaussian covariance function for $\nu \to \infty$.

Gaussian processes with parametric covariance function
\(
k(\vec{x}_1, \vec{x}_2)
= \vec{\phi}(\vec{x}_1)\T \mat{\Sigma} \vec{\phi}(\vec{x}_2)
\)
with features $\vec{\phi} \colon \gpidcs \to \R^m$ and $\mat{\Sigma} \in \R^{m \times m}$
positive-(semi)definite have paths in $\gppathsp$ if $\vecelem{\phi}_i \in \gppathsp$.
In this case, \cref{asm:gp-gm} is also satisfied, since the Gaussian measure can be
explicitly constructed as the law of the random function $\rvec{w}\T \vec{\phi}$,
where $\rvec{w} \sim \gaussian{\vec{0}}{\mat{\Sigma}}$.

\subsubsection{Priors for Inferring Weak Solutions of Linear PDEs}
\label{sec:priors-weak-solutions}
A typical choice for the solution spaces $\solsp$ of linear PDEs in weak formulation
(see \cref{sec:weak-formulation}), are \emph{Sobolev spaces} \citep{Adams2003Sobolev}.
Unfortunately, it is impossible to construct a Gaussian process prior $\solprior$,
whose paths are elements of a Sobolev space $\solsp$.
This is due to the fact that Sobolev spaces are, technically speaking, not function
spaces, but rather spaces of equivalence classes $[u]_{\sim}$ of functions, which are
equal almost everywhere \citep{Adams2003Sobolev}.
By contrast, the path spaces of Gaussian processes are proper function spaces, which
means that, in this setting, $\paths{\solprior} \subseteq \solsp$ is impossible.

Fortunately, if the path space $\gppathsp \supset \paths{\solprior}$ of $\solprior$ can
be continuously embedded in $\solsp$, i.e.~there is a continuous and injective linear
operator $\iota \colon \gppathsp \to \solsp$, commonly referred to as an
\emph{embedding}, then the inference procedure above can still be applied.
If such an embedding exists, we can interpret the paths of the GP as elements of
$\gppathsp$ by applying $\iota$ implicitly.
For instance, $\evallinop{\weakbilin}{\solprior, \testfn}$ is then a shorthand
notation for $\evallinop{\weakbilin}{\evallinop{\iota}{\solprior}, \testfn}$.
Fortunately, since the embedding is assumed to be continuous, the conditions for GP
inference with linear operator observations are still met when applying $\iota$
implicitly.
The canonical choice for the embedding in the case of Sobolev spaces is
$\evallinop{\iota}{u} = [u]_{\sim}$.

\begin{example}[Mat\'ern covariances and Sobolev spaces]
  \label{ex:matern-sample-spaces}
  \citet{Kanagawa2018GPKernMeth} show that, under certain assumptions, RKHS sample
  spaces of GP priors with Mat\'ern covariance functions are continuously embedded in
  Sobolev spaces whose smoothness depends on the parameter $\nu$ of the covariance
  function.
  To be precise, let $\dom \subset \R^d$ be open and bounded with Lipschitz boundary
  such that the cone condition \citep[Definition 4.6]{Adams2003Sobolev} holds.
  Denote by $k_{\nu, l}$ the Mat\'ern kernel with smoothness parameter $\nu > 0$ and
  lengthscale $l > 0$.
  Then, with probability 1, the sample paths of a Gaussian process $\solprior$ with
  covariance function $k_{\nu, l}$ are contained in any RKHS $\rkhs{k_{\nu', l'}}$ with
  $l' > 0$ and
  \begin{equation}
    0 < \underbrace{\nu' + \frac{d}{2}}_{\rdefeq m'} < \nu
  \end{equation}
  \citep[Corollary 4.15 and Remark 4.15]{Kanagawa2018GPKernMeth}, i.e.~%
  $\paths{\solprior} \subset \rkhs{k_{\nu', l'}}$.
  Moreover, if $m' \in \mathbb{N}$, then the RKHS $\rkhs{k_{\nu', l'}}$ is
  norm-equivalent to the Sobolev space $\sobolev{m'}[\dom]$
  \citep[Example 2.6]{Kanagawa2018GPKernMeth}.
  This implies that the canonical embedding
  \begin{equation}
    \iota \colon \rkhs{k_{\nu', l'}} \to \sobolev{m'}[\dom],
    \gpprior(\cdot, \omega) \mapsto [\gpprior(\cdot, \omega)]_{\sim_{\sobolev{m'}[\dom]}}
  \end{equation}
  is continuous.
\end{example}

For $\solsp = \sobolev{m'}[\dom]$, the example above shows that the Matérn covariance
function $k_{\nu, l}$ with $\nu = m' + \epsilon$ for any $\epsilon > 0$
leads to an admissible GP prior.
The choice $\epsilon = \frac{1}{2}$ makes evaluating the covariance function
particularly efficient \citep{Rasmussen2006GPML}.
For instance, in \cref{sec:mwr-info-ops}, we used $\nu = \frac{3}{2} = 1 + \frac{1}{2}$
for a weak form linear PDE with solution space $\solsp = \sobolev{1}[\dom]$.
However, the elements of the Sobolev space $\sobolev{m}[\dom]$ are only $m$-times weakly
differentiable, which means that $\sobolev{2}[\dom]$ is not an admissible choice in
\cref{sec:cpu-stationary-1d}.

  \section{Linear Partial Differential Equations}
\label{sec:linear-pdes}
\begin{definition}[Multi-index]
  \label{def:multi-index}
  Using a $d$-dimensional \emph{multi-index} $\vec{\alpha} \in \N_0^d$, we can represent
  (mixed) partial derivatives of arbitrary order as
  \begin{equation*}
    \mipderiv{\vec{\alpha}}[\vec{x}]{}
    \defeq \mpderiv*[\abs{\vec{\alpha}}]{%
      \pdiff[(\vecelem{\alpha}_1)]{\vecelem{x}_1} \cdots \pdiff[(\vecelem{\alpha}_d)]{\vecelem{x}_d}
    }{},
  \end{equation*}
  where $\abs{\vec{\alpha}} \defeq \sum_{i = 1}^d \vecelem{\alpha}_i$.
  If the variables w.r.t. which we differentiate are clear from the context, we also
  denote this (mixed) partial derivative by $\mipderiv{\vec{\alpha}}{}$.
  For two multi-indices $\vec{\alpha}, \vec{\alpha}' \in \N_0^d$, we write
  $\vec{\alpha} \le \vec{\alpha}'$ iff $\vecelem{\alpha}_i \le \vecelem{\alpha}'_i$ for
  all $i = 1, \dotsc, d$.
\end{definition}
\begin{definition}[Linear differential operator]
  \label{def:linear-diffop}
  A \emph{linear differential operator} $\linop{D} \colon \solsp \to \rhssp$ of order
  $k$ between a space $\solsp$ of $\R^{d'}$-valued functions and a space $\rhssp$ of
  real-valued functions defined on some common open domain $\dom \subset \R^d$ is a
  linear operator that linearly combines partial derivatives up to $k$-th order of its
  input function, i.e.~%
  \begin{equation*}
    \linopat{D}{\solvv}
    \defeq
    \sum_{i = 1}^{d'} \sum_{\vec{\alpha} \in \N_0^d, \abs{\vec{\alpha}} \le k}
    A_{i, \vec{\alpha}} \mipderiv{\vec{\alpha}}{\solvv_i},
  \end{equation*}
  where $A_{i,\vec{\alpha}} \in \R$ for every $i \in \{ 1, \dotsc, d' \}$ and every
  multi-index $\vec{\alpha} \in \N_0^d$ with $\abs{\vec{\alpha}} \le k$.
\end{definition}

\subsection{Weak Derivatives and Sobolev Spaces}

\begin{definition}[Test Function]
  \label{def:test-fn}
  Let $\dom \subset \R^d$ be open and let
  \begin{equation*}
    \testfns{\dom} \defeq
    \set{
      \phi \in \cdfns[\infty]{\dom}[\R]
      \suchthat
      \operatorname{supp} \left( \phi \right) \subset \dom \text{ is compact}
    }
  \end{equation*}
  be the space of smooth functions with compact support in $\dom$.
  A function $\phi \in \testfns{\dom}$ is dubbed \emph{test function} and we refer to
  $\testfns{\dom}$ as the \emph{space of test functions}.
\end{definition}

\begin{theorem}[Sobolev Spaces\protect\footnotemark]
  \footnotetext{This theorem is a summary of \citep[Definitions 3.1 and 3.2 and Theorems 3.3 and 3.6]{Adams2003Sobolev}}
  \label{thm:sobolev-spaces}
  Let $\dom \subset \R^d$ be open, $k \in \N_{> 0}$, and
  $p \in [1, \infty) \cup \set{\infty}$.
  The functional
  \begin{equation}
    \label{eqn:sobolev-norm}
    \norm{u}[k, p, \dom] \defeq
    \begin{cases}
      \left(
      \sum_{\abs{\alpha} \le k} \norm{\mipderiv{\alpha}{u}}[{\L{p}[\dom]}]^p
      \right)^{\nicefrac{1}{p}}                                               & \text{if } p < \infty, \\
      \max_{\abs{\alpha} \le k} \norm{\mipderiv{\alpha}{u}}[{\L\infty[\dom]}] & \text{if } p = \infty,
    \end{cases}
  \end{equation}
  where the $\mipderiv{\alpha}{}$ are weak partial derivatives, is called a
  \emph{Sobolev norm}.
  A Sobolev norm $\norm{u}[k, p, \dom]$ is a norm on subspaces of $\L{p}[\dom]$, on which the
  right-hand side is well-defined and finite.
  A \emph{Sobolev space} of order $k$ is defined as the subspace
  \begin{equation*}
    \sobolev[k]{p}[\dom]
    \defeq \set{
      u \in \L{p}[\dom]
      \suchthat
      \mipderiv{\alpha}{u} \in \L{p}[\dom]\ \text{for}\ \abs{\alpha} \le k
    }.
  \end{equation*}
  of $\L{p}$.
  Sobolev spaces $\sobolev[k]{p}[\dom]$ are Banach spaces under the Sobolev norm
  $\norm{\cdot}[k,p,\dom]$.
  The Sobolev space $\sobolev{k}[\dom] \defeq \sobolev[2]{k}[\dom]$ is a separable
  Hilbert space with inner product
  \begin{equation}
    \label{eqn:sobolev-inprod}
    \inprod{u_1}{u_2}[k, \dom]
    \defeq \sum_{\abs{\alpha} \le k} \inprod{\mipderiv{\alpha}{u_1}}{\mipderiv{\alpha}{u_2}}[{\L2[\dom]}]
  \end{equation}
  and norm
  \(
  \norm{\cdot}[k, \dom]
  \defeq \sqrt{\inprod{\cdot}{\cdot}[k, \dom]}
  = \norm{\cdot}[k, 2, \dom].
  \)
\end{theorem}

  \bibliography{linpde_gp}

  \ifoptionfinal{}{%
    \newpage
    \tableofcontents
  }
\end{document}